\documentclass{article}

\usepackage{microtype}
\usepackage{graphicx}
\usepackage{subcaption}
\usepackage{booktabs} %

\usepackage{hyperref}
\usepackage{url}

\usepackage[preprint]{icml2026}

\usepackage{amsmath}
\usepackage{amssymb}
\usepackage{mathtools}
\usepackage{amsthm}

\usepackage[capitalize,noabbrev]{cleveref}

\theoremstyle{plain}
\newtheorem{theorem}{Theorem}[section]

\theoremstyle{definition}

\newtheorem{assumption}{Assumption}
\theoremstyle{remark}

\newcommand{\ourmethod}{RUMS}
\newcommand{\ourmethodutility}{RUMS-Utility}
\newcommand{\ourmethodselect}{RUMS-Subset}
\newcommand{\ourmethodmodel}{RUMS-Models}
\newcommand{\ourmethodconserv}{RUMS-Binary}
\newcommand{\ourmethodlessconserv}{RUMS-Multi}
\newcommand{\pf}{PersonaFeedback}
\newcommand{\fb}{FreebaseQA}
\newcommand{\wcp}{WildChat-Personal}
\newcommand{\wcnp}{WildChat-NonPersonal}

\usepackage[textsize=tiny]{todonotes}

\usepackage{times}

\usepackage{amsmath}
\usepackage{amsthm}
\usepackage{amssymb}
\usepackage{mathtools} %
\usepackage{bm}
\usepackage{bbm}

\usepackage{graphicx}
\usepackage{tikz}
\usepackage{adjustbox}
\usepackage{caption}
\usepackage{subcaption}

\usepackage{booktabs}
\usepackage{multirow}

\usepackage{algorithm}  
\usepackage[noend]{algpseudocode}

\usepackage[english]{babel}
\usepackage{mdframed}
\usepackage{lipsum}
\usepackage{fancyvrb}
\usepackage{enumitem}
\usepackage{xspace}

\usepackage{breakcites}
\usepackage{url}

\usepackage{delimset} %
\usepackage{colortbl} 
\usepackage{booktabs}
\usepackage{xcolor}
\usepackage{subcaption} 
\usepackage{wrapfig}
\usepackage[utf8]{inputenc}
\usepackage{multicol}
\usepackage{tikz}
\usepackage{xcolor}
\usepackage{fontawesome5}
\usetikzlibrary{shapes.geometric, arrows.meta, positioning, backgrounds, fit, decorations.pathreplacing}

\definecolor{primaryblue}{RGB}{116,185,255}
\definecolor{darkblue}{RGB}{9,132,227}
\definecolor{purple1}{RGB}{102,92,234}
\definecolor{purple2}{RGB}{118,75,162}
\definecolor{pink1}{RGB}{253,121,168}
\definecolor{pink2}{RGB}{232,67,147}
\definecolor{green1}{RGB}{0,184,148}
\definecolor{green2}{RGB}{85,239,196}
\definecolor{orange1}{RGB}{255,107,107}
\definecolor{orange2}{RGB}{255,178,43}
\definecolor{red1}{RGB}{225,112,85}
\definecolor{lightgray}{RGB}{248,249,250}
\definecolor{darkgray}{RGB}{99,110,114}

\tikzset{
    training/.style={
        rectangle, rounded corners=8pt, 
        fill=purple1, text=white, 
        minimum width=3.5cm, minimum height=0.8cm,
        font=\small\bfseries, align=center
    },
    step/.style={
        rectangle, rounded corners=5pt,
        fill=white, fill opacity=0.2, text=white,
        minimum width=3cm, minimum height=0.6cm,
        font=\footnotesize, align=center
    },
    memory selected/.style={
        rectangle, rounded corners=2pt,
        fill=green1, text=white,
        minimum width=0.8cm, minimum height=0.8cm,
        font=\tiny\bfseries
    },
    memory unselected/.style={
        rectangle, rounded corners=2pt,
        fill=red1, text=white, fill opacity=0.7,
        minimum width=0.8cm, minimum height=0.8cm,
        font=\tiny\bfseries
    },
    user input/.style={
        rectangle, rounded corners=8pt,
        fill=primaryblue, text=white,
        minimum width=4cm, minimum height=1cm,
        font=\small\bfseries, align=center
    },
    model box/.style={
        rectangle, rounded corners=8pt,
        fill=pink1, text=white,
        minimum width=4cm, minimum height=1cm,
        font=\small\bfseries, align=center
    },
    output box/.style={
        rectangle, rounded corners=8pt,
        fill=green1, text=white,
        minimum width=4cm, minimum height=1cm,
        font=\small\bfseries, align=center
    },
    flow arrow/.style={
        -Stealth, thick, color=purple1, line width=1.5pt
    },
    badge/.style={
        rectangle, rounded corners=10pt,
        fill=orange2, text=black,
        font=\tiny\bfseries,
        minimum width=1cm, minimum height=0.4cm
    }
} %

\renewcommand{\epsilon}{\varepsilon}  %

\usepackage{sidecap}
\sidecaptionvpos{figure}{c}

\usepackage{booktabs}
\usepackage{colortbl}
\usepackage{xcolor}
\usepackage{siunitx}
\newtheorem{prompt}{Prompt}

\definecolor{lightgreen}{RGB}{220,255,220}
\definecolor{lightred}{RGB}{255,220,220}

\newcommand{\cellval}[1]{%
  \ifdim #1 pt > -.0001pt
    \ifdim #1 pt > 0.1pt
      \cellcolor{lightgreen}{\textbf{#1}}%
    \else
      \cellcolor{lightgreen}{#1}%
    \fi
  \else
    \ifdim #1 pt < -0.1pt
      \cellcolor{lightred}{\textbf{#1}}%
    \else
      \cellcolor{lightred}{#1}%
    \fi
  \fi
}
 
\icmltitlerunning{Response-Aware User Memory Selection}

\begin{document}

\twocolumn[
  \icmltitle{Response-Aware User Memory Selection for LLM Personalization}

  \icmlsetsymbol{equal}{*}

  \begin{icmlauthorlist}
    \icmlauthor{Jillian Fisher}{1}
    \icmlauthor{Jennifer Neville}{2}
    \icmlauthor{Chan Young Park}{2}
  \end{icmlauthorlist}

  \icmlaffiliation{1}{Department of Computer Science and Engineering, University of Washington, Seattle, WA, United States of America}
  \icmlaffiliation{2}{Microsoft Research, Redmond, WA, United States of America}

  \icmlcorrespondingauthor{Jillian Fisher}{jrfish@uw.edu}
  \icmlkeywords{NLP, Personalization, Information-Theoretics}

  \vskip 0.3in
]

\printAffiliationsAndNotice{}  %

\begin{abstract}
A common approach to personalization in large language models (LLMs) is to incorporate a subset of the user memory into the prompt at inference time to guide the model's generation. Existing methods select these subsets primarily using similarity between user memory items and input queries, ignoring how features actually affect the model's response distribution. We propose \textbf{R}esponse-\textbf{U}tility optimization for \textbf{M}emory \textbf{S}election (\ourmethod{}), a novel method that selects user memory items by measuring the mutual information between a subset of memory and the model’s outputs, identifying items that reduce response uncertainty and sharpen predictions beyond semantic similarity. We demonstrate that this information-theoretic foundation enables more principled user memory selection that aligns more closely with human selection compared to state-of-the-art methods, and models $400 \times$ larger. Additionally, we show that memory items selected using \ourmethod{} result in better response quality compared to existing approaches, while having up to 95\% reduction in computational cost.  

 \end{abstract}

\section{Introduction}
As large language models (LLMs) become ubiquitous in daily interactions, personalization has emerged as critical for user satisfaction. To meet this demand, current approaches to personalization generally fall into two categories: training-based methods and inference-time methods. Training-based approaches, such as low-rank preference models \citep{bose2025lorepersonalizingllmslowrank} or parameter-efficient fine-tuning modules \citep{tan2024personalized}, adapt the model parameters for individual users. Although these methods can be effective for small populations, they become computationally prohibitive when scaled to many users. As a result, inference-time personalization, such as including user-specific context in the input prompt, is often the preferred solution.

However, inference-time methods raise a central question: \textbf{What information, if any, should be included to guide the model?} A naïve strategy is to include all stored user memories in the prompt. However, this method often reduces performance, as irrelevant or distracting information may overwhelm the model and produce responses that are less accurate or sound unnatural \citep{li-etal-2023-compressing, rajeev2025cats}. Moreover, this strategy becomes increasingly impractical as the memory size grows. Therefore, existing methods instead attempt to select and retrieve relevant subsets of the user memory by measuring the relationship between a \textit{memory item and the input query}. In other words, these methods try to pick out only the user memory items that seem most related to the user’s current request. This is often done by heuristic retrieval \citep{bm25s}, surface-level semantic similarity \citep{lei-etal-2023-unsupervised}, or prompting LLMs \citep{handa2024bayesianpreferenceelicitationlanguage}. While intuitive,\textbf{ we hypothesize that measuring how memory items affect output (i.e. response) distributions directly, rather than similarity to input queries, provides a more reliable signal for personalization}.
Selecting user memory items based on semantic similarity to the input may biases selection toward redundant, surface-aligned items and away from non-obvious but important ones, leading to both over-selection and under-selection relative to true utility.
Such shortcomings can introduce noisy items that degrades response quality while increasing latency, cost, and privacy risks \citep{hagström2025cubbenchmarkingcontextutilisation}.

\begin{figure*}
    \centering
    \includegraphics[width=0.8\linewidth]{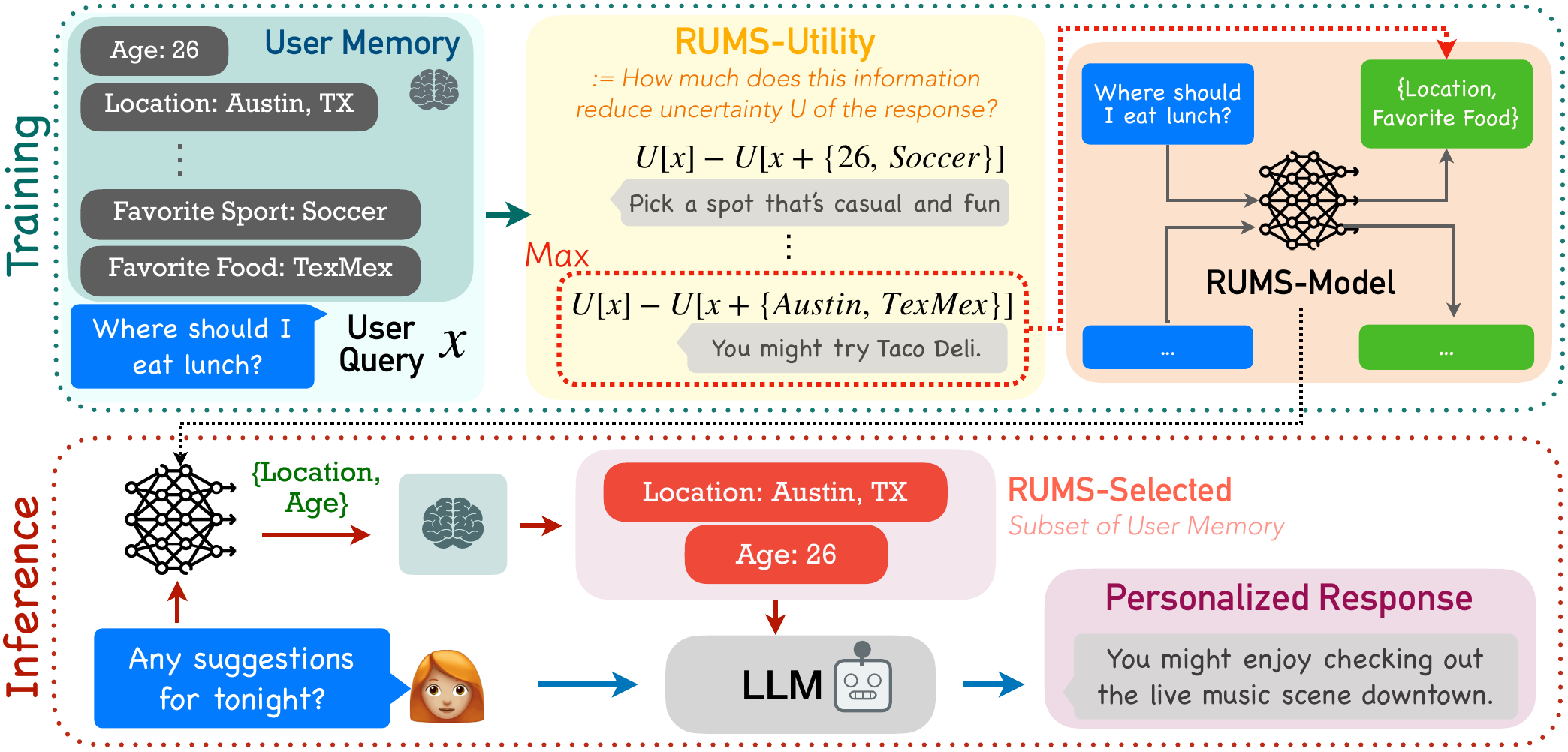}
    \vspace{-5pt}
    \caption{Overview of \ourmethod{}. \ourmethodutility{} quantifies how much memory subsets reduce response uncertainty (training phase), while \ourmethodmodel{} efficiently selects informative items at inference. 
    }
    \vspace{-10pt}
    \label{fig:figure1}
\end{figure*}

To test this hypothesis, we introduce Response-Utility for Memory Selection (\ourmethod{}), which comprises of two key components: a novel utility function (\textit{\ourmethod{}-Utility}) that quantifies the informativeness of user memory subsets, and a lightweight model (\textit{\ourmethod{}-Model}) that efficiently approximates this utility computation at inference time.

\textit{\ourmethod{}-Utility} utilizes an information-theoretic framework to directly measure how much each subset of user memories reduces uncertainty in the model's response distribution. Specifically, it computes the conditional mutual information between candidate memory subsets and the model's output distribution given the input query. We provide a theoretical foundation for this objective, proving that under the existence of a latent variable representing user intent, minimizing predictive entropy is approximately equivalent to maximizing expected user utility. While this guarantee assumes the model’s distribution aligns with human preferences, our empirical results confirm that modern LLMs are sufficiently calibrated for this entropy signal to serve as a high-fidelity proxy for true utility.
 
However, computing \textit{\ourmethod{}-Utility} for all queries is computationally expensive, particularly with large user profiles. Therefore, we use the utility function during an offline training phase to generate high-quality training data. We then train \textit{\ourmethodmodel{}}, a lightweight encoder-classifiers based on DeBERTa \citep{he2021deberta}, to directly predict memory items that are most informative given a query. This learned approach preserves the principled selection quality of the utility function while enabling efficient deployment. 

\textbf{Through extensive empirical testing, the experimental results validate our hypothesis:} \ourmethod{} enables more principled memory item selection that is more closely aligned to human selection compared to state-of-the-art (SOTA) memory selection methods and models $400\times$ its size. We also show that the items selected using \ourmethod{} result in better response quality than other methods, while lowering the cost of personalization overall.

Our work offers novel conceptual, theoretical, and empirical contributions to personalization:
\vspace{-4pt}
\begin{itemize}[leftmargin=*, itemsep=0pt, topsep = 0pt]
    \item We propose \ourmethod{}, a lightweight, response-based method for user memory selection that directly accounts for the impact of memory items on model outputs.  
    \item We establish a principled utility function, \ourmethod{}-Utility, and provide a theoretical proof linking entropy reduction to expected utility maximization. We further validate the necessary assumption of model-human alignment through empirical results, showing that \ourmethod{}-Utility can effectively guide the training of a lightweight \ourmethod{}-Model. 
    \item We show that \ourmethod{} (i) reliably identifies when personalization is beneficial, aligning closely with human judgments, (ii) selects memory subsets that align more strongly with human annotations than SOTA retrieval and prompting baselines, and (iii) improves generation quality over SOTA baselines and $400\times$ larger models on synthetic and real-world datasets, while adapting to underlying model capabilities and costing up to 95\% less than other methods.
\end{itemize}
\vspace{-6pt}
\section{Related Works}

\textbf{User Memory for LLM Personalization} Memory augmentation enables personalization in LLMs by retrieving subsets of user memory at inference time to guide responses. Simple approaches prepend user context to prompts, such as profile summaries \citep{Richardson2023, pan2025secom} or past interactions \citep{kang_personalize, liu2023chatgptgoodrecommenderpreliminary}. However, as the amount of user memory grows, these methods quickly become intractable. Due to this, there has been an increased trend in selecting a \textit{subset} of the user memory to prepend. For example, \citet{chen_MAP} retrieves relevant historical items via semantic similarity for recommendations, while LaMP \citep{salemi2023lamp} investigates personalized sentence retrieval. PEARL \citep{mysore-etal-2024-pearl} and ROPG \citep{10.1145/3626772.3657783} train retrieval components with user signals to improve personalization. For a broader overview, see \citet{llm_personal_survey}. However, these methods rely on the similarity between user input and memory item, without considering their impact on the model’s downstream response.

\textbf{Selective Personalization} 
A related line of reserach is selective conditioning or abstention which studies how models can ignore user memory or abstain from predictions when inputs are uninformative. \citet{10.5555/3295222.3295241} introduce selective classification for neural networks, and recent work extends this to language models, proposing mechanisms to ignore input features or context that may degrade performance \citep{hagström2025cubbenchmarkingcontextutilisation}. In memory selection, this corresponds to sometimes choosing no subset when personalization is unnecessary. Prior work typically relies on global thresholds \citep{li2023selectivecontext} or static policies \citep{NEURIPS2023_2dbb8bfe}, whereas our method learns a sequential, adaptive policy that dynamically selects which items to condition on based on the input and context.

\textbf{Information-theoretic Uses in NLP}
Information-theoretic ideas have long informed NLP \citep{brown-etal-1992-estimate}, spanning multimodal learning \citep{Almudevar2025AligningMultimodal}, information control \citep{ravfogel-etal-2020-null}, interpretability \citep{pmlr-v80-chen18j}, and prompt design \citep{zhang-etal-2022-prompt-based}. In feature selection, EDDI \citep{ma2018eddi} estimated conditional mutual information for individual features, while later work extended this to dynamic or sequential selection \citep{10.5555/3618408.3618665,gadgil2024estimating}. Recent methods also use differences in next-token likelihood with and without context to guide generation through distribution editing \citep{he2025context}. While much of this work has focused on domains such as image classification and fraud detection, the closest parallel to personalization arises in recommendation systems using Bayesian active learning to select a model query to ask the user, with the goal of improving recommendation quality \citep{handa2024bayesianpreferenceelicitationlanguage}. We draw on these principles to select memory items that maximize information gain for personalized responses. 

\vspace{-1pt}
\section{Method}
\vspace{-1pt}

We aim to develop a method that selects a subset of user memory based on the goal of maximizing the utility of the model’s responses. 

\textbf{Problem Formulation}
Consider user $u$ with query $x$ (e.g. ``Where should I eat lunch?''). We let $\theta$ denote an LLM and $Y\!=\!(Y_1,...,Y_T)$ denote a random variable representing an output sequence distributed according to $p_\theta(Y \vert x)$. We measure the quality of a specific realized response $y$ by its utility to the user, $\text{U}_u(y)$. 

Also, we consider a set of memory items $\mathcal{M}$ (e.g. Location), in which user $u$ has specific values $m_u$ (e.g. Location: Austin, TX). Our aim is to select the memory subset $S \subseteq \mathcal{M}$ with instantiated values $m_{u,S}=\{m_u^{(a)}:a \in S\}$, which when given to the model with the input $x$ yields the highest response utility. Formally,  
\begin{align}
    S^*(x,u) = \arg\max_{S \subseteq \mathcal{M}} \; \mathbb{E}_{y \sim p_\theta(\cdot \mid x, m_{u,S})}\big[\text{U}_u(y)\big].
\end{align}
For brevity, unless otherwise specified, expectations are taken w.r.t. the model's predictive distribution, $\mathbb{E}[\cdot] = \mathbb{E}_{y \sim p_\theta(\cdot \mid x, m_{u,S})}[\cdot]$. Also, we assume $x$ is fixed throughout.

However, we note that $\text{U}_u$ is unobserved and cannot be evaluated prior to generation, so there is a need for a proxy utility function. 

\vspace{-.1cm}
\subsection{RUMS-Utility}
In order to formulate a surrogate utility function, we frame memory selection for personalization as an information acquisition problem: given a large space of user memory, the goal is to select the items that most improve our response to an input.
Given this framing, we propose a novel utility function based on information gain from memory subsets, serving as an entropy-based surrogate for $\text{U}_u(y)$. Formally,
\begin{align}
\text{U}_{RUMS}(S, u)  = H_\theta \left( Y \vert x \right) - H_\theta \left( Y \vert x, m_{u,S} \right) \label{eqn:novel_utility}
\end{align}
where $H_\theta(\cdot)$ denotes Shannon entropy under the model distribution (defined in \cref{appx:shannon_entropy_defn}). Given that only the second term in \cref{eqn:novel_utility} is dependent on the subset selected $S$, we can simplify the objective and state that \ourmethod{} chooses the optimal subsets $S^*$ such as,
\begin{align*}
    S^*(x,u) = \arg \min_{S \subseteq \mathcal{M}} H_\theta(Y \mid x, m_{u,S}) 
\end{align*}
The intuition is that a memory subset $S$ is valuable to the extent it reduces uncertainty in the model’s predictions. For instance, for ``where should I eat lunch?'', without user information, the model spreads probability across countless restaurants, yielding high entropy. However, adding even a single relevant fact, such as ``Location: Austin, TX'' or ``Dietary preference: Vegan'', sharply concentrates probability on relevant options. Therefore, the utility of the memory subset $S$ is measured directly by reduction in the entropy.

\vspace{-.1cm}
\subsection{Theoretical Foundation}
We formalize conditions under which reducing predictive uncertainty leads to higher expected user utility. While the theoretical foundation rely on a latent-variable abstraction, we validate empirically (Section 4) that minimizing predictive entropy aligns with improved user-perceived utility.

\begin{assumption}[Existence of an Optimal Latent Variable]
\label{assump:existence_latent}
For each user $u$ and query $x$, there exists a latent random variable $Z_u \sim p(Z_u \mid x)$ 
representing the user-specific latent target that determines what constitutes an optimal response.
User utility depends on alignment between the generated response and this latent target.

Formally, where $\ell$ is a strictly proper scoring rule, we assume
\vspace{-.1cm}
\[
\mathbb{E}\big[\text{U}_u(y)\big]
=
- \mathbb{E}_{z_u \sim p(Z_u \mid x)}
\mathbb{E}_{y \sim p_\theta(\cdot \mid x, m_{u,S})}
[\ell(y, z_u)].
\]

\end{assumption}
\vspace{-.2cm}
Conceptually, $Z_u$ is an abstract latent variable capturing the maximal predictive information needed for an optimal response.\footnote{While this offers a formal link between predictive distributions and decision quality \citep{gneiting2007scoring}, it assumes static user preferences. In practice, preferences can be non-stationary and shaped by model outputs, so a fixed target may miss the co-evolutionary dynamics of long-term interactions.}

\begin{assumption}[Approximate Sufficiency of Latent Variable]
\label{assump:suffic_latent}
Conditioned on the query $x$ and latent target $Z_u$, the memory values $m_{u,S}$ do not provide substantial additional information about the output $Y$:
\vspace{-.1cm}
\begin{align*}
p_\theta(Y \mid x, Z_u, m_{u,S}) \approx p_\theta(Y \mid x, Z_u).
\end{align*}
\end{assumption}
\vspace{-.2cm}
This assumption treats $Z_u$ as an informational bottleneck, assuming memory items primarily inform the response by clarifying latent intent. However, this may overlook scenarios where the raw data, such as specific formatting, provides utility not captured by an abstract intent. Therefore, we adopt approximate rather than exact sufficiency to account for real-world cases where memory might contain extraneous details beyond the latent preference \citep{tishby1999information, alemi2017deepib}.

\begin{theorem}[Optimality of \textsc{RUMS}]
\label{thm:rums_optimality}
Under Assumptions \ref{assump:existence_latent} and \ref{assump:suffic_latent}, for any two memory subsets $S, S' \subseteq \mathcal{M}$,
\begin{align*}
H_\theta(Y \mid x, m_{u,S}) \le H_\theta(Y \mid x, m_{u,S'})
\quad \Rightarrow \quad \\
\mathbb{E}_{y \sim p_\theta(\cdot \mid x, m_{u,S})}[\text{U}_u(y)]
\ge
\mathbb{E}_{y \sim p_\theta(\cdot \mid x, m_{u,S'})}[\text{U}_u(y)].
\end{align*}
\end{theorem}
\vspace{-.2cm}
Theorem~\ref{thm:rums_optimality} formalizes the intuition that reducing predictive entropy improves expected utility, with approximation arising from the latent-variable abstraction in \cref{assump:suffic_latent}.
\vspace{-.6cm}
\paragraph{Proof Sketch.}
Predictive entropy decomposes into uncertainty about the latent target $Z_u$ and residual uncertainty about $y$. Under \cref{assump:suffic_latent}, selected memory mainly reduces uncertainty about $Z_u$, and because strictly proper scoring rules (\cref{assump:existence_latent}) reward matching the conditional distribution, this reduction directly improves expected utility. Full derivation is provided in Appendix~\ref{appx:optimality_proof}.
\vspace{-.2cm}
\paragraph{Interpretation and Limitations} \cref{thm:rums_optimality} provides a principled justification for selecting memory subsets that reduce predictive entropy. However, in practice, \textit{this surrogate objective is only meaningful if the model’s predictive distribution $p_\theta(y \vert x, m_{u,S})$ is aligned with human judgments}, that is, responses the model assigns higher probability to under relevant memory subsets tend to be more useful to humans. While \cref{thm:rums_optimality} provides a principled foundation, \textit{empirical evaluation (Section 4) is necessary} to verify that reductions in entropy translate into improvements in actual user-perceived utility.

\vspace{-.1cm}
\subsection{Practical Estimation}
Although \ourmethodutility{} is theoretically grounded, directly computing it for LLM-generated sequences is intractable, as the space of possible outputs y grows exponentially with the sequence length T. Therefore, we employ practical approximations via token-level decomposition and sampling.

\vspace{-.2cm}
\paragraph{Token-Level Decomposition}  
To achieve tractability, we leverage the chain rule of probability to decompose the sequence-level entropy into a manageable sum of conditional token-level entropies:
\vspace{-.2cm}
\begin{align*}
H_\theta&(Y \mid x, m_{u,S}) =
&\mathbb{E}
\Big[ \sum_{t=1}^T H\big(p_\theta(y_t \mid y_{<t}, x, m_{u,S})\big) \Big].
\end{align*}
Here, the expectation is taken w.r.t. $p_\theta(y_t \mid y_{<t}, x, m_{u,S})$ the model’s conditional distribution over the next token. This decomposition reduces the problem to evaluating token-level entropies, which is computationally manageable for long sequences (see proof of this composition in \cref{appx:seq_entropy_proof}).
\vspace{-.4cm}
\paragraph{Monte Carlo Approximation}  
The outer expectation over token prefixes remains intractable, so we approximate it using Monte Carlo sampling. We generate $N$ sequences $\{y^i\}_{i=1}^N$ from $p_\theta(\cdot \mid x, m_{u,S})$ and compute
\[
\hat{H}_\theta(Y \mid x, m_{u,S}) = \frac{1}{N} \sum_{i=1}^N \frac{1}{T_i} \sum_{t=1}^{T_i} H\big(p_\theta(y_t^i \mid y_{<t}^i, x, m_{u,S})\big),
\]
where $T_i$ is the length of the $i$-th sequence. We normalize by sequence length to estimate per-token entropy, which allows comparison across variable-length generations.

This approximation converges to the true per-token sequence entropy as $N \to \infty$ under i.i.d. sampling from $p_\theta(\cdot \mid x, m_{u,S})$ with finite expectation. In practice, however, factors such as constrained decoding or approximate sampling distributions can introduce bias. Despite these limitations, we find $N=5$ samples provides a good balance between accuracy and efficiency. For more empirical analysis, see \cref{appx:robust_utility_scores}. 

The final estimated utility is then,
\vspace{-4pt}
\begin{align}
\widehat{\text{U}}_{\textsc{RUMS}}(S, u) = \hat{H}_\theta(Y \mid x) - \hat{H}_\theta(Y \mid x, m_{u,S}). \label{eqn:method_utility}
\vspace{-6pt}
\end{align}

Then, our final objective selects the memory subset that maximizes this utility:
\vspace{-4pt}
\begin{align}
    S^*(x, u) = \arg \max_{S \subseteq M} \widehat{U}_{\textsc{RUMS}}(S,u). \label{eqn:method_obj_fn}
\end{align}
$\widehat{\text{U}}_{\textsc{RUMS}}(S, u)$ serves as a practical surrogate for user utility and retains the theoretical connection to expected utility established in Theorem~\ref{thm:rums_optimality}.

\vspace{-.1cm}
\subsection{\ourmethod-Models: Learning to Approximate Utility}
\label{sec:method:model}
While \ourmethod{} provides a principled approach to memory selection, explicitly computing entropy reductions over all possible subsets at inference time is prohibitive, particularly for users with large memory profiles. To enable efficient deployment, we use \ourmethod{} offline to generate training data and train a lightweight model to directly predict memory selections at inference.
\vspace{-.3cm}
\paragraph{Training Data Generation}
Our training pipeline uses \ourmethod{} to generate (user query, optimal memory subset) pairs.
For each query $x$, we pair it with a diverse set of users $\{u_1, \ldots, u_K\}$, each with memory items $m_u$. When profiles are large ($|m_u| \geq 50$), we first use an LLM to filter to $k$ candidate items. We then compute $\widehat{\text{U}}(S, u)$ from \cref{eqn:method_obj_fn} for candidate subsets $S$, selecting the highest-utility subset $S^*(x,u) = \arg\max_{S} \widehat{\text{U}}(S,u)$ if it exceeds threshold $\tau$, or the empty set otherwise. This yields $K$ training examples $\{(x, u_k, S^*(x,u_k))\}_{k=1}^K$ per query. Pairing each query with diverse profiles encourages the model to learn genuinely informative item selection rather than spurious correlations. The threshold $\tau$, tuned on validation data, enables automatic abstention when personalization would not help.

\textbf{Model Architecture and Training}
We train \ourmethod{}-Models as lightweight classifiers to approximate utility computation at inference. For dynamic memory schemas, where item types vary across users, a binary classifier (\ourmethodconserv{}) $f(x, m_i) \rightarrow \{0,1\}$ evaluates each memory item $\{m_i\}_{i=1}^n$ independently, handling arbitrary structures but requiring $n$ forward passes. For static schemas with $d$ fixed item types, a multiclass classifier (\ourmethodlessconserv{}) $f(x) \rightarrow \{0,1\}^d$ outputs all selections in a single pass. The resulting \ourmethod{}-Models match the inference speed of simple similarity-based retrieval while retaining principled selection quality, making it practical for large-scale deployment. %
\begin{figure*}[t]
    \centering
    \begin{subfigure}[c]{0.35\textwidth} %
        \centering
        \includegraphics[width=0.85\linewidth]{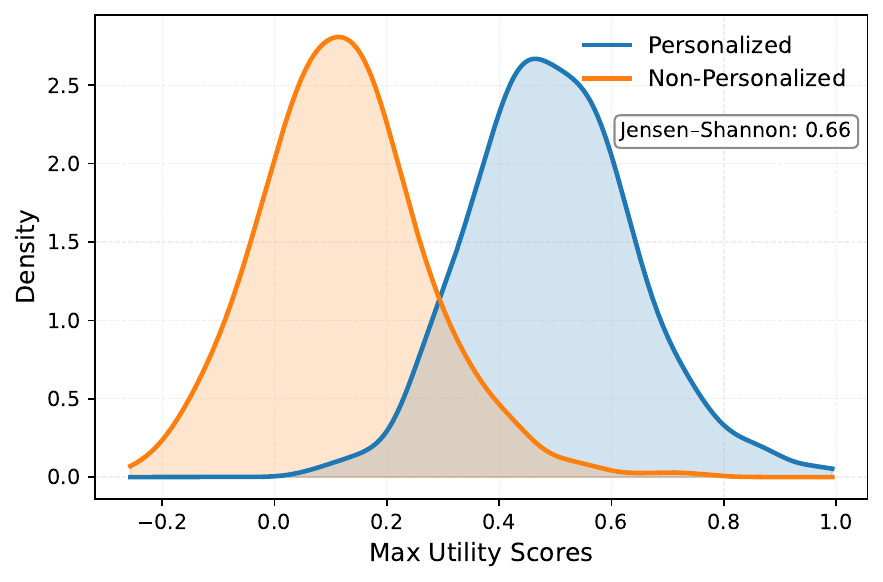}
        \caption{Distribution of maximum utility scores for personalized (\pf{}) and non-personalized (\fb{}) inputs.}
        \label{fig:h1_distr_static}
    \end{subfigure}
    \hfill
    \begin{subfigure}[c]{0.6\textwidth} %
        \centering
         \caption{Recall (accuracy on personalized) and specificity (accuracy on non-personalized) comparison across datasets and methods. \textbf{Bolded} is the highest value in each category.}
        \label{tab:all_agree_binary_personalization}

\small
\centering
\begin{tabular}{l p{1cm} p{1cm} p{1cm} p{1.5cm}}
\toprule
\textbf{Dataset / Metric} & \ourmethod{}-Utility & GPT4-Few & GPT4-Zero & Semantic Similarity \\
\midrule
Synthetic (n=20) / Recall & 94.4 & \textbf{100.0} & 33.3 & \textbf{100.0} \\
Synthetic (n=20) / Specificity & \textbf{94.4} & 50.0 & 50.0 & 0.0 \\
Real-World (n=31) / Recall & \textbf{92.3} & 76.9 & 53.9 & 7.7 \\
Real-World (n=31) / Specificity & 96.0 & \textbf{100.0} & \textbf{100.0} & \textbf{100.0} \\
\bottomrule
\end{tabular}
    \end{subfigure}
    \vspace{-.2cm}
    \caption{Results for H1 analysis.}
    \label{fig:combined_h1}
    \vspace{-.2cm}
\end{figure*}
 \section{Results}\label{sec:training_method}

In this section, we evaluate \ourmethod{} by defining three outputs: \textit{\ourmethodutility{}}, the utility score from \cref{eqn:method_utility}; \textit{\ourmethodselect{}}, the selected memory subset from \cref{eqn:method_obj_fn}; and \textit{\ourmethodmodel{}}, models trained on \ourmethod{}-generated data.

We evaluate our approach along three dimensions: 
\begin{itemize}[leftmargin=28pt,topsep=0pt,partopsep=0pt, itemsep=0pt, parsep=0pt]
\item[H1:] \ourmethodutility{} provides a reliable signal of human preference for distinguishing between inputs that do and do not benefit from personalization.
\item[H2:] The subsets selected, \ourmethodselect{}, choose memory subsets more aligned with human preference than SOTA baselines and larger models. Furthermore, the utility signal is strong enough that models trained on them (\ourmethodmodel{}) also outperform baselines and larger models.
\item[H3:] Memory subsets selected with \ourmethodmodel{} improve downstream response generation compared to state-of-the-art baselines and larger models.
\end{itemize}
\vspace{-.2cm}
\subsection{Experimental Setup}\label{sec:exp_setup}
Across all experiments (H1–H3), we consider $n=50$ static user memory items (e.g., \textit{location}, \textit{age}, \textit{name})\footnote{We generated 100 candidate memory items, but items beyond the top 50 were increasingly obscure and less real-world attributes.}, and assumes that all values are given for each user (see \cref{appx:memory_items} for list of items). The task is to select up to $5$ items that are most relevant for answering a given user input $x$.

{\setlength{\textfloatsep}{0.01pt} %
\setlength{\tabcolsep}{3pt}
\begin{table*}[t]
\centering
\footnotesize
\setlength{\tabcolsep}{3pt} %
\caption{Precision, recall, and F1 scores of memory items selected compared to human gold label annotation across datasets and methods.}
\label{tab:H2_synth_realworld}
\resizebox{.8\textwidth}{!}{
\begin{tabular}{l l c c c c c c}

\toprule
\textbf{Dataset} & \textbf{Metric} & \ourmethodselect{} & \ourmethodconserv{} & \ourmethodlessconserv{} & Semantic Sim. & GPT4-Zero & GPT4-Few \\
\midrule
 & Precision & \textbf{0.61} & 0.60 & 0.23 & 0.09 & 0.26 & 0.53 \\
Synthetic  & Recall    & 0.70 & \textbf{0.89} & 0.39 & 0.23 & 0.42 & 0.65 \\
 & F1        & 0.70 & \textbf{0.78} & 0.31 & 0.14 & 0.35 & 0.64 \\
\midrule
 & Precision & \textbf{0.48} & 0.32 & 0.30 & 0.17 & 0.39 & 0.42 \\
Real World & Recall    & 0.36 & \textbf{0.60} & 0.38 & 0.27 & 0.52 & 0.39 \\
 & F1        & \textbf{0.53} & 0.45 & 0.35 & 0.20 & 0.43 & 0.39 \\
\bottomrule
\end{tabular}
}
\vspace{-.3cm}
\end{table*}
}

\textbf{Datasets.} 
To test both instances that would and would not benefit from personalization, we use a diverse mixture of possible user inputs. We evaluate across $3$ datasets:  

\begin{itemize}[leftmargin=*, topsep=-2pt,itemsep=0pt]
    \item \textbf{\pf{}} \citep{tao2025personafeedbacklargescalehumanannotatedbenchmark}: 
    A synthetic dataset of LLM-generated user inputs using randomly sampled user memory, ensuring each input benefits from personalization. We translated the original Chinese data using a translation model \citep{TiedemannThottingal:EAMT2020}.
    Example: \textit{``Do you have any music for the trip?''}

    \item \textbf{\fb{}} \citep{jiang-etal-2019-freebaseqa}: A dataset of trivia-style questions and answers. Example: \textit{``What type of snake is a sidewinder?''}

    \item \textbf{WildChat} \citep{zhao2024wildchat}: Real-world ChatGPT interactions (English only, non-toxic, and request-based). We use only the first human-written input. Examples: \textit{``im hungry''; ``who is Mohamed Gamal Mohamed Maniieh''}
\end{itemize}

By nature of the data, we classify \pf{} as personalized inputs and \fb{} as non-personalized inputs, and refer to these together as the ``Synthetic'' dataset. For WildChat, two annotators labeled each input as personalized or non-personalized (agreement required). We denote these subsets as \wcp{} and \wcnp{}, which form the ``Real-World'' dataset. 

\textbf{User Profiles.} 
We used GPT-4 \citep{openai2024gpt4technicalreport} to generate 50 common user memory items and n=100 consistent user profiles (avoiding contradictions like ``vegan'' with ``favorite food = steak'') for training and testing. See \cref{appx:profile_gen} for prompts, distributions of memory item values and generation details.

\textbf{\ourmethod{}-Model Training} To create training data with \ourmethod{}, we used Llama-3.1-8B \citep{meta_llama_3_1_8b_instruct} to compute entropies. Following the procedure in \cref{sec:method:model}, GPT-4 was first prompted to reduce the candidate memory items from 50 to at most 10 for each query before computing the utility scores. The utility scores were estimated with $n=5$ Monte Carlo samples. Lastly, a threshold of $\tau=0.29$, determined on a held-out validation set,was applied to choose between the subset with maximum utility or the empty set (i.e., no personalization). We use $n=2.5$K from \pf{} and $2$K from \fb{} for training. Since both datasets are also used for evaluation, we rely on held-out splits to avoid overlap and interpret results as in-domain performance. To assess generalization, we also evaluate on WildChat, a real-world dataset from a different distribution.

We train two \ourmethod{}-Model variants: The first, \ourmethodconserv{}, is a binary classifier trained on individual (query, memory value) pairs. Because each memory value is treated separately and negative labels dominate, this model tends to under-select, making it better suited for noise-sensitive settings. The second, \ourmethodlessconserv{}, is a multiclass classifier trained on denser labels constructed by merging the selected subsets from the $u=5$ user profiles per query. This means, all memory items chosen by any of the profiles were included. As a result, the model tends to over-select, which is advantageous for more robust models that can filter noise. Both are based on DeBERTa-v3-large \citep{he2021deberta}. Full details are in \cref{appx:model_training}.

\textbf{Baselines} 
Current approaches to memory selection primarily rely on two strategies: semantic similarity-based retrieval and LLM prompting. To comprehensively evaluate against existing methods, we employ baselines representing both approaches: (1) \textbf{Semantic Similarity}: We compute cosine similarity between input and memory item embeddings using Sentence Transformer \citep{all-MiniLM-L6-v2}, representing the dominant paradigm of surface-level similarity matching in current retrieval systems, (2) \textbf{GPT-4 Prompting}: We use direct prompting of GPT-4 (\textit{zero-shot} and \textit{few-shot} - up to 4 in-context examples) as a representative LLM-based approach and upper bound on model capability, given its 400× larger size compared to \ourmethodmodel{}.
For H3, which evaluates downstream generation quality, we thus additionally include established retrieval methods and naive baselines to ensure comprehensive coverage: (3)\textbf{ SOTA Retrieval Methods}: We evaluate \textit{BM25} \citep{bm25s} and \textit{ReContriever} \citep{lei-etal-2023-unsupervised} as representative SOTA retrieval approaches for memory selection, (4) \textbf{Naive Baselines}: We include \textit{None} (no items), \textit{All} (50 items), and \textit{Random} (5 items) to establish performance bounds and test robustness to different selection strategies.
{\setlength{\textfloatsep}{.05pt}
\begin{table*}[t]
\centering
\small
\caption{Average win rate of our method against baselines on Synthetic (left) and Real World (right) datasets. Cells in \textcolor{green}{green} mark cases where our method outperforms the baseline, and cells in \textcolor{red}{red} mark cases where the baseline does better. Values with magnitude above 0.1 are \textbf{bolded} to highlight strong differences. Results are reported for two response models (LLaMA 70B Instruct and GPT-4) and two \ourmethod{} variants: \ourmethodconserv{} (fewer items) and \ourmethodlessconserv{} (more items).}
\label{tab:synth_realworld_winrates}
\resizebox{\linewidth}{!}{%
\begin{tabular}{cc}

\begin{tabular}{lcccccc}
&& \multicolumn{3}{c}{Synthetic} &&\\
\toprule
 & \multicolumn{3}{c}{LLaMA 70B} & \multicolumn{3}{c}{GPT-4} \\
 & Utility & Binary & Multi & Utility & Binary & Multi \\
\midrule
None              & \cellval{0.11} & \cellval{0.25}  & \cellval{-0.13} & \cellval{0.20} & \cellval{0.26}  & \cellval{0.36} \\
All               & \cellval{0.18} & \cellval{0.16}  & \cellval{-0.19} & \cellval{0.20} & \cellval{-0.16} & \cellval{0.02} \\
Random            & \cellval{0.44} & \cellval{0.34}  & \cellval{0.13}  & \cellval{0.10} & \cellval{0.06}  & \cellval{0.19} \\
Semantic Sim.     & \cellval{0.43} & \cellval{0.27}  & \cellval{0.09}  & \cellval{-0.12} & \cellval{-0.03} & \cellval{0.05} \\
BM25              & \cellval{0.03} & \cellval{0.19}  & \cellval{-0.13} & \cellval{0.18} & \cellval{0.20}  & \cellval{0.32} \\
ReContriever      & \cellval{0.31} & \cellval{0.27}  & \cellval{-0.03} & \cellval{-0.05} & \cellval{-0.04} & \cellval{0.03} \\
GPT4-Few          & \cellval{0.07} & \cellval{0.18}  & \cellval{-0.16} & \cellval{-0.07} & \cellval{-0.09} & \cellval{-0.01} \\
\bottomrule
\end{tabular}

&

\begin{tabular}{cccccc}
& \multicolumn{4}{c}{Real World} &\\
\toprule
\multicolumn{3}{c}{LLaMA 70B} & \multicolumn{3}{c}{GPT-4} \\
Utility & Binary & Multi & Utility & Binary & Multi \\
\midrule
\cellval{-0.05} & \cellval{0.05}  & \cellval{0.03}  & \cellval{0.10} & \cellval{0.20}  & \cellval{0.17} \\
\cellval{0.03} & \cellval{-0.04} & \cellval{-0.10} & \cellval{-0.14} & \cellval{-0.14} & \cellval{-0.06} \\
\cellval{0.25} & \cellval{0.11}  & \cellval{0.13}  & \cellval{0.05} & \cellval{0.10}  & \cellval{0.04} \\
\cellval{0.16} & \cellval{-0.03} & \cellval{0.00}  & \cellval{-0.06} & \cellval{-0.02} & \cellval{0.00} \\
\cellval{0.06} & \cellval{0.01}  & \cellval{0.05}  & \cellval{0.10} & \cellval{0.18}  & \cellval{0.13} \\
\cellval{0.22} & \cellval{0.06}  & \cellval{0.01}  & \cellval{0.07} & \cellval{-0.05} & \cellval{0.13} \\
\cellval{0.01} & \cellval{0.00}  & \cellval{-0.08} & \cellval{-0.08} & \cellval{0.04}  & \cellval{0.11} \\
\bottomrule
\end{tabular}
\end{tabular}
}
\end{table*}
}
 \vspace{-.3cm}
\subsection{H1: Does \ourmethodutility{} Distinguish Personalized from Non-Personalized Inputs?}
We hypothesize that entropy reduction when conditioning on user memory reflects its importance. Thus, inputs that do not benefit from personalization should yield low maximum  utility scores (i.e., the maximum utility score across all subsets of candidate memory items for a specific user), while personalized inputs should yield high ones.
To test this, we compare the maximum utility scores on a known non-personalized (\fb{}) and a known personalized (\pf{}) datasets. We sampled 500 inputs per dataset and computed the maximum utility score using \cref{eqn:method_utility} for each input. The resulting distributions were compared using divergence metrics and hypothesis testing. 
\paragraph{Utility scores separate personalized from non-personalized inputs.} In \cref{fig:h1_distr_static} we show that personalized inputs exhibit much higher utility scores (mean = $0.50$, std = $0.12$) than non-personalized ones (mean = $0.12$, std = $0.09$). The Jensen–Shannon divergence between distributions is high at $0.66$, and a two-sample Kolmogorov–Smirnov test yields a statistic of $0.82$ ($p < .001$), strongly rejecting the null hypothesis of equal distributions. These results indicate that maximum utility scores reliably separate personalized from non-personalized inputs. 
For further exploration on hyperparameter selection and it's effects on the robustness of these results see \cref{appx:robust_utility_scores}.

\vspace{-.2cm}
\subsubsection{H1$'$: Do Utility Scores Align with Human Judgments of Personalization?}
We next test whether personalization decisions made by the utility score thresholding aligns with human judgment. We randomly sampled $n\!=\!25$ inputs from each of \pf{}, \fb{}, \wcnp, and \wcp. Four human annotators labeled whether each input benefits from personalization, which yielded an agreement between the parties of Cohen's $\kappa\!=\!0.59$; we retained only inputs with unanimous agreement ($n\!=\!64$). 

In order to get a robust signal using \ourmethod{}, we used $n\!=\!10$ Monte Carlo samples and $u\!=\!10$ randomly selected user-profiles. We used the average maximum utility score over the profiles as \ourmethodutility{}. We then tuned a threshold on 20\% validation data. 
We compared \ourmethodutility{} against GPT-4 (prompted to determine personalization necessity) and semantic similarity baselines with a threshold tuned on the validation set. The 
performance was measured by recall (accuracy in identifying personalized inputs) and specificity (accuracy in identifying non-personalized inputs). If both humans and the model selected no items, the sample was counted correct for both metrics. 

\textbf{\ourmethodutility{} correlate more strongly with human judgments.}
As shown in Table \ref{tab:all_agree_binary_personalization}, \ourmethodutility{} outperforms GPT-4 and semantic similarity baselines. Compared to GPT-4 FewShot, it achieves +7\% recall and +15\% specificity overall. Gains are strongest on synthetic data, with competitive results on real-world. Other baselines show imbalanced behavior (e.g., GPT-4 ZeroShot often avoids personalization, giving high specificity but low recall). In contrast, utility scores perform well on both metrics, reflecting their ability to capture each item’s true effect on generation.
\vspace{-.2cm}
\subsection{H2: Does \ourmethodselect{} Match Human Judgment?}
Building on H1, we now ask whether our method identifies which specific memory items help personalization. We randomly selected $n=15$ samples from each of \pf{}, \fb{}, \wcnp{}, and \wcp{} for a total of $n=60$. As before, four annotators selected any (or none) of 50 possible items they judged useful, yielding $41\%$ majority ($k \geq 3$) agreement.

While moderate, this level of agreement is expected given the subjective nature of the task. We therefore considered an item as selected by humans if a majority agreed, resulting in an average of 4.19 items selected per user query for inputs where at least one item was chosen. 

We evaluated (i) direct selection from utility scores \ourmethodselect{} (up to 5)
and (ii) \ourmethodmodel{}, the trained models described in \cref{sec:training_method}. 
We compare \ourmethodselect{} and \ourmethodmodel{} against GPT-4 (prompted to select up to five relevant memory items, both zero- and few-shot) and semantic similarity (selecting up to five highest-scoring items with cosine similarity $>0$).
Performance was measured with precision, recall, and F1 against human labels. 

\textbf{\ourmethodselect{} aligns more closely with human annotation.}
As shown in \cref{tab:H2_synth_realworld}, the two methods based on \ourmethod{} achieve on average $20\%$ higher precision, $31\%$ higher recall, and $9\%$ higher F1 compared to the baselines across both synthetic and real-world datasets. Semantic similarity performs worst, reflecting reliance on surface overlap rather than downstream utility.
GPT-4 shows strong recall but substantially lower precision, indicating that it tends to over-select. In contrast, the results show that the raw utility scores, \ourmethodselect{}, themselves provide strong signal for identifying useful memory items, as \ourmethodmodel{} achieves similar F1, demonstrating that small models with modest amounts of data can learn effective memory selection from utility scores. For analysis of the approximation gap between \ourmethodutility{} and \ourmethodmodel{}, see \cref{appx:H2_approx_gap}; qualitative examples appear in \cref{appx:H2_qualitative}.

\vspace{-.2cm}
\subsection{H3: Do Items Selected by \ourmethodmodel{} Improve Downstream Generation?}
Having established that \ourmethod{} makes effective decisions about when and what to personalize (H1-H2), we now demonstrate that these selections translate to tangible improvements in response quality. 
This tests our core claim: principled selection leads to better personalization.

\textbf{Experimental Setup.}
We evaluate \ourmethod{} on both synthetic (\pf{} and \fb{}, 100 samples each) and real-world datasets (WildChat, 50 personalized and 50 non-personalized) to assess generalization. Responses are generated using two models: GPT-4, a high-performing model with robust capabilities, and LLaMA 70B Instruct \citep{meta_llama_3_1_70b_instruct}, a smaller open-source model that is relatively more susceptible to noise. 
We evaluate response quality using \ourmethodutility{} and \ourmethodmodel{} against all baselines described in \cref{sec:exp_setup}. However, we are most interested in \ourmethodmodel{} performance, since \ourmethodutility{} is much more computationally expensive to run at inference. 
 Performance is measured by GPT-4 win rate using chain-of-thought prompting and confidence thresholding (Appendix~\ref{appx:eval_prompts}).

\begin{figure}[t]
\centering
\small
\begin{subfigure}{0.8\linewidth}
\caption{Average number of memory items by method and dataset type.}
\label{fig:avg_tokens}
\begin{tabular}{lcc}
\toprule
Method & Synthetic & Real World \\
\midrule
All & 50.00 & 50.00 \\
None & 0.00 & 0.00 \\
Random & 5.00 & 5.00 \\
Semantic Similarity & 4.63 & 4.79 \\
BM25 & 0.60 & 0.35 \\
ReContriever & 5.00& 5.00 \\
GPT4-Few & 2.80 & 3.00 \\
\rowcolor{gray!20}\ourmethodconserv{} & 1.58 & 2.22 \\
\bottomrule
\end{tabular}
\end{subfigure}
\begin{subfigure}{0.8\linewidth}
\includegraphics[width=\linewidth]{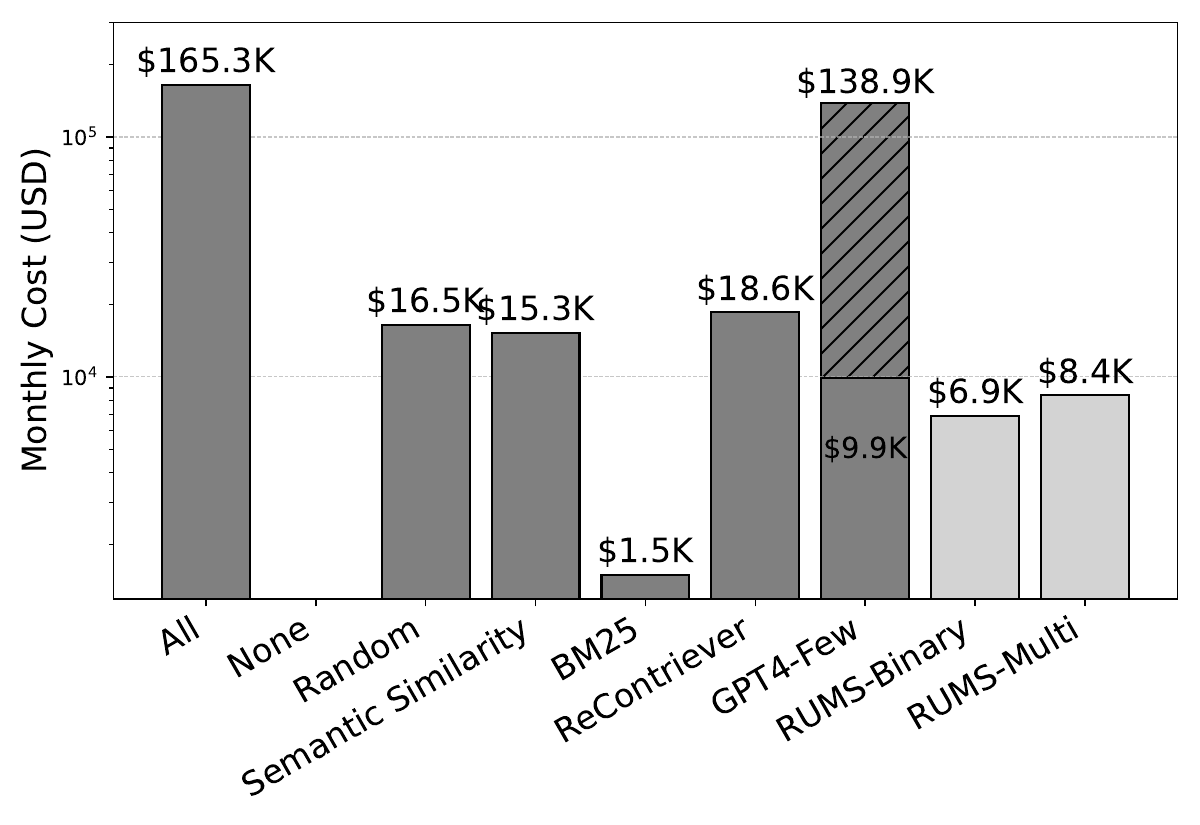}
\caption{Estimated additional average monthly cost by method. Diagonal line indicate added cost of LLM calls for memory selection.}
\label{fig:cost_analysis}
\end{subfigure}
\caption{Cost Analysis}
\vspace{-20pt}
\end{figure}

\textbf{\ourmethod{} improves response quality across diverse settings}
Table \ref{tab:synth_realworld_winrates} reports generation win rates comparing our method against the baselines synthetic and real-world datasets.
On the synthetic dataset, \ourmethodconserv{} consistently outperforms all baselines, achieving at least a $15\%$ higher win rate when responses are generated using the less robust LLaMA 70B model. Most notably, it has an $18\%$ winrate over GPT4-Fewshot, a model $400\times$ it's size. Most of these gains come from non-personalized inputs, where a lower-capability model struggles to ignore noisy items, reducing its performance. In contrast, when using a more capable model such as GPT-4 for response generation, \ourmethodlessconserv{} outperforms almost all baselines, primarily driven by personalized inputs. In this setting, the model can leverage more items, including potentially noisy ones, while still filtering out unhelpful items to a large extent. Real-world results follow the same trend. \ourmethodconserv{} achieves a $12\%$ higher win rate than the best SOTA baseline, ReContriever with LLaMA 70B, while \ourmethodlessconserv{} exceeds GPT4-FewShot by $11\%$ and BM25/ReContriever by $13\%$ with GPT-4. The \textit{All} baseline performs reasonably with only 50 items, but prior work shows scaling degrades performance \citep{rajeev2025cats}. See Appendix~\ref{appx:breakdown_results} for breakdowns. These results were validated with a small-scale human evaluation as well (\ref{appx:H3_human_eval}). These results empirically validate Theorem 3.1: entropy-reducing memory subsets consistently improve response quality, confirming that predictive entropy serves as a reliable proxy for user utility.

\textbf{\ourmethodmodel{} performs similarly to \ourmethodutility{}.} We find, that similar to the results of H2, the downstream results using \ourmethodutility{} are most similar to \ourmethodconserv{}. Also, in general, \ourmethodutility{} performs very strongly against all baselines, with over $10\%$ better winrate against almost all baselines when generating with LLaMA 70B, and all low-level baselines (None, All, Random, and BM25) when generating with GPT-4.

\textbf{\ourmethodmodel{} reduces cost compared to other under-performing methods.}
Because personalization may be applied to every model input, efficiency is critical. We track both the number of memory items selected and the estimated monthly inference cost, computed as the extra tokens from selected items multiplied by a pricing model of \$0.01 per 1K tokens for one million daily queries (current GPT-4 rates). \cref{fig:avg_tokens} shows that \ourmethodmodel{} select substantially fewer memory items than most baselines, translating to significant cost savings (\cref{fig:cost_analysis}): \ourmethodconserv{} costs \$6.9K/month and \ourmethodlessconserv{} costs \$8.4K/month, compared to \$165.3K for All and \$18.6K for the competitive ReContriever. Both \ourmethod{} variants are cheaper than GPT4-FewShot, even when the additional cost of memory subset selection required by GPT4 baselines is excluded. Thus, \ourmethod{} balances informativeness and efficiency, enabling scalable personalization.

\section{Conclusion} We introduced \ourmethod{}, an information-theoretic method for user memory selection that leverages response-aware utility as a more reliable signal for personalization than input-similarity. Our results show that \ourmethod{} accurately identifies when personalization is needed, selects memory subsets that align better with human judgment than SOTA baselines, and improves downstream generation quality while reducing computational cost by up to $95\%$. Future work includes applying \ourmethod{} to multi-turn conversations, enhancing \ourmethodmodel{} with stronger user profile generation, and extending evaluation across a broader range of LLMs.

\section*{Acknowledgments}
We thank Philippe Laban, Tobias Schnabel, and Tarek Naous for their valuable discussions and insightful feedback throughout this project. 

\bibliography{icml}
\bibliographystyle{icml2026}

\newpage
\appendix
\onecolumn
\section{Additional Experimentation}  
In this section, we present additional experiments to supplement our main results.   

\subsection{Robustness Analysis of Utility Scores}\label{appx:robust_utility_scores}  
In H1, we hypothesized that \ourmethod-Utility, the maximum utility scores, can reliably distinguish between user inputs that would benefit from personalization and those that would not. Here, we further investigate how hyperparameters in the \ourmethod-Utility algorithm affect this reliability and provide justification for these effects. Specifically, we examine the influence of model type, model size, model family, decoding method type, maximum number of tokens generated, number of Monte Carlo samples, and number of user profiles. For each analysis, all other variables are held constant using the LLaMa base model, 8B size, sampling decoding, $T=20$ tokens, $n=1$ sample, and averaged over $n=3$ user profiles.  

\paragraph{Hyperparameter: Model Type}  
We first compare the LLaMa-8B base model to an LLaMa-8B instruction-tuned model. As shown in \cref{fig:hyperparam_modeltype}, the base model produces more separated distributions of utility scores than the instruct model, with a Jensen-Shannon divergence score of 0.56 compared to 0.32 for the instruct model. This is consistent with prior research showing that instruction-tuning tends to pre-sharpen output distributions \citep{10.5555/3692070.3693952}. Based on this observation, we chose to use the base model for our experiments.

\begin{figure}[h!]
    \centering
    \begin{subfigure}{0.45\textwidth}
        \centering
        \includegraphics[width=\linewidth]{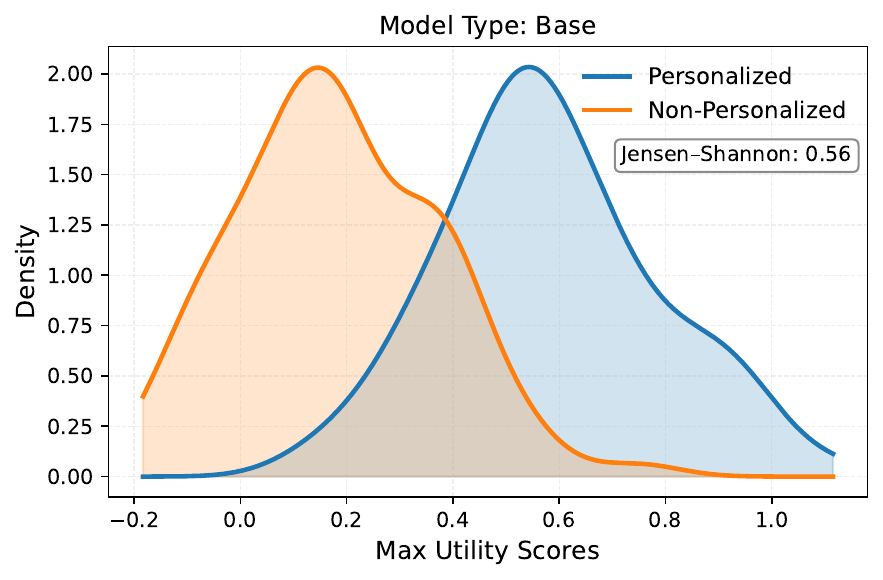}
        \caption{Model Type: Base}
    \end{subfigure}
    \hfill
    \begin{subfigure}{0.45\textwidth}
        \centering
        \includegraphics[width=\linewidth]{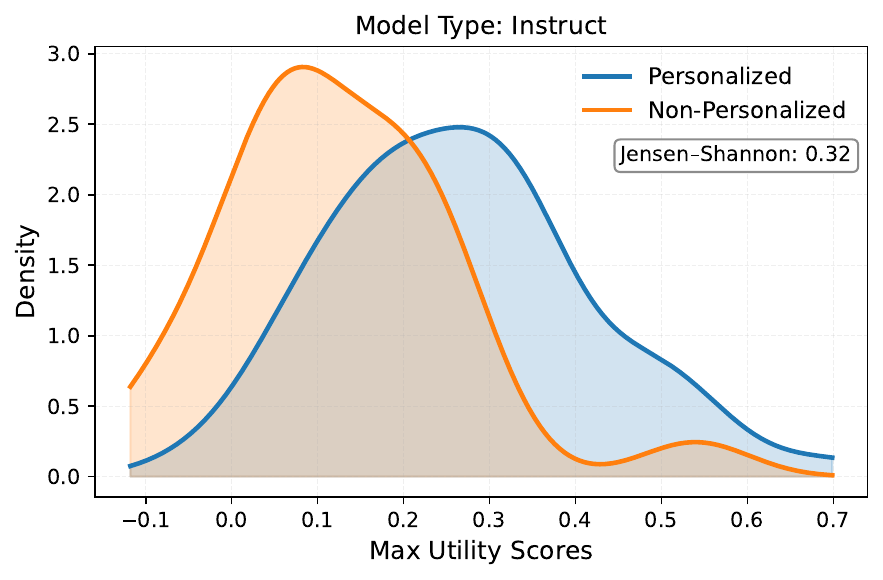}
        \caption{Model Type: Instruct}
    \end{subfigure}
    \caption{Distribution of maximum utility scores for a personalized (blue) and non-personalized (orange) dataset differing by type of model used for entropy calculation. }
    \label{fig:hyperparam_modeltype}
\end{figure}

\paragraph{Hyperparameter: Model Size}  
Next, we examined the effect of model size on the entropy calculation. As shown in \cref{fig:hyperparam_modelsize}, the distribution of utility scores changes when using LLaMA-1B compared to LLaMA-8B. While the impact is less pronounced than the effect of model type, we observe an increase of 0.11 in the Jensen-Shannon divergence score when using the 8B model. We hypothesize that larger models are better able to capture subtle and complex patterns in their response distributions. Given that the improvement is relatively modest, we opted to use the 8B model to balance performance with computational efficiency and avoid the added cost of larger models.

\begin{figure}[h!]
    \centering
    \begin{subfigure}{0.45\textwidth}
        \centering
        \includegraphics[width=\linewidth]{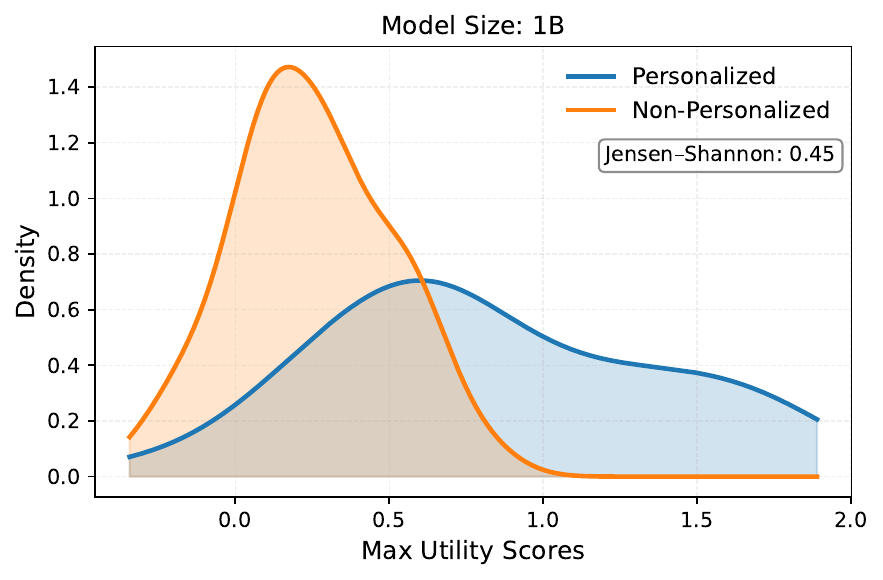}
        \caption{Model Size: 1B}
    \end{subfigure}
    \hfill
    \begin{subfigure}{0.45\textwidth}
        \centering
        \includegraphics[width=\linewidth]{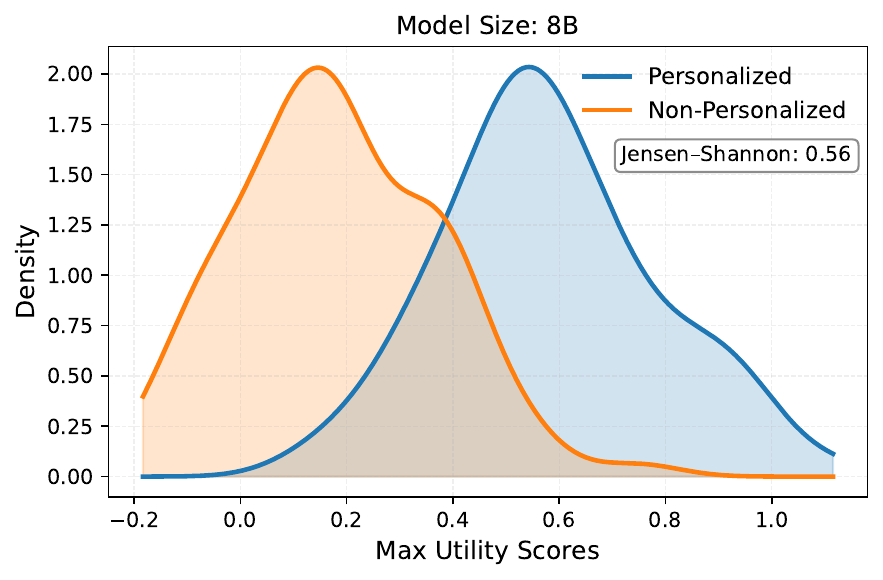}
        \caption{Model Size: 8B}
    \end{subfigure}
    \caption{Distribution of maximum utility scores for a personalized (blue) and non-personalized (orange) dataset differing by size of model used for entropy calculation. }
    \label{fig:hyperparam_modelsize}
\end{figure}

\paragraph{Hyperparameter: Model Family}\label{appx:model_family_ablation} 
We next examine the effect of model family, comparing Llama-8b \cite{meta_llama_3_1_8b_instruct}, Mistral-7b \cite{mistral_7b}, and Gemma-9B \cite{gemma-9B}. As shown in \cref{fig:hyperparam_family}, we find that Mistral model performs similar to Llama (0.41 vs. 0.56), where the Gemma model performs badly (0.17). This could indicate that some model families better capture human preference than others. Based on this observation, we use Llama-7B in our experiments.

\begin{figure}[!ht]
    \centering
    \begin{subfigure}{0.30\textwidth}
        \centering
        \includegraphics[width=\linewidth]{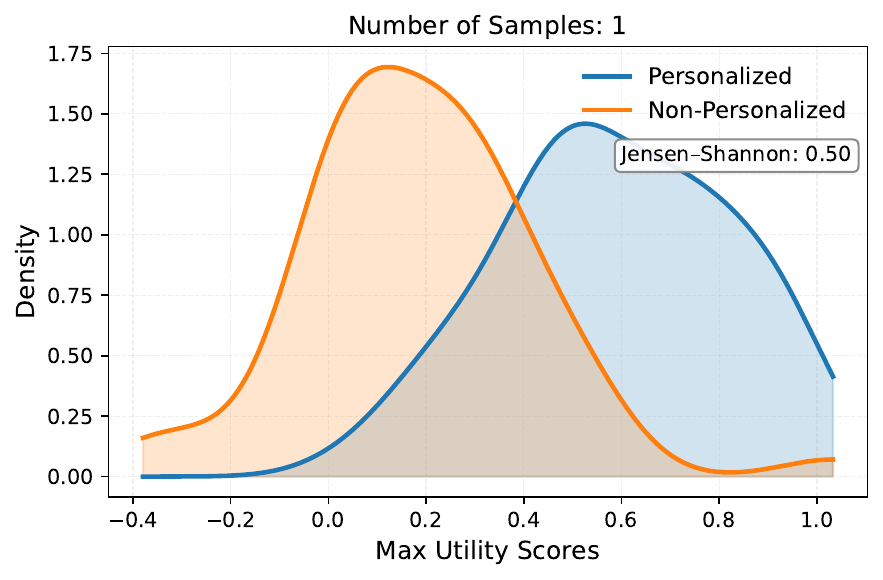}
        \caption{Model Family: Llama}
    \end{subfigure}
    \hfill
    \begin{subfigure}{0.30\textwidth}
        \centering
        \includegraphics[width=\linewidth]{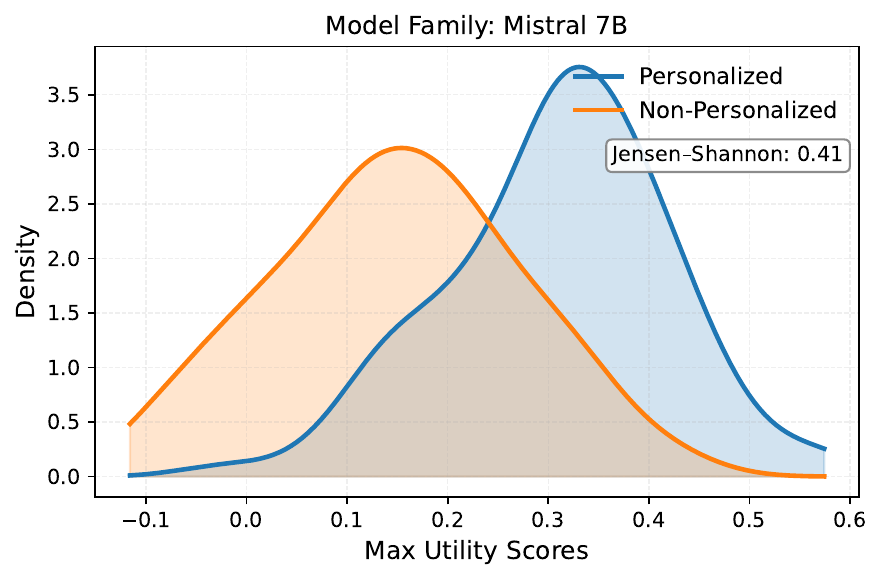}
        \caption{Model Family: Mistral}
    \end{subfigure}
    \hfill
    \begin{subfigure}{0.30\textwidth}
        \centering
        \includegraphics[width=\linewidth]{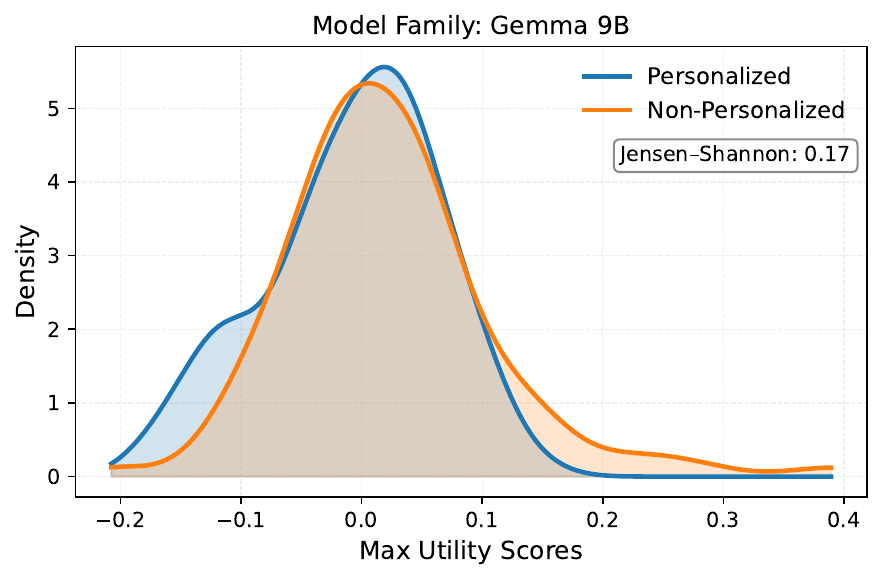}
        \caption{Model Family: Gemma}
    \end{subfigure}
    \caption{Distribution of maximum utility scores for a personalized (blue) and non-personalized (orange) dataset calculated using models of different families. }
    \label{fig:hyperparam_family}
\end{figure}

\paragraph{Hyperparameter: Decoding Type}  
We next examine the effect of decoding type on the sequential average token entropy, comparing greedy and sampling decoding. As shown in \cref{fig:hyperparam_decodingtype}, sampling decoding produces more diverse distributions than greedy decoding. This is expected, as sampling introduces stochasticity into token selection, allowing the model to explore multiple plausible continuations, whereas greedy decoding always selects the highest-probability token at each step, leading to more concentrated and less variable distributions. Based on this observation, we use sampling decoding in our experiments.

\begin{figure}[h!]
    \centering
    \begin{subfigure}{0.45\textwidth}
        \centering
        \includegraphics[width=\linewidth]{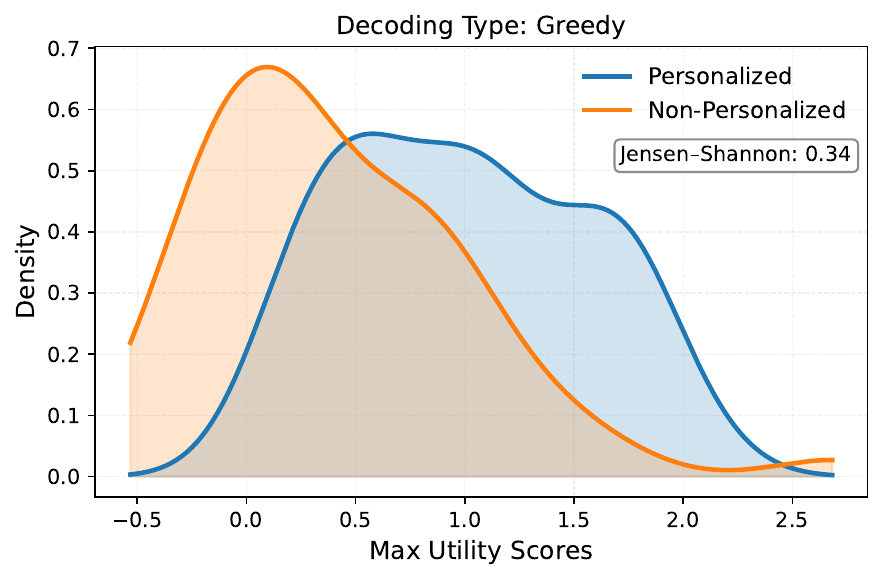}
        \caption{Decoding Type: Greedy}
    \end{subfigure}
    \hfill
    \begin{subfigure}{0.45\textwidth}
        \centering
        \includegraphics[width=\linewidth]{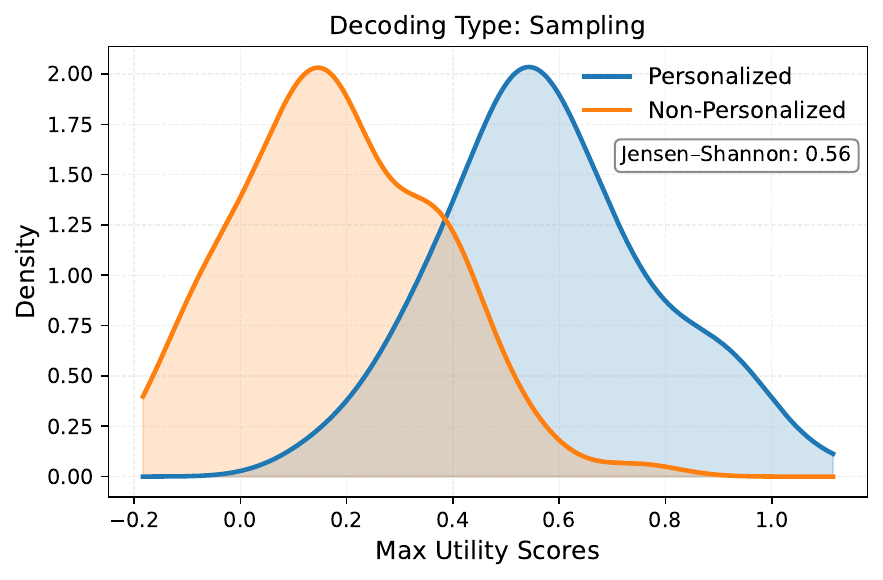}
        \caption{Decoding Type: Sampling}
    \end{subfigure}
    \caption{Distribution of maximum utility scores for a personalized (blue) and non-personalized (orange) dataset differing by decoding type. }
    \label{fig:hyperparam_decodingtype}
\end{figure}

\paragraph{Hyperparameter: Number of Tokens}  
As noted in the main text, we limit the number of tokens used in the calculation of \cref{eqn:method_utility} to avoid bias introduced by longer generations, since longer outputs tend to have lower entropy. To study this, we randomly sampled $n=10{,}000$ responses from the OpenAssistant dataset \citep{kopf2023openassistant}, using only the first human input in each conversation that was labeled as English and non-toxic. We used this dataset since it is a real world dataset that could provide a wide diversity in types of user-LLM interactions. We then computed the average token-level entropy by token index, capped at 100 tokens. As shown in \cref{fig:hyperparam_avg_token_entropy}, entropy begins to decrease significantly around $T=20$ tokens, indicating that longer generations introduce a bias toward lower entropy.  
\begin{figure}
    \centering
    \includegraphics[width=0.5\linewidth]{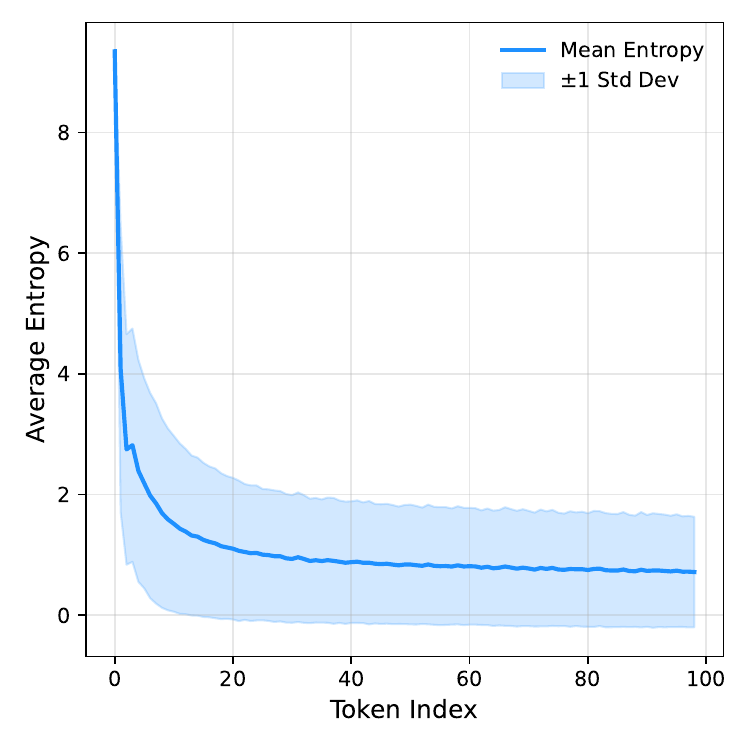}
    \caption{Average entropy by the index of the token being generated.}
    \label{fig:hyperparam_avg_token_entropy}
\end{figure}

Further analysis in \cref{fig:hyperparam_ntoken} compares the distribution of maximum utility scores when using 20 versus 100 maximum tokens for entropy calculations. Using 100 tokens reduces the Jensen-Shannon divergence by 0.18 compared to using only 20 tokens. Based on these results, we limit the calculation to the first 20 tokens in our experiments.

\begin{figure}[!ht]
    \centering
    \begin{subfigure}{0.45\textwidth}
        \centering
        \includegraphics[width=\linewidth]{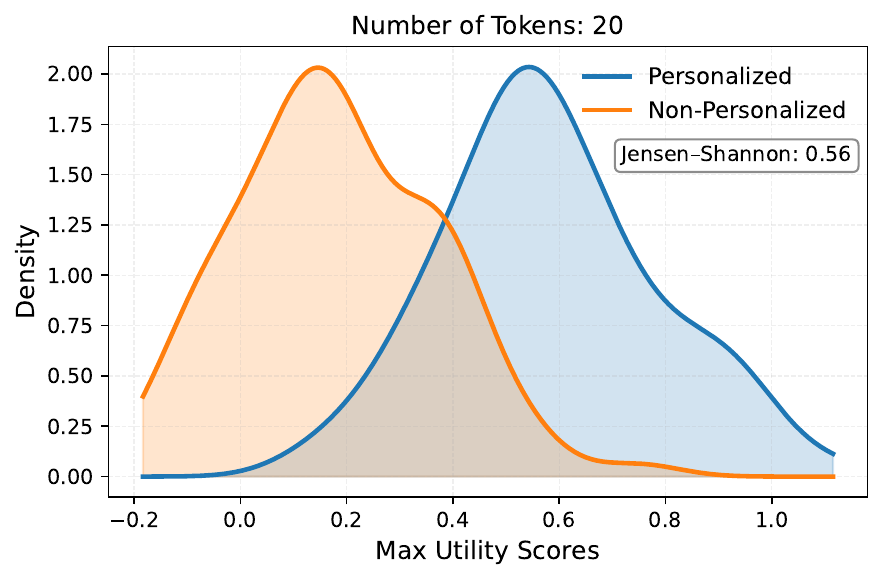}
        \caption{Number of Tokens: 20}
    \end{subfigure}
    \hfill
    \begin{subfigure}{0.45\textwidth}
        \centering
        \includegraphics[width=\linewidth]{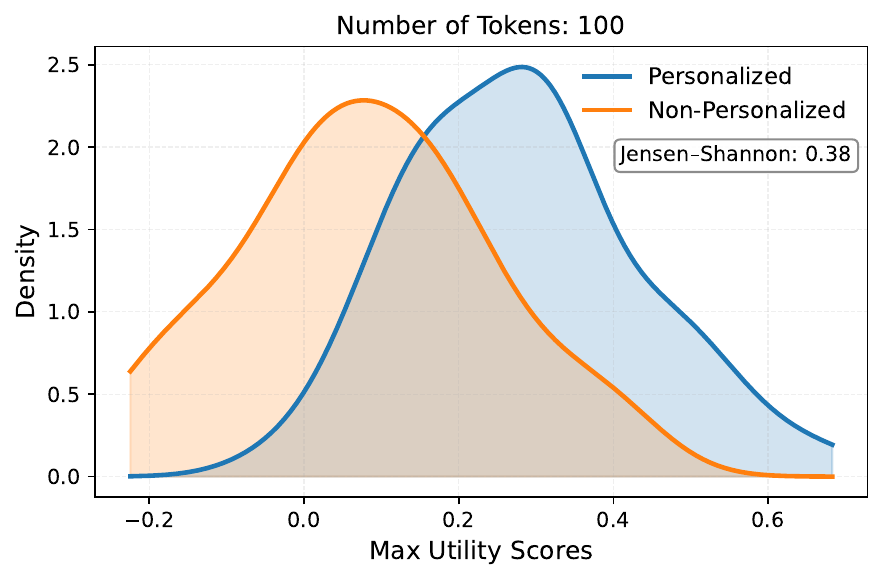}
        \caption{Number of Tokens: 100}
    \end{subfigure}
    \caption{Distribution of maximum utility scores for a personalized (blue) and non-personalized (orange) dataset using different number of tokens. }
    \label{fig:hyperparam_ntoken}
\end{figure}

\paragraph{Hyperparameter: Number of Samples}  
Next, we examine the effect of the number of Monte Carlo samples used in \cref{eqn:method_utility} on the robustness of the utility score. As shown in \cref{fig:hyperparam_nsamples}, increasing the number of samples improves the separation between personalized and non-personalized inputs, which aligns with expectations since the estimate should converge toward the true parameter as the sample size increases if the samples are i.i.d. However, this comes with a tradeoff: more samples require additional computation, longer runtime, and higher resource costs. Balancing these considerations, we chose $n=5$ samples for our experiments.
\begin{figure}[h!]
    \centering
    \begin{subfigure}{0.30\textwidth}
        \centering
        \includegraphics[width=\linewidth]{images/hyperparam_analysis/mi_scores_distr_1_nsamples.pdf}
        \caption{Number of Samples: 1}
    \end{subfigure}
    \hfill
    \begin{subfigure}{0.30\textwidth}
        \centering
        \includegraphics[width=\linewidth]{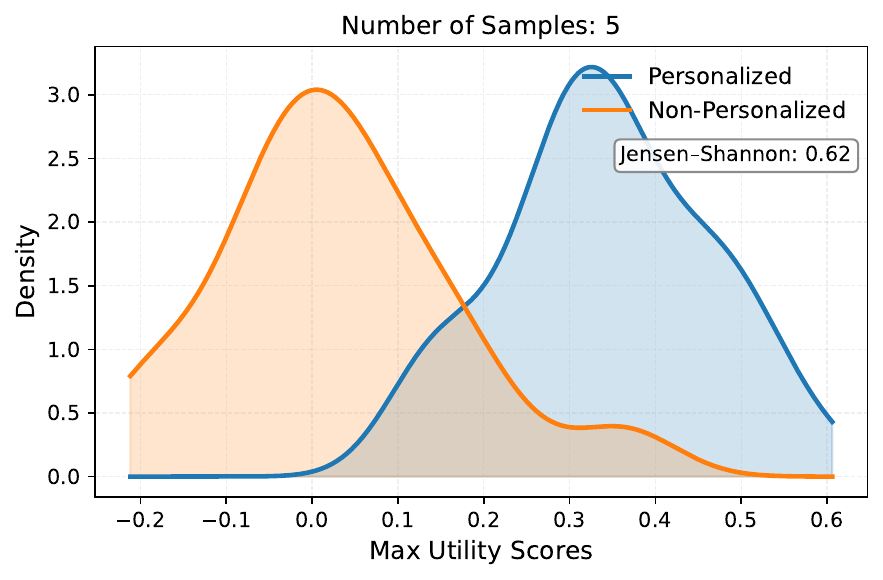}
        \caption{Number of Samples: 5}
    \end{subfigure}
    \hfill
    \begin{subfigure}{0.30\textwidth}
        \centering
        \includegraphics[width=\linewidth]{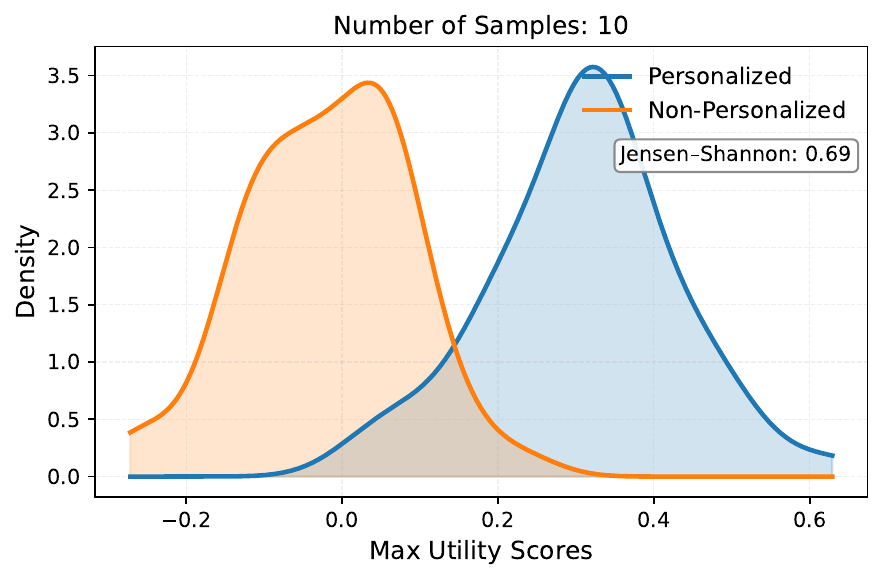}
        \caption{Number of Samples: 10}
    \end{subfigure}
    \caption{Distribution of maximum utility scores for a personalized (blue) and non-personalized (orange) dataset differing number of Monte Carlo samples average over. }
    \label{fig:hyperparam_nsamples}
\end{figure}

\paragraph{Hyperparameter: Number of User Profiles} Lastly, we examine how the robustness of the average maximum utility score changes as we increase the number of (query,user profile) pairs included in the average. Consistent with estimation theory, we find that the difference between the personalized and non-personalized distributions grows as the sample size increases. Notably, there is a substantial jump of 0.22 in Jensen–Shannon divergence when increasing from 1 to 5 users, but only a marginal increase of 0.01 when going from 5 to 10 users. Based on this, we choose $u=5$ users for constructing our training data to balance performance with cost of inference.
\begin{figure}[h!]
    \centering
    \begin{subfigure}{0.30\textwidth}
        \centering
        \includegraphics[width=\linewidth]{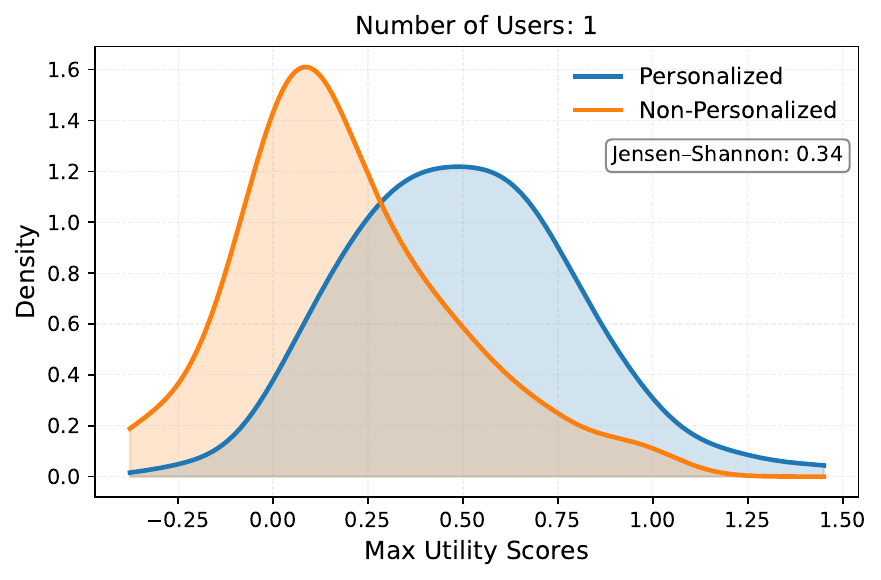}
        \caption{Number of Users: 1}
    \end{subfigure}
    \hfill
    \begin{subfigure}{0.30\textwidth}
        \centering
        \includegraphics[width=\linewidth]{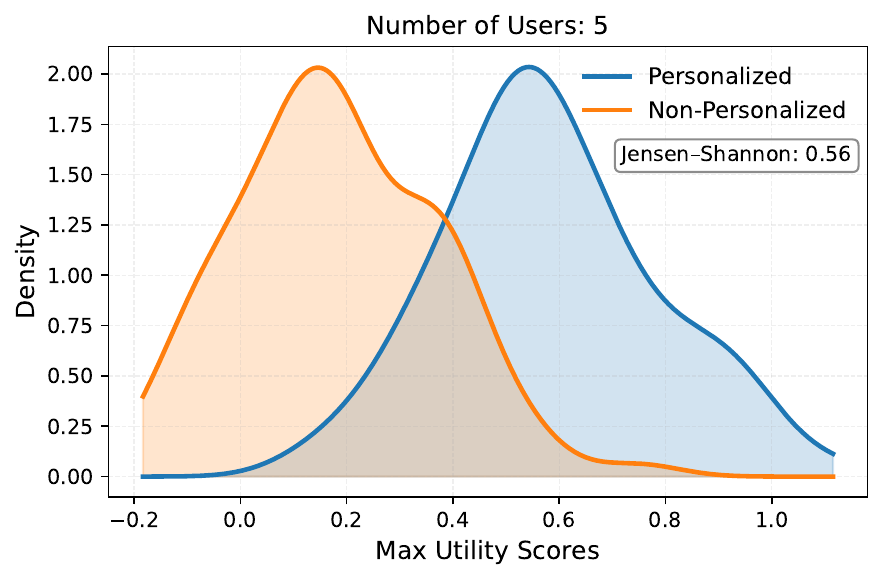}
        \caption{Number of Users: 5}
    \end{subfigure}
    \hfill
    \begin{subfigure}{0.30\textwidth}
        \centering
        \includegraphics[width=\linewidth]{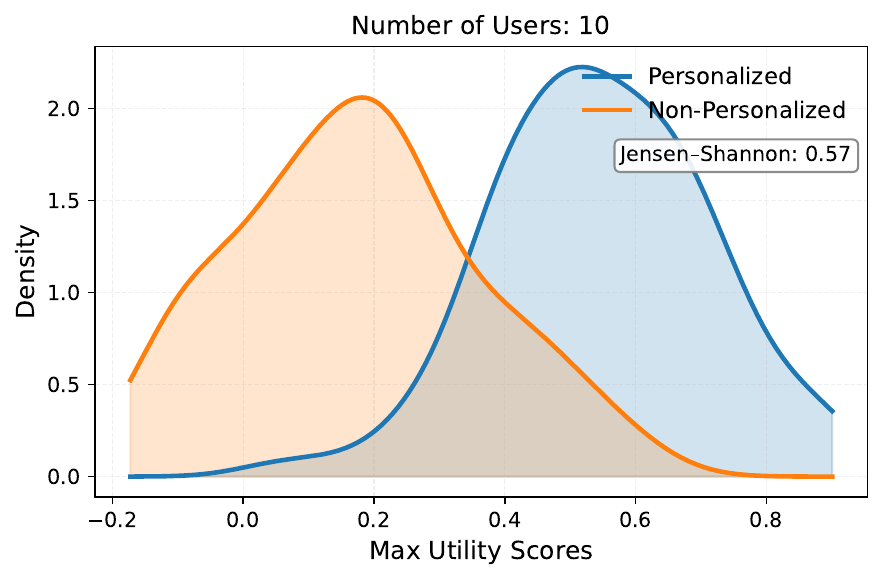}
        \caption{Number of Users: 10}
    \end{subfigure}
    \caption{Distribution of maximum utility scores for a personalized (blue) and non-personalized (orange) dataset differing by number of users averaged over. }
    \label{fig:hyperparam_nusers}
\end{figure}

\subsection{Human-Chosen Memory Subset Performance on Downstream Response}\label{appx:human_features_response_analysis}
In H2, we analyze the correlation between the memory subsets selected by \ourmethod{}-Utility, \ourmethod{}-Model and the baselines compared to human-selected subsets. Human labels are treated as the gold standard, as we hypothesize that they would lead to better downstream responses. To evaluate this, we calculate the generation win rate of responses produced using human-selected features compared to those produced by \ourmethod-Model and the baselines. We use the same evaluation technique outlined in \cref{appx:exp_details}. As shown in \cref{tab:winrate_humans}, human-selected memory subsets yield better downstream responses than all baselines when using the less capable response model for generation (LLaMA 70B). With the closes' performance being RUMS-Binary. When generating responses with a more capable model, GPT-4, we still see slight better performance using the human selected memory items, although much less than when using a less capable mode. This shows that human chosen memory items should be used in H2 as a measure of a good subset of memory. 

\begin{table}[h!]
    \centering
\caption{Win rate of responses preferred using human selected memory subsets compared to other methods. A number above 0 is the percentage of responses generated by Humans chosen subsets that were preferred.}
    \label{tab:winrate_humans}
    \begin{tabular}{l c c}
        \hline
         & LLaMA 70B & GPT-4 \\
        \hline
        Method & Human Win Rate & Human Win Rate \\
        \hline
        None & \cellval{0.31} & \cellval{0.16}\\
        All & \cellval{0.28} & \cellval{-0.26}\\
        Random & \cellval{0.42} & \cellval{0.06}\\
        Semantic Similarity & \cellval{0.38} & \cellval{0.01}\\
        BM25 & \cellval{0.18}& \cellval{0.26}\\
        ReContriever& \cellval{0.52} & \cellval{0.01}\\
        GPT4-Few & \cellval{0.27} & \cellval{0.01}\\
        RUMS-Binary & \cellval{0.01} & \cellval{-0.05}\\
        RUMS-Multi& \cellval{0.39} & \cellval{0.02} \\
        \hline
    \end{tabular}

\end{table}
 
\subsection{Analysis of Approximation Gap Between RUMS-Utility and RUMS-Models} \label{appx:H2_approx_gap}

\subsubsection{Analysis of \ourmethodmodel{} Signal Learning}
To better understand what signals the \ourmethodmodel{} are learning, we add an analysis which explores the correlation between memory items selected (included) and not selected (excluded) by these models, and the connection to the true RUMS-Utility score. Since we are interested in personalization, we use the PersonaFeedback test set (n=100).

In \cref{fig:include_exclude_dist_bin} and \cref{fig:include_exclude_dist_multi}, we show the distribution of average utility score per sample. In \cref{fig:include_exclude_diff_bin} and \cref{fig:include_exclude_diff_multi}, we show the distribution of difference of average utility score for included - excluded features per sample. We find the distribution of included and excluded utility scores are statistically different (p-value $\leq 0.03$, Mann-Whitney test). Similarly, We find that the distribution of differences between the average utility score for items the RUMS-model included and the average utility score of memory items the RUMS-Model excluded is greater than 0 (p-value $<.001$, Wilcoxon one-sided test), meaning that, on average, we find that memory items the RUMS-Model include have a truly higher utility score than those it does not. This lends strong evidence that \ourmethodmodel{} correctly learn signals of the raw utility scores. 

\begin{figure}[t]
\centering
\begin{minipage}{0.48\linewidth}
  \centering
  \includegraphics[width=\linewidth]{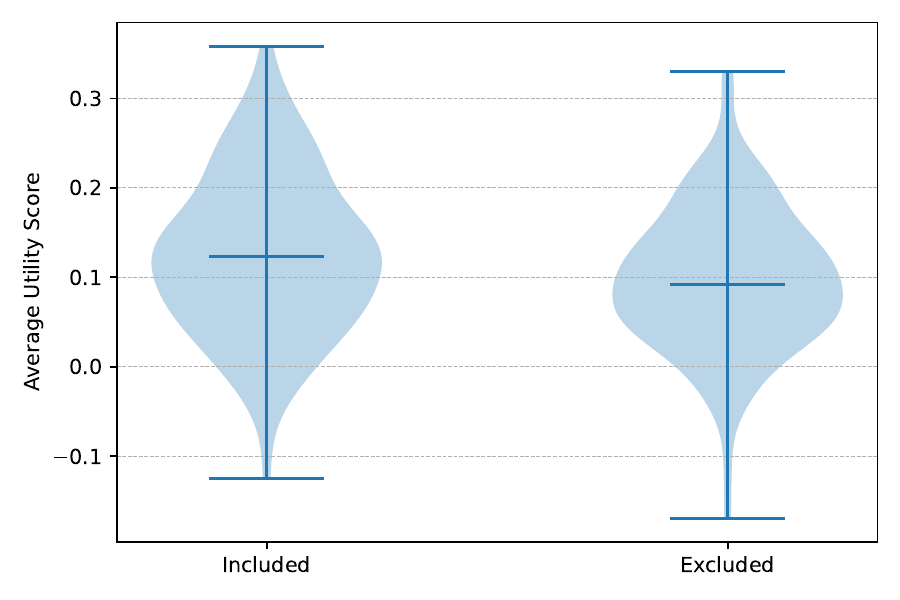}
  \caption{RUMS-Binary: Distribution}
  \label{fig:include_exclude_dist_bin}
\end{minipage}\hfill
\begin{minipage}{0.48\linewidth}
  \centering
  \includegraphics[width=\linewidth]{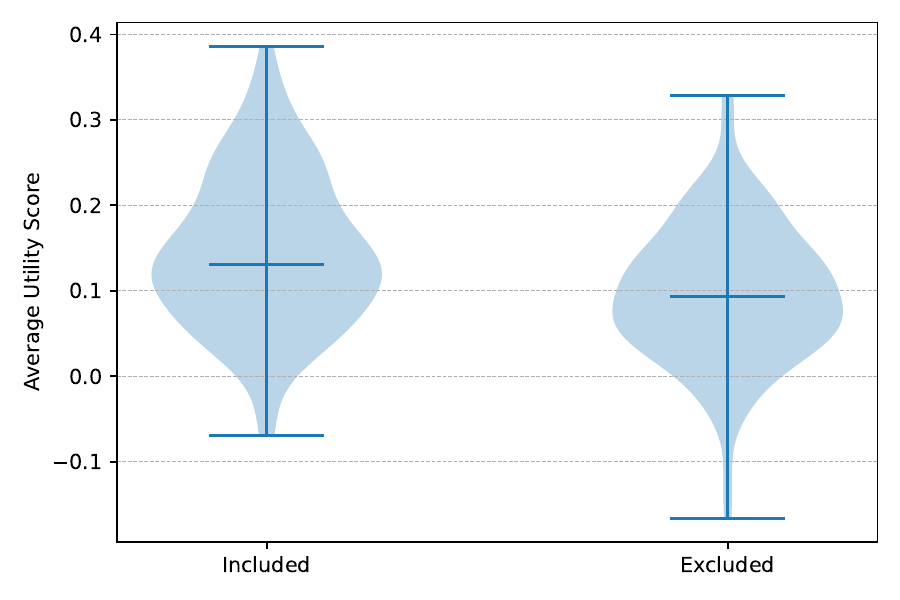}
  \caption{RUMS-Multi: Distribution}
  \label{fig:include_exclude_dist_multi}
\end{minipage}
\end{figure}

\begin{figure}[t]
\centering
\begin{minipage}{0.48\linewidth}
  \centering
  \includegraphics[width=\linewidth]{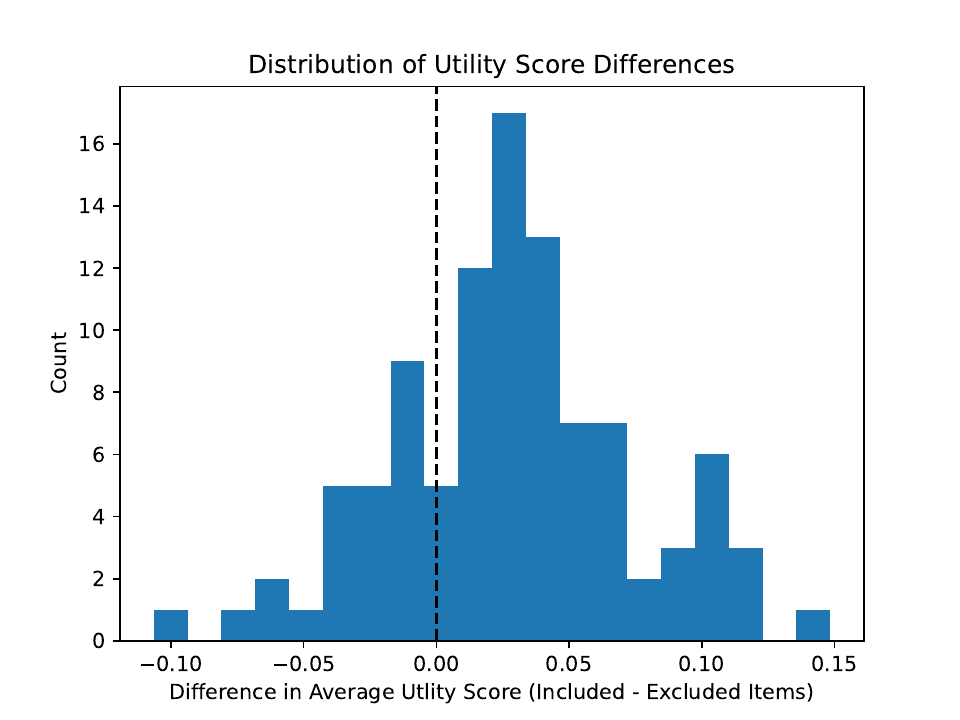}
  \caption{RUMS-Binary: Difference}
  \label{fig:include_exclude_diff_bin}
\end{minipage}\hfill
\begin{minipage}{0.48\linewidth}
  \centering
  \includegraphics[width=\linewidth]{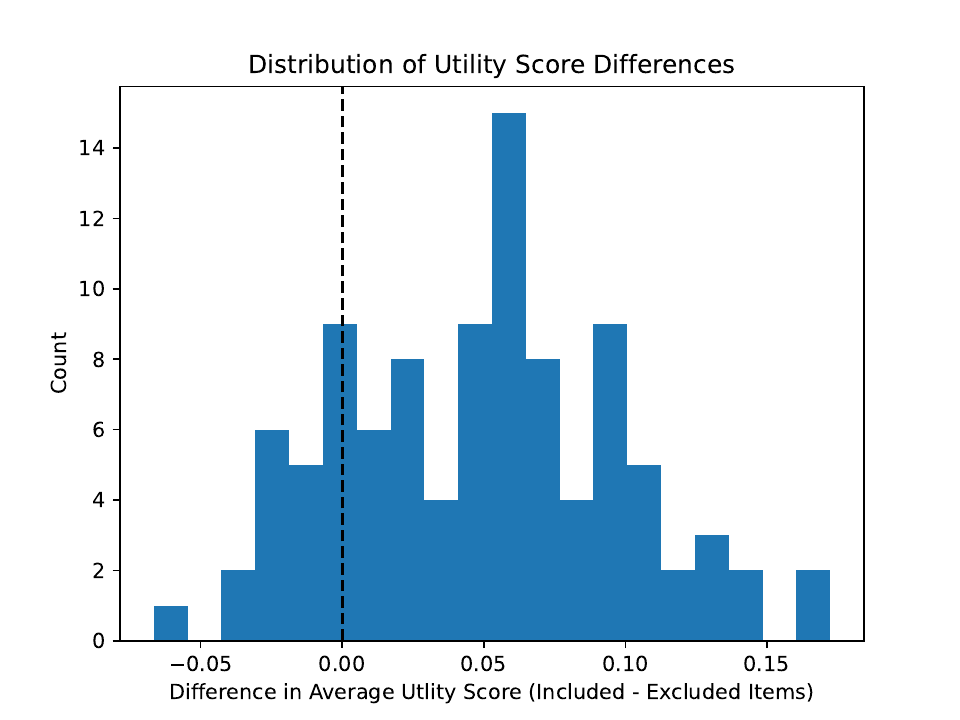}
  \caption{RUMS-Multi: Difference}
  \label{fig:include_exclude_diff_multi}
\end{minipage}
\end{figure}

\subsubsection{Comparing RUMS-Utility to \ourmethodmodel{} on H2} \ourmethod
As mentioned in the main paper, we opt to use RUMS-Models, rather than the raw RUMS-Utility, due to its high efficacy and large decrease in computational compute. This is a common approach when using BOED methods, due to algorithmic components which make it time intensive. Although our H2 results show that RUMS-Models have similar, but lower, efficacy compared to RUMS-Utility, we want to further analyze this approximation gap. 

To do this, we expand the results for H2. First, we calculated a majority feature list from the RUMS-Utility chosen data for the samples in this section, by including all features that were in a majority (>=5) samples using the RUMS utility scores directly over 5 diverse profiles. Then, we calculated the same metrics (precision, recall, and F1), of this new gold label group to feature list selected by RUMS-Binary and RUMS-Multi. 

{\setlength{\textfloatsep}{0.01pt} %
\begin{table}[t]
\centering
\scriptsize
\setlength{\tabcolsep}{1pt} %
\caption{Precision, recall, and F1 scores of memory items selected compared to human gold label annotation across datasets and methods.}
\label{tab:H2_approx_gap}
\begin{tabular}{l l c c}

\toprule
\textbf{Dataset} & \textbf{Metric} & \ourmethodconserv{} & \ourmethodlessconserv{}\\
\midrule
 & Precision & 0.73 & 0.33 \\
Synthetic  & Recall    & 0.90 & 0.44 \\
 & F1        & 0.78 & 0.37  \\
\midrule
 & Precision & 0.26 & 0.35  \\
Real World & Recall    & 0.42 & 0.35 \\
 & F1        & 0.32 & 0.34  \\
\bottomrule
\end{tabular}
\end{table}
} 
In general, we find that the models have high agreement on in distribution datasets (synthetic) but perform moderate on out of distribution dataset (real world). In line with our results for H2, we find that RUMS-Binary has very high correlation with RUMS-Utility with a high F1 of $0.78$ on the synthetic dataset. Although RUMS-Binary only has a moderate F1 of $.37$, we note that most of this comes from a decrease in performance on non-personalized datasets and the over-selection of RUMS-Multi. This is highlighted by the higher recall scores, in general, and slightly lower precision (e.g. $0.9$ vs. $.73$ for RUMS-Binary). 

\subsection{Approximation of \ourmethodlessconserv{} using \ourmethod{}-Utility}
The RUMS-Utility objective evaluates subsets of memory items jointly, capturing interactions between features. In contrast, RUMS-Binary predicts each item independently. Since conditional mutual information is not generally additive, this raises the question of how much combinatorial information is lost when using the item-wise approximation.

To quantify this gap, we provide both a theoretical decomposition and empirical measurement.

\paragraph{Theoretical Decomposition} 
Let $S = \{m_1, ..., m_k\}$ be a subset of memory items. The utility function used in RUMS is
\begin{align*}
    U(S) = I(Y;M_S \vert x).
\end{align*}

Using the chain rule of conditional mutual information,
\begin{align*}
    I(Y;M_S \vert x) = \sum_{i=1}^k I(Y;m_i \vert x, m_{<i}),
\end{align*}
where $m_{<i} = \{m_1,...,m_{i=1}\}$.

Next, the independent approximation used by RUMS-Binary replaces conditional gains with marginal gains,
\begin{align*}
    I(Y;m_i \vert x, m_{<i}) \approx I(Y;m_i \vert x).
\end{align*}
This yields the additive approximation,
\begin{align*}
    \tilde{U}(S) = \sum_{i=1}^k I(Y;m_i \vert x).
\end{align*}

Therefore, the combinatorial interaction gap is expressed,
\begin{align}
    \Delta(S) = U(S) - \tilde{U}(S) = \sum_{i=1}^k \left(I(Y;m_i \vert x, m_{<i}) - I(Y;m_i \vert x) \right). \label{eqn:approx_gap}
\end{align}
Therefore, \cref{eqn:approx_gap} shows that the total approximation error equals the sum of per-feature interaction terms.

\paragraph{Empirical Measurement} To better understand this approximation error in our experimentation, we estimate interaction effects using the utilities computed by RUMS-Utility for the test set of human annotated examples (n=60) using only the personalized datasets. For each query and selected subset S in our four test sets, we compute two metrics.
\begin{enumerate}
    \item \textbf{Subset Interaction Gap}: $\Delta(S) = U(S) - \sum_{i \in S} U(\{\})$, which measures the total combinatorial gain.
    \item \textbf{Sequential Conditional Interaction Gap} Given an ordering of memory items, we define $\delta_i = \left[ U(S_i) - U(S_{i-1}\right] - U(\{i\})$, where $S_i$ is the prefix subset. This measures how much the conditional contributions of each memory item deviates from independence. 
\end{enumerate}

\begin{table}
\centering
\begin{tabular}{lp{3cm}p{3cm}}
\hline
Dataset    & Subset Gap & Sequential Gap \\
\hline
Synthetic  & 0.23                & 0.11            \\
\hline
Real World & 1.01                & 0.09 \\     
\hline
\end{tabular}
\caption{Results from the analysis of the combinatorial approximation gap between RUMS-Utility and RUMS-Binary. We show the average normalized subset interaction gap and the average sequential conditional interaction gap for personalized datasets. }
\label{tab:combinatorial_approx_gap}
\end{table}

\cref{tab:combinatorial_approx_gap} summarizes the measured combinatorial gaps across both synthetic and real-world datasets. We find the following:
\begin{itemize}
    \item \textbf{Subset interaction gap}: The average normalized subset gap is 0.23 for the synthetic dataset and 1.01 for the real-world dataset, indicating that joint subsets can provide substantial additional utility beyond the sum of individual feature contributions.
    \item \textbf{Sequential conditional interaction gap}: The average per-feature sequential gap is 0.11 for synthetic data and 0.09 for real-world data, showing that individual features’ conditional contributions deviate only slightly from the independent approximation.
\end{itemize}

These results highlight two key patterns. First, joint feature subsets offer additional utility, particularly in real-world datasets where the subset gap is larger. However, despite the larger subset-level gains, sequential conditional gaps remain small, suggesting that no single feature strongly violates the independence assumption.
Together, these findings indicate that while higher-order interactions contribute to total utility, they are distributed across features and individually weak.

The analysis confirms that RUMS-Binary’s item-wise approximation captures most of the subset utility. Even in datasets with substantial subset-level gains, the limited per-feature sequential gaps explain why RUMS-Binary performs well in practice while remaining computationally efficient. This supports its use as a scalable approximation to full subset-based RUMS-Utility selection.

\subsection{\ourmethod{}-Utility Threshold Analysis} \label{appx:threshold_analysis}
\subsubsection{\ourmethod{}-Utility Threshold Effect on H1' Personalization Accuracy} We examined how personalization accuracy (agreement rate and F1), measured in H1', changes as we vary the threshold from 0 to 1 in 0.01 increments, on the test set used in our paper. For both datasets, Synthetic and Real World, we provide a graph comparing the metric (\% agreement or F1) across thresholds. We also mark the optimal threshold (decided on a separate evaluation set), which we used in all experimentation. Lastly, we also show a zoomed-in view around the optimal threshold $\pm 0.005$ for each graph (see \cref{fig:threshold_first} -   \cref{fig:threshold_last}).

In general, we found that personalization decisions are not very sensitive in respect to the threshold chosen. As we can see, for the Synthetic test set, when we change the threshold $+/- 0.1$ we see only about a $15\%$ decrease in agreement rate and only a $10\%$ for the Real World Data. So, although the optimal threshold does seem to be heavily dataset dependent, it does seem moderately robust. 

\begin{figure}[t]
\centering
\begin{minipage}{0.48\linewidth}
  \centering
  \includegraphics[width=\linewidth]{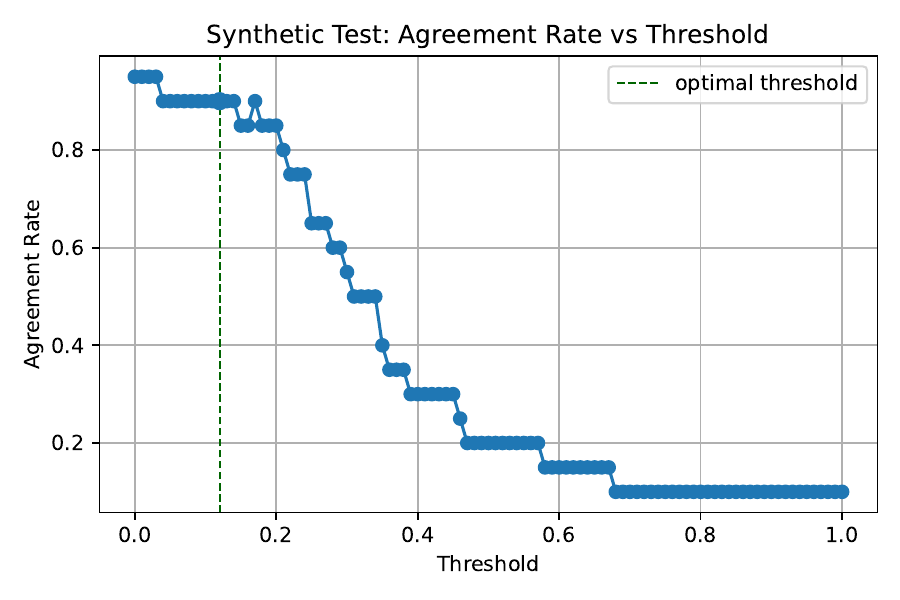}
  \caption{Synthetic: \% Agreement vs. Threshold}
  \label{fig:threshold_first}
\end{minipage}\hfill
\begin{minipage}{0.48\linewidth}
  \centering
  \includegraphics[width=\linewidth]{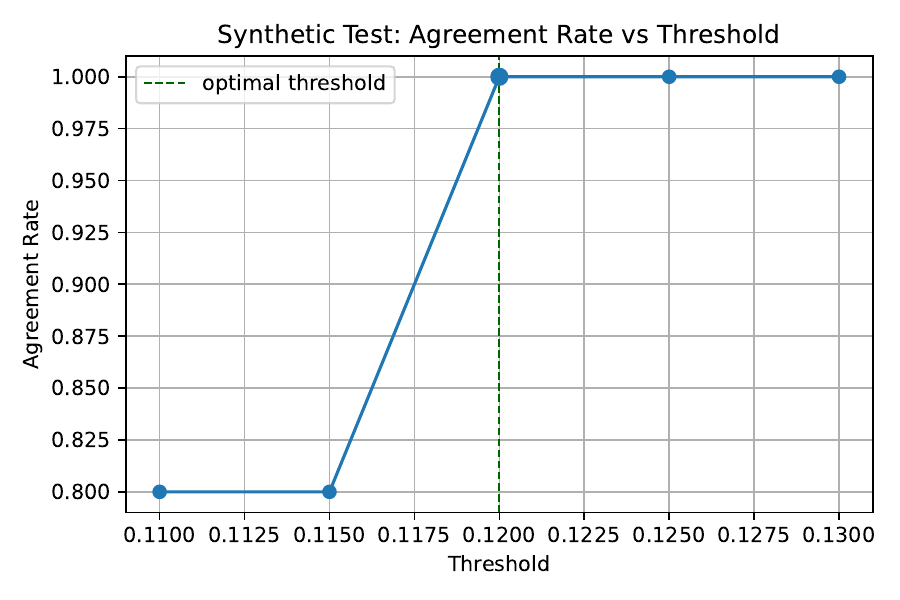}
  \caption{Synthetic: \% Agreement vs. Threshold (zoomed in)}
\end{minipage}
\end{figure}

\begin{figure}[t]
\centering
\begin{minipage}{0.48\linewidth}
  \centering
  \includegraphics[width=\linewidth]{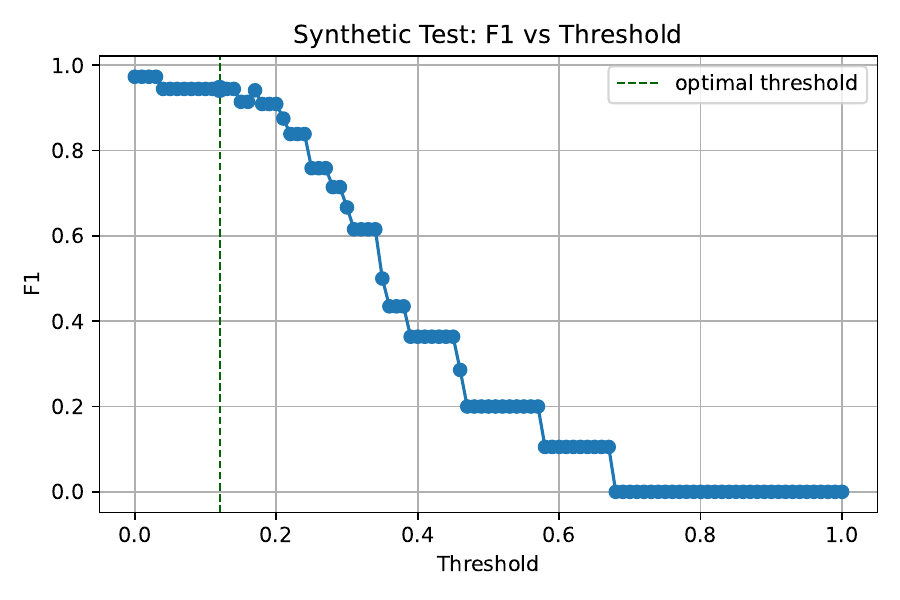}
  \caption{Synthetic: F1 vs. Threshold}
\end{minipage}\hfill
\begin{minipage}{0.48\linewidth}
  \centering
  \includegraphics[width=\linewidth]{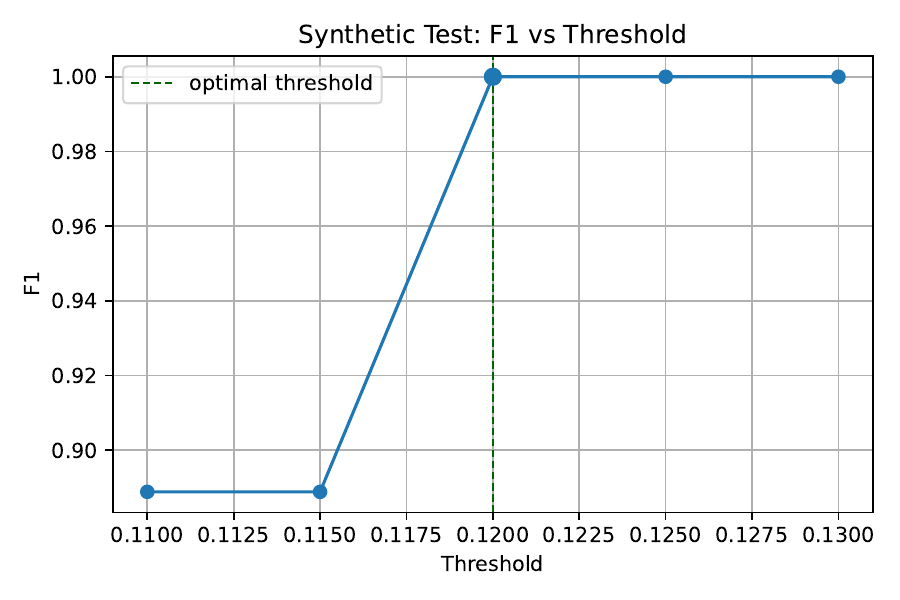}
  \caption{Synthetic: F1 vs. Threshold (zoomed in)}
\end{minipage}
\end{figure}

\begin{figure}[t]
\centering
\begin{minipage}{0.48\linewidth}
  \centering
  \includegraphics[width=\linewidth]{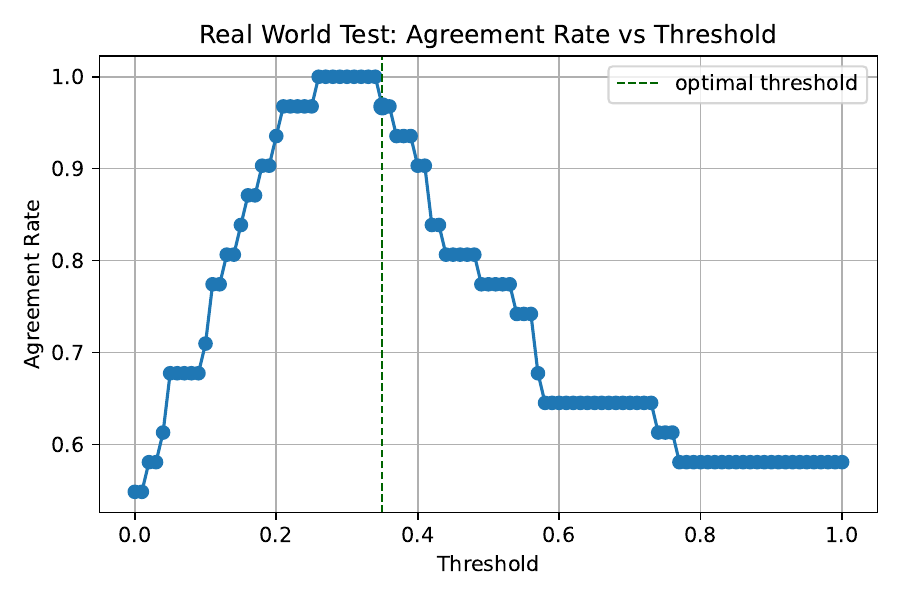}
  \caption{Real World: \% Agreement vs. Threshold}
\end{minipage}\hfill
\begin{minipage}{0.48\linewidth}
  \centering
  \includegraphics[width=\linewidth]{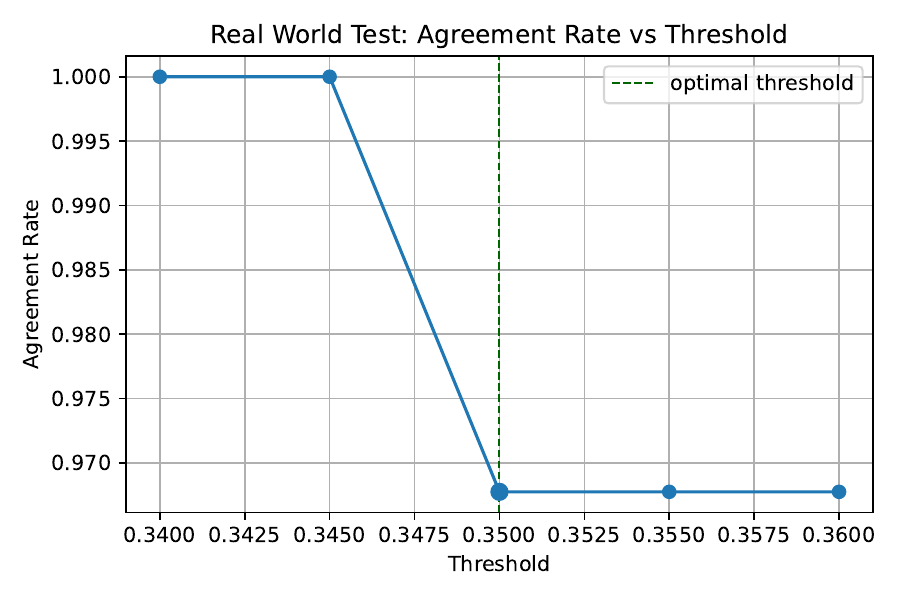}
  \caption{Real World: \% Agreement vs. Threshold (zoomed in)}
\end{minipage}
\end{figure}

\begin{figure}[t]
\centering
\begin{minipage}{0.48\linewidth}
  \centering
  \includegraphics[width=\linewidth]{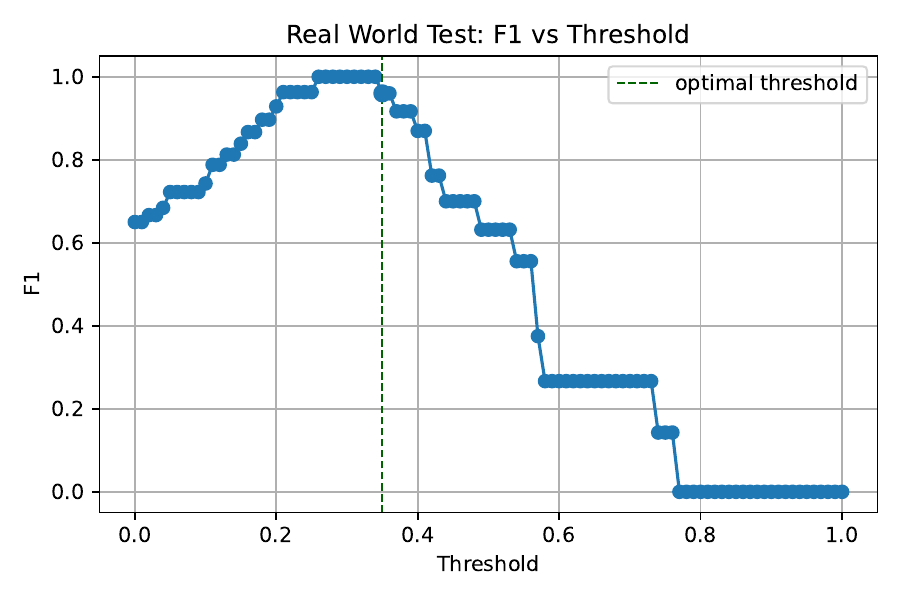}
  \caption{Real World:F1 vs. Threshold}
\end{minipage}\hfill
\begin{minipage}{0.48\linewidth}
  \centering
  \includegraphics[width=\linewidth]{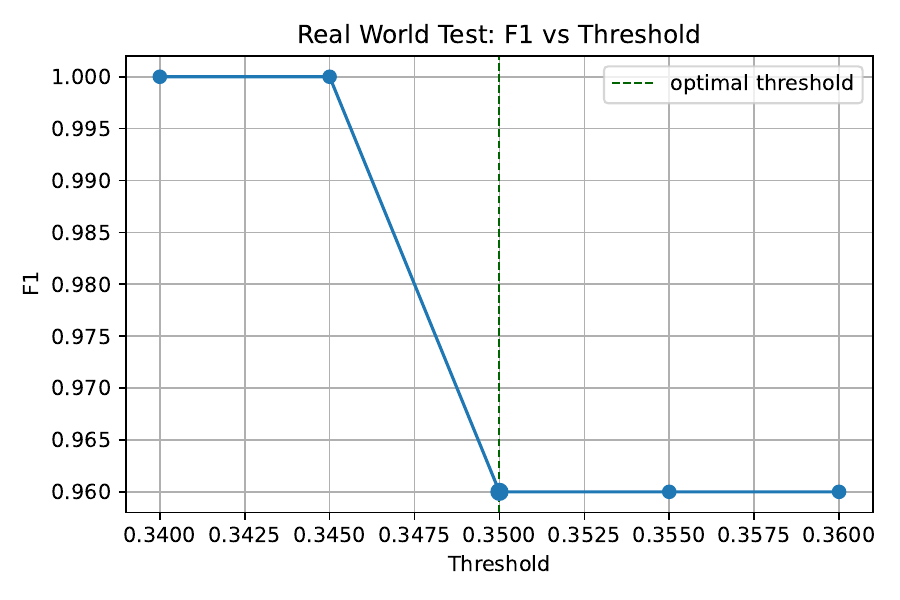}
  \caption{Real World: F1 vs. Threshold (zoomed in)}
  \label{fig:threshold_last}
\end{minipage}
\end{figure}

\subsubsection{User-specific \ourmethod{}-Utility Threshold Effect on Personalization} We explore the idea of adaptive threshold per-user profile might change. To do this, we use a sample of $n=100$ user inputs from FreebaseQA (non-personalized dataset) and PersonaFeedback (personalized dataset) respectively. For $u=10$ static user profiles, we compute the maximum utility score for each. Then, we use the same procedure for finding the optimal threshold in the paper, to find the optimal threshold for each user. We then analyze the distribution of optimal thresholds for robustness using mean and standard deviations.

\begin{figure}
    \centering
    \includegraphics[width=0.5\linewidth]{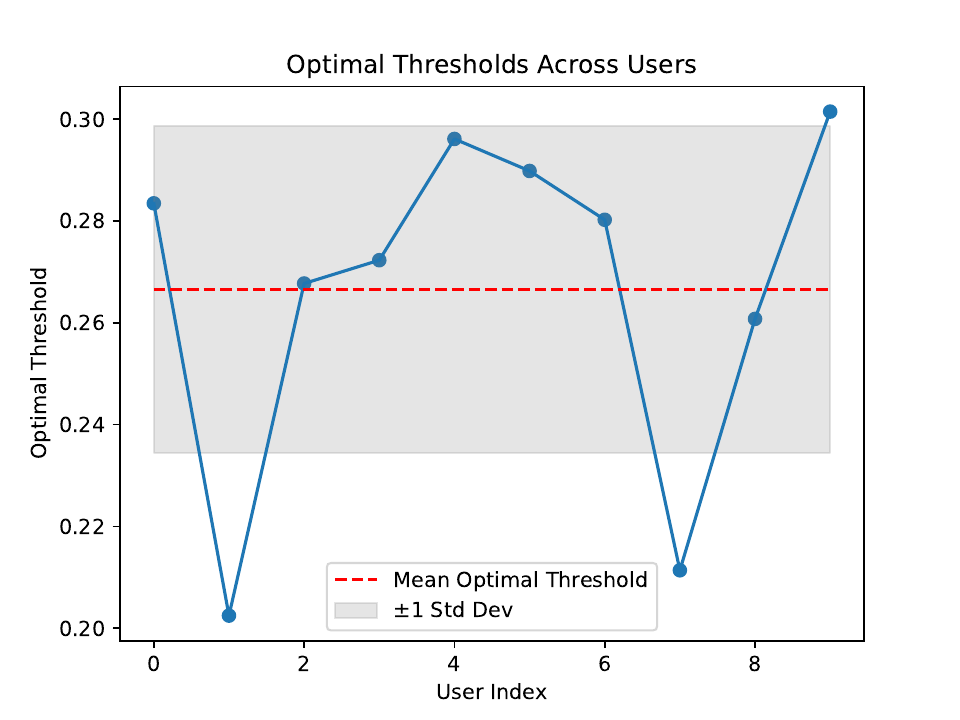}
    \caption{Optimal threshold calculated for each $n=10$ users based on utility scores. It shows the optimal thresholds for the 10 users (solid blue line), as well as the mean threshold (dotted red line) and $\pm 1$ standard deviation (gray shaded area)}
    \label{tab:user_sepecific_threshold}
\end{figure}

In \cref{tab:user_sepecific_threshold}, we present the distribution of optimal thresholds for the 10 users (solid blue line), as well as the mean threshold (dotted red line) and $\pm 1$ standard deviation (gray shaded area). We found that $70\%$ users fell within a threshold of $[0.24,0.30]$. So although there is some outlier profiles, we do not think there is enough variability that compensate for the added computational cost of per-user thresholds. We thank the reviewer for this query and find it interesting so we will add this analysis to our appendix. 

\subsubsection{\ourmethod{}-Utility Threshold Effect on Number of Features Selected}\label{appx:feature_selection_vs_threshold}  
As mentioned in the main text, \ourmethod{}-Utility has a threshold hyperparameter that controls the sensitivity of memory item selection. Adjusting this threshold results in either more or fewer memory items being chosen. \Cref{tab:avg_memory_items} illustrates how the average number of selected memory items changes as the threshold is varied. As the threshold increases, a stronger reduction in entropy is required for a memory item to be selected, leading to smaller average subset sizes. This analysis highlights that the threshold provides a simple yet effective way to balance the tradeoff between including more context and avoiding unnecessary or low-utility memory items.

\begin{table}[h!]
    \centering
     \caption{Average number of memory items selected at different thresholds used in \ourmethod{}-Utility.}
    \label{tab:avg_memory_items}
    \begin{tabular}{c c c c c}
        \hline
        Threshold & \pf{} & \fb{} & \wcp{} & \wcnp{} \\
        \hline
        -100 & 2.76 & 1.83 & 2.77 & 2.70 \\
        0.00    & 2.76 & 1.56 & 2.73 & 2.57 \\
        0.29 & 2.43 & 0.21 & 1.31 & 1.38 \\
        0.50  & 1.11 & 0.00    & 0.29 & 0.27 \\
        0.80  & 0.14 & 0.00    & 0.00    & 0..00    \\
        \hline
    \end{tabular}

\end{table}

\subsection{Examples of Selected Memory Items}\label{appx:H2_qualitative} \Cref{tab:feature_comparison} illustrates one personalized and one non-personalized case. Consistent with \cref{tab:H2_synth_realworld}, GPT4-ZeroShot and Semantic Similarity over-select, adding noise in non-personalized cases; GPT4-FewShot selects less in general but still selects items for non-personalized inputs. \ourmethodconserv{} balances personalization by matching human annotations in the personalized case and abstaining in the non-personalized case.

\begin{table*}[t]
\centering
\scriptsize
\caption{Examples of memory items selected by each method compared to human gold labels. Overlapping memory with items human annotation is \textbf{bolded}.}
\label{tab:feature_comparison}
\setlength{\tabcolsep}{4pt}
\renewcommand{\arraystretch}{1.2}
\resizebox{\textwidth}{!}{
\begin{tabular}{%
    p{0.20\linewidth}|
    >{\columncolor{gray!15}}p{0.15\linewidth}|
    p{0.18\linewidth}|
        p{0.17\linewidth}|
    p{0.12\linewidth}|
    p{0.15\linewidth}
}
\hline
\textbf{User Input} & \textbf{Humans*} & \textbf{\ourmethodconserv{}} & \textbf{GPT4-ZeroShot} &\textbf{GPT4-FewShot} & \textbf{Semantic Similarity} \\
\hline
\begin{tabular}[t]{@{} p{1\linewidth}@{}} \textbf{[PERSONALIZED] } \\I became a millionaire recently and want to give money to some local businesses\end{tabular}
& \begin{tabular}[t]{@{}l@{}} Location \\ Hobbies and Interests \\ Volunteer Activities\end{tabular}
& \begin{tabular}[t]{@{}l@{}} \textbf{Location} \\ \textbf{Hobbies and interests} \\ \textbf{Volunteer activities} \\ Personal values \\   Industry of employment\end{tabular}
& \begin{tabular}[t]{@{}l@{}} \textbf{Location} \\ Occupation \\ Industry of employment \\ Current projects \\ Long-term aspirations \end{tabular}

& \begin{tabular}[t]{@{}l@{}} \textbf{Location} \\ Recent life events \end{tabular}
& \begin{tabular}[t]{@{}l@{}} \textbf{Location}\\ \textbf{Volunteer activities} \\ Financial situation \\ Annual Income  \\ Pet ownership \end{tabular} \\
\hline
\textbf{[NON-PERSONALIZED]} Who was the Prime Minister of Gt. Britain at the outbreak of The First World War?
& \begin{tabular}[t]{@{}l@{}} No features selected \end{tabular}
& \begin{tabular}[t]{@{}l@{}} \textbf{No features selected} \end{tabular}
& \begin{tabular}[t]{@{}l@{}} Location \\ Occupation \\ Name \\ Political affiliation \end{tabular}
& \begin{tabular}[t]{@{}l@{}} Location \\ Occupation \end{tabular}
& \begin{tabular}[t]{@{}l@{}} Travel history \\ Occupation \\ Name \\ Political affiliation \\ Current projects \end{tabular} \\
\hline
\end{tabular}
}
\end{table*}

\subsection{Cost Analysis}
To better understand the practical implications of our approach, we analyze the tradeoff between computational cost and model performance across all methods. Since our methods and baselines differ in the number of memory items they select, and hence the number of additional tokens they introduce into the prompt, it is important to quantify how these differences affect overall efficiency at scale.

We measure computation overhead with two complementary metrics. First, we calculate the \textit{average number of memory items selected} per dataset and method. This captures how the amount of items each method chooses in order to fulfill personalization or how little for no personalization. Second, we compute the \textit{average number of additional tokens} that must be appended to the prompt during inference in order to include the selected memory items. This provides a direct measure of the inference overhead. \Cref{tab:H3_nfeat} and \cref{tab:H3_ntoken} summarize these results, showing clear differences across methods: baselines such as All or Random predictably add large numbers of items, while information-retrieval heuristics like BM25 are highly selective. Our methods lie in between, striking a balance between selecting informative items and avoiding unnecessary additions.

To translate these measures into practical cost, we estimate the monthly cost of running one million queries per day for each method using pricing aligned with GPT-4. Specifically, we convert the additional tokens from Table 2 into dollar cost under the assumption that each token incurs a fixed marginal cost of \$0.01 per 1K tokens and 30 days per month. This allows us to approximate the total monthly inference cost for each method, and thus assess the economic tradeoff of personalization strategies at scale.

The results are visualized in the \cref{fig:H3_cost}. As expected, the All baseline is by far the most expensive, exceeding \$160K per month due to its maximal inclusion of all memory items. On the other hand, our model \ourmethodconserv{} yields much lower additional cost of only \$6.9K which is about $30\%$ lower than just the inference cost of including the memory features selected by GPT4-FewShot, not including the cost to generate these memory items. Our other model, \ourmethodlessconserv{}, also has low cost of only \$8.4K, keeping costs well below stronger baselines like GPT-4, ReContriever, and Semantic Similarity. Overall, this analysis highlights that \ourmethodmodel{} can be deployed at scale in a cost-efficient manner, with our methods demonstrating favorable tradeoffs between informativeness and economic feasibility.

\begin{table*}[t]
\flushleft
\begin{minipage}{0.70\linewidth}
\centering
\caption{Average number of memory items selected by method per dataset.}
\label{tab:H3_nfeat}
\begin{tabular}{lcccc}
\toprule
Method & \fb{} & \pf{} & \wcp{} & \wcnp{} \\
\midrule
All & 50 & 50 & 50 & 50 \\
None & 0 & 0 & 0 & 0 \\
Random & 5 & 5 & 5 & 5 \\
Semantic Similarity & 4.71 & 4.54 & 4.84 & 4.74 \\
BM25 & 0.39 & 0.81 & 0.28 & 0.42 \\
ReContriever & 5 & 5 & 5 & 5 \\
GPT4-Few & 1.21 & 4.39 & 3.90 & 2.10 \\
\ourmethodconserv{} & 0.02 & 3.14 & 2.80 & 1.64 \\
\ourmethodlessconserv{} & 1.45 & 3.58 & 2.36 & 1.76 \\
\bottomrule
\end{tabular}
\end{minipage}
\hfill
\begin{minipage}{0.70\linewidth}
\centering
\caption{Average number of extra tokens required for response generation.}
\label{tab:H3_ntoken}
\begin{tabular}{lcccc}
\toprule
Method & \fb{} & \pf{} & \wcp{} & \wcnp{} \\
\midrule
All & 551.59 & 551.28 & 552.16 & 550.40 \\
None & 0 & 0 & 0 & 0 \\
Random & 54.87 & 54.98 & 54.34 & 54.20 \\
Semantic Similarity & 48.04 & 51.44 & 53.04 & 49.58 \\
BM25 & 1.48 & 8.95 & 3.20 & 4.86 \\
ReContriever & 60.87 & 62.93 & 61.22 & 61.72 \\
GPT4-Few & 12.96 & 49.31 & 46.02 & 23.58 \\
\ourmethodconserv{} & 0.23 & 37.39 & 34.88 & 19.10 \\
\ourmethodlessconserv{} & 13.88 & 45.34 & 30.90 & 21.08 \\
\bottomrule
\end{tabular}
\end{minipage}
\end{table*}

\begin{figure}
    \centering
    \includegraphics[width=0.7\linewidth]{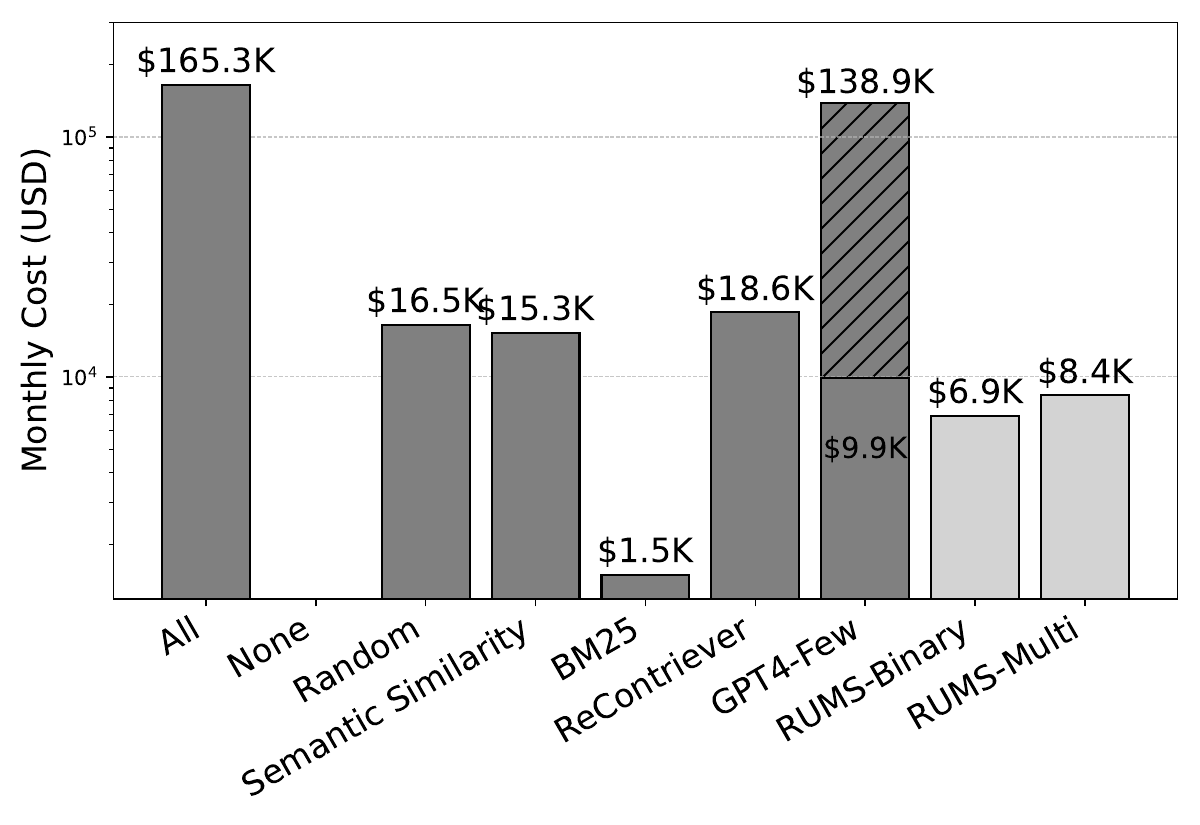}
    \caption{Estimated monthly cost of additional prompt tokens (1M queries per day, 30 days, GPT-4 input pricing at \$0.01 per 1K tokens).}
    \label{fig:H3_cost}
\end{figure}

\subsection{Time Analysis}
We ran a wall-clock analysis to showcase the time resources needed to run each method (\cref{tab:wall_clock}. We calculated the average time for doing inference on one sample using $n=3$ runs. We find that the \ourmethodlessconserv{} have runtimes comparable to the retrieval methods (Semantic Similarity, BM25, and ReContriever) and \ourmethodconserv{} has times comparable to GPT4-Fewshot. As mentioned, the RUMS-Utility has a very high runtime and therefore we do not recommend it to be used during inference. 

\begin{table}[t]
\centering
\scriptsize
\caption{Average time for inference of one sample, calculated over $n=3$ trials. All values are in units of seconds.}
\label{tab:wall_clock}
\resizebox{0.43\textwidth}{!}{%
\begin{tabular}{lc}
\toprule
Method & Average Time (s) \\
\midrule
All & 0.0 \\
None & 0.0 \\
Random & 0.0 \\
Semantic Similarity & 0.21 \\
BM25 & 0.01 \\
ReContriever & 0.08 \\
GPT4-Few & 1.22 \\
\rowcolor{gray!25}\ourmethodconserv{} & 0.14 \\
\rowcolor{gray!25}\ourmethodlessconserv{} & 0.04 \\
\rowcolor{gray!15}RUMS-Utility & 7.28 \\
\bottomrule
\end{tabular}
}

\end{table}
 
\subsection{Results of H3 using Human Judgment}
For evaluation in H3, LLM-as-Judge was chosen due to it's validation as similar to third-party human annotation for personalization when both in-context examples and model confidence scores \citep{dong-etal-2024-llm}. However, to validate that selection alignment translates to better responses, we conducted a human study comparing \ourmethodconserv{} with GPT-4 Few-Shot (strongest baseline). We had two human annotators (CSE PhD graduate students) annotate which of the two response they felt were better using the same rubric as the LLM-as-Judge in H3.

\begin{table}[!ht]
\centering
\caption{Winrate of responses generated using memory subset of \ourmethodutility{} vs. GPT4-Few, determined by human evaluation.}
\label{tab:H3_human_eval}
\begin{tabular}{lc}
\toprule
Dataset & \ourmethodutility{} Winrate\\
\midrule 
Synthetic & 0.33 \\
Real World & -0.03\\
\bottomrule
\end{tabular}
\end{table}
 \label{appx:H3_human_eval}

We found \ourmethodconserv{} outperforms GPT4-Few by $33\%$ in human preference, despite GPT-4 Few-Shot using a model $~400 \times$ larger for both selection and generation. This result (1) validates our LLM-as-Judge evaluation methodology ($33\%$ human win rate closely matches the $18\%$ win rate from Table 3 against GPT-4 Few-Shot on synthetic data), (2) Confirms that RUMS's human-aligned selection decisions (from H2) translate to better final outputs (H3).

\subsection{Analysis of Selected Entropy Calculation Model}
As shown in \cref{appx:model_family_ablation}, we found that Llama-7 was the best performing model on calculating entropy scores that correctly seperate personalized versus non-personlized inputs. However, given that Mistral-7B was also moderately good at this task, we test how RUMS would work using Mistral-7B on both binary decisions of personalization (H1), and choosing memory item subsets (H2). 

\begin{table}[ht]
\centering
\caption{H1 Results}
\label{tab:mistral_h1}
\begin{tabular}{llcc}
\toprule
\textbf{Dataset} & \textbf{Model} & \textbf{Recall} & \textbf{Specificity} \\
\midrule
Synthetic    & Llama-7B & \textbf{0.94} & \textbf{0.94} \\
Real World   & Llama-7B & 0.92 & \textbf{0.96} \\
\addlinespace
Synthetic    & Mistral-7B & 0.83 & 0.88 \\
Real World   & Mistral-7B & 0.92 & 0.77 \\
\bottomrule
\end{tabular}
\caption{Results of using Llama-7B or Mistral-7B to calculate entropy scores on H1 experiments. Bolded is the higher value between the two models.}
\end{table}

\begin{table}[ht]
\centering
\caption{H1 Results}
\label{tab:mistral_h2}
\begin{tabular}{llccc}
\toprule
\textbf{Dataset} & \textbf{Model} & \textbf{{Precision}} & \textbf{Recall} & \textbf{F1} \\
\midrule
Synthetic    & Llama-7B & 0.61 & \textbf{0.70} & \textbf{0.70} \\
Real World   & Llama-7B & 0.48 & \textbf{0.36} & 0.53 \\
\addlinespace
Synthetic    & Mistral-7B & 0.61 & 0.62 & 0.69 \\
Real World   & Mistral-7B & \textbf{0.61} & 0.30 & \textbf{0.69} \\
\bottomrule
\end{tabular}
\caption{Results of using Llama-7B or Mistral-7B to calculate entropy scores on H2 experiments. Bolded is the higher value between the two models.}
\end{table} 
\cref{tab:mistral_h1} show the results from H1 and \cref{tab:mistral_h2} show the results from H2 from using Mistral-7B compared to Llama-7B. We find that although in general Llama-7B beats Mistral-7B, they perform quite closely. This might indicate that even if slightly less separation is seen between personalized and non-personalized entropy scores they can still be useful signal. 

\subsection{Analysis of \textit{Optional} Filter Step}
In our methods section, we mention that if a large number of candidate memory items exist it can be useful to use a more coarse method (such as a prompted LLM) to reduce the number of memory features for RUMS calculation. In our experimentation, we use GPT-4 as this method of coarse deselection which chooses up to 5 memory items from the original 50 to measure. In this section we validate the use of GPT-4 for this purpose. 

To validate the use of GPT-4 and other possible other models, we use the randomly sampled personalized inputs which we have a consensus human annotated memory items selected (n=60) used in H2. Then, we analyzed how well the large general purpose models (GPT-4, Llama-70B, Qwen-72B, and Gemma-27B) filtered memory items (set to k=5) created by prompting the model, compared to the human selected memory items. We note that since these models are instructed to overselect (i.e. filter to a list of 5) we focus mostly on the recall metric, rather than precision of F1. Also, we only show the results for the personalized datasets, as this is when this is mostly used. 

\begin{table}[ht]
\centering
\caption{Model Performance on Personalized Datasets}
\label{tab:personalized_results}
\begin{tabular}{llccc}
\toprule
\textbf{Dataset} & \textbf{Model} & \textbf{Precision} & \textbf{Recall} & \textbf{F1} \\
\midrule
\multirow{4}{*}{\shortstack[l]{Synthetic\\Personalized}}
 & GPT-4       & 0.24 & \textbf{0.77} & 0.49 \\
 & Llama-70B   & 0.23 & 0.74          & 0.47 \\
 & Qwen-72B    & 0.21 & 0.74          & 0.45 \\
 & Gemma-27B   & 0.22 & 0.66          & 0.45 \\
\midrule
\multirow{4}{*}{\shortstack[l]{Real World\\Personalized}}
 & GPT-4       & 0.46 & \textbf{0.70} & 0.59 \\
 & Llama-70B   & 0.41 & 0.69          & 0.53 \\
 & Qwen-72B    & 0.44 & 0.67          & 0.55 \\
 & Gemma-27B   & 0.48 & 0.64          & 0.56 \\
\bottomrule
\end{tabular}
\caption{Results of using large general purpose models for filtering. We compare the memory subsets filtered by each model, through prompting, with a human annotated list. }
\label{tab:filter_results}
\end{table} 
\cref{tab:filter_results} shows the results from this analysis. We find that in general, all the models perform similarly with high recall to human annotations. However, we do see that GPT-4, the model we used, was the best at 0.77 and 0.70 recall for synthetic and real world data respectively. In general, this supports this method of coarse deselection. 

\begin{table}[ht]
\centering
\caption{Pairwise Agreement Between Models and Humans}
\label{tab:pairwise_agreement}
\begin{tabular}{lccccc}
\toprule
 & \textbf{Humans} & \textbf{GPT-4} & \textbf{Llama-70B} & \textbf{Qwen-72B} & \textbf{Gemma-27B} \\
\midrule
\textbf{Humans}    & --- & ---  & ---  & ---  & --- \\
\textbf{GPT-4}     & 0.4 & ---  & ---  & ---  & --- \\
\textbf{Llama-70B} & 0.5 & 0.7  & ---  & ---  & --- \\
\textbf{Qwen-72B}  & 0.5 & 0.8  & 0.6  & ---  & --- \\
\textbf{Gemma-27B} & 0.5 & 0.6  & 0.7  & 0.7  & --- \\
\bottomrule
\end{tabular}
\caption{Pairwise overlap (percentage) of the top-10 most frequently selected memory items across models and humans.}
\label{tab:filter_top_10}
\end{table} 
To further validate these results, we were interested in if there is any bias in the type of features selected by the different models. So, we compared the top-10 most selected memory items per model. We found that on average, the models shared 6.8 memory items with other models and 4.75 items with humans. \cref{tab:filter_top_10} shows the full results. This gave evidence that different models did not seem to have unique preferences or selections, thus adding further robustness to this method. 

\subsection{Full Results of H3}\label{appx:breakdown_results}
In this section we expand H3 by breaking down results for the Synthetic and Real World datasets by inputs intended to be personalized (left) and not personalized (right). This comparison highlights differences between \ourmethodconserv{} and \ourmethodlessconserv{}. With the less capable response model, LLaMA-70B, \ourmethodconserv{} (which favors fewer memory items) outperforms most baselines, showing especially strong gains on non-personalized inputs, with a $38\%$ average win rate over baselines on the Synthetic dataset. Similar patterns hold on the Real World dataset, except for the None and BM25 baselines, which never select items.  

By contrast, \ourmethodlessconserv{} (which selects more memory items) performs better with the stronger GPT-4 model, particularly on personalized inputs, achieving $25\%$ and $26\%$ average win rates on the personalized Synthetic and Real World datasets respectively. These results demonstrate how \ourmethod{} can be tuned to different settings, adapting to the response model and personalization needs.

\begin{table}[ht]
\caption{Average winrate comparison on Synthetic data under Not Personal (left) and Personal (right) conditions. Positive values (green) indicate better performance, negative values (red) indicate worse performance. Values greater than 0.1 in magnitude are bolded.}
\label{tab:synth_notpersonal_personal}
\centering
\resizebox{\linewidth}{!}{%
\begin{tabular}{cc}
\begin{subtable}[t]{0.9\linewidth}
\centering
\caption{Non-personalized inputs}
\begin{tabular}{lcccc}
\toprule
 & \multicolumn{2}{c}{LLaMA 70B} & \multicolumn{2}{c}{GPT-4} \\
 & \ourmethodconserv{}. & \ourmethodlessconserv{} & \ourmethodconserv{}. & \ourmethodlessconserv{} \\
\midrule
None              & \cellval{-0.03}  & \cellval{-0.79} & \cellval{0.02}  & \cellval{0.00} \\
All               & \cellval{0.45}   & \cellval{-0.24} & \cellval{0.00}  & \cellval{0.02}  \\
Random            & \cellval{0.50}   & \cellval{-0.12} & \cellval{-0.03} & \cellval{-0.01} \\
Semantic Similarity          & \cellval{0.60}   & \cellval{0.01} & \cellval{0.02}  & \cellval{0.00}  \\
BM25              & \cellval{0.12}   & \cellval{-0.66} & \cellval{-0.02} & \cellval{0.01} \\
ReContriever      & \cellval{0.62}   & \cellval{-0.14} & \cellval{0.01}  & \cellval{0.03} \\
GPT4-Few        & \cellval{0.42}   & \cellval{-0.34} & \cellval{0.00}  & \cellval{0.01} \\
\bottomrule
\end{tabular}
\end{subtable}
&
\begin{subtable}[t]{0.9\linewidth}
\centering
\caption{Personalized inputs}
\begin{tabular}{lcccc}
\toprule
 & \multicolumn{2}{c}{LLaMA 70B} & \multicolumn{2}{c}{GPT-4} \\
 & \ourmethodconserv{}. & \ourmethodlessconserv{} & \ourmethodconserv{}. & \ourmethodlessconserv{} \\
\midrule
none              & \cellval{0.52}  & \cellval{0.53}  & \cellval{0.50}  & \cellval{0.72} \\
All               & \cellval{-0.14} & \cellval{-0.14} & \cellval{-0.31} & \cellval{0.01} \\
Random            & \cellval{0.18}  & \cellval{0.37}  & \cellval{0.15}  & \cellval{0.39} \\
Semantic Similarity          & \cellval{-0.06} & \cellval{0.17}  & \cellval{-0.07} & \cellval{0.09} \\
BM25              & \cellval{0.25}  & \cellval{0.40}  & \cellval{0.42}  & \cellval{0.40} \\
ReContriever      & \cellval{-0.09} & \cellval{0.09}  & \cellval{-0.09} & \cellval{0.09} \\
GPT4-Few        & \cellval{-0.07} & \cellval{0.03}  & \cellval{-0.17} & \cellval{0.03} \\
\bottomrule
\end{tabular}
\end{subtable}
\end{tabular}
} %

\end{table} %
\begin{table*}[ht]
\centering
\caption{Average winrate comparison on Real World data under Not Personal (left) and Personal (right) conditions. Positive values (green) indicate better performance, negative values (red) indicate worse performance. Values greater than 0.1 in magnitude are bolded.}
\label{tab:realworld_notpersonal_personal}
\resizebox{\linewidth}{!}{%
\begin{tabular}{cc}
\begin{subtable}[t]{.9\linewidth}
\centering
\caption{Non-personalized inputs}
\begin{tabular}{lcccc}
\toprule
 & \multicolumn{2}{c}{LLaMA 70B} & \multicolumn{2}{c}{GPT-4} \\
 & \ourmethodconserv{} & \ourmethodlessconserv{} & \ourmethodconserv{} & \ourmethodlessconserv{}\\
\midrule
None              & \cellval{-0.36} & \cellval{-0.64} & \cellval{-0.12} & \cellval{-0.26} \\
All               & \cellval{0.08}  & \cellval{-0.14}  & \cellval{-0.14} & \cellval{-0.16} \\
Random            & \cellval{-0.02} & \cellval{-0.06} & \cellval{-0.06} & \cellval{-0.26} \\
Semantic Similarity          & \cellval{0.04}  & \cellval{-0.02} & \cellval{-0.02} & \cellval{-0.12}  \\
BM25              & \cellval{-0.40} & \cellval{-0.16} & \cellval{-0.16} & \cellval{-0.26} \\
ReContriever      & \cellval{0.12}  & \cellval{0.04}  & \cellval{0.04}  & \cellval{0.02}  \\
GPT4-Few        & \cellval{0.08}  & \cellval{0.12} & \cellval{0.12}  & \cellval{-0.02}  \\
\bottomrule
\end{tabular}
\end{subtable}
&
\begin{subtable}[t]{.9\linewidth}
\centering
\caption{Personalized inputs}
\begin{tabular}{lcccc}
\toprule
 & \multicolumn{2}{c}{LLaMA 70B} & \multicolumn{2}{c}{GPT-4} \\
 & \ourmethodconserv{} & \ourmethodlessconserv{} & \ourmethodconserv{} & \ourmethodlessconserv{} \\
\midrule
None              & \cellval{0.46} & \cellval{0.70} & \cellval{0.52} & \cellval{0.60} \\
All               & \cellval{-0.16} & \cellval{-0.04} & \cellval{-0.14} & \cellval{0.04} \\
Random            & \cellval{0.24} & \cellval{0.42}  & \cellval{0.26}  & \cellval{0.34}  \\
Semantic Similarity          & \cellval{-0.10} & \cellval{0.00} & \cellval{-0.02} & \cellval{0.12}  \\
BM25              & \cellval{0.42} & \cellval{0.58} & \cellval{0.52}  & \cellval{0.52}  \\
ReContriever      & \cellval{0.00} & \cellval{0.24} & \cellval{-0.14} & \cellval{0.24}  \\
GPT4-Few        & \cellval{-0.08} & \cellval{-0.04} & \cellval{-0.04} & \cellval{-0.04}  \\
\bottomrule
\end{tabular}
\end{subtable}
\end{tabular}
} %
\end{table*} 
\subsection{Qualitative Analysis and Results} To better explore the cases when memory items selected by RUMS-Utility does not match the human selected, we further analyze the n=60 samples from H2. First, we define the different categories of cases we found: 
\begin{itemize}
    \item Correct: RUMS-Utility matches Human chosen list exactly
    \item Underselect: Human selected items and RUMS-Utility did not
    \item Some: RUMS-Utility selected a mixture of items selected at least one human, and not selected by humans
    \item All: RUMS-Utility selected all items which at least one human selected (but possibly not all humans)
    \item None: RUMS-Utility only selected items which no human selected
    \item Underselect - All: RUMS-Utility did select some items, which were chosen by all humans, but humans selected more items as well which RUMS-Utility did not. 
\end{itemize}

Using these categories, we count the number of samples for each case.

\begin{table}[t]
\centering
\small
\caption{Breakdown of selection behavior categories with representative user inputs and selected memory items by humans (all agree), humans (at least one agree), and RUMS Utility.}
\label{tab:qualitative_analysis}
\setlength{\tabcolsep}{4pt}
\renewcommand{\arraystretch}{1.2}
\begin{tabular}{p{1.4cm} c p{3.5cm} p{3cm} p{3.5cm} p{3cm}}
\toprule
\textbf{Category} & \textbf{\%} & \textbf{Example}& &\textbf{Chosen Items}&\\
& & User Input & Humans (All) & Humans ($>1$) & RUMS-Utility \\
\midrule
Correct & 38 &
``In May 1994 the Channel Tunnel was opened by Queen Elizabeth II and which French President?'' &
[] &
[] &
[] \\

Underselect & 15 &
``How to use a vpn while vacationing in china'' &
[Tech savvy] &
[Tech savvy, Feedback style, Comm mode, Tech use] &
[none] \\

All & 13 &
``Choose between new zara sneaker or used new balance'' &
[] &
[Hobbies, Income, Aspirations, Finances, Exercise, Personality, Job, Industry, Pastimes, Events, Gender, Location] &
[Hobbies, Income] \\

Some & 22 &
``Buying a new notebook to write a script'' &
[] &
[Location, Income, Tech savvy, Finances, Personality] &
[Hobbies, Tech savvy, Fav books] \\

None & 5 &
``My PC is r6 3600 and 1660s need new GPU'' &
[Projects, Finances] &
[Income, Job, Hobbies, Industry, Projects, Aspirations, Finances, Pastimes] &
[Tech savvy] \\

Underselect - All & 7 &
``Im hungry'' &
[Foods, Diet needs, Fitness, Location] &
[Foods, Diet needs, Fitness, Location, Income, Finances, Travel, Gender] &
[Foods, Diet needs, Fitness] \\
\bottomrule
\end{tabular}
\end{table}
 
Table \cref{tab:qualitative_analysis} shows the results from this analysis, including the type of category, percentage of samples, and an example of each. We find majority of non-correct samples are from RUMS-Utility selecting values which at least one human selected but not all humans. This makes sense, as memory selection for personalization can be highly subjective and so this just shows that RUMS-Utility may select some of the more on the edge items. There is also a moderate level of underselect. This lends more weight to the idea that RUMS-Utility is a more conservative method, however, this could be changed based on the threshold selected by the user. 

\section{Formulas}
\subsection{Shannon Entropy} \label{appx:shannon_entropy_defn} We define the Shannon Entropy as a measure of uncertainty within a system. For a random variable $x \in X$ and a distribution according to $p: X \rightarrow [0,1]$ such that $p(x): = \mathbb{P}[X =x]$. Then we formalize the Shannon entropy $H$ as, 
\begin{align*}
    H(x) = - \sum_{x \in X} p(x) \log p(x)
\end{align*}

\section{Optimality of \ourmethod{} Proof}\label{appx:optimality_proof}

\begin{proof}
We begin with the joint conditional entropy of $(y, z_u)$ given $(x, m_{u,S})$. By the chain rule,
\begin{align*}
H(y, z_u \mid x, m_{u,S})
&= H(y \mid x, m_{u,S}) + H(z_u \mid x, y, m_{u,S}) \\
&= H(z_u \mid x, m_{u,S}) + H(y \mid x, z_u, m_{u,S}).
\end{align*}

Equating the two expressions and rearranging yields
\[
H(y \mid x, m_{u,S})
=
H(y \mid x, z_u, m_{u,S})
+
H(z_u \mid x, m_{u,S})
-
H(z_u \mid x, y, m_{u,S}).
\]

By Assumption~2, memory $m_{u,S}$ provides little additional information about $y$ beyond $(x, z_u)$. We therefore approximate
\[
H(y \mid x, z_u, m_{u,S}) \approx H(y \mid x, z_u).
\]
Substituting this approximation gives
\[
H(y \mid x, m_{u,S})
\approx
H(y \mid x, z_u)
+
H(z_u \mid x, m_{u,S})
-
H(z_u \mid x, y, m_{u,S}).
\]

The first term, $H(y \mid x, z_u)$, does not depend on the choice of subset $S$. In addition, Assumption~2 implies that conditioning on $m_{u,S}$ does not substantially refine the posterior over $z_u$ once $(x, y)$ are observed, so
\[
H(z_u \mid x, y, m_{u,S}) \approx H(z_u \mid x, y),
\]
which is also independent of $S$.

Combining these observations, we obtain the approximation
\[
H(y \mid x, m_{u,S}) \approx C + H(z_u \mid x, m_{u,S}),
\]
where $C$ is a constant that does not depend on $S$.

Therefore, minimizing predictive entropy $H(y \mid x, m_{u,S})$ is approximately equivalent to minimizing uncertainty about the latent variable $z_u$ given the selected memory.

Now consider expected utility as defined in Assumption~1,
\[
\mathbb{E}[U_u(y)]
=
-
\mathbb{E}_{z_u}
\mathbb{E}_{y \sim p_\theta(\cdot \mid x, m_{u,S})}
[\ell(y, z_u)].
\]
Because $\ell$ is a strictly proper scoring rule, the expected loss is minimized when the predictive distribution $p_\theta(y \mid x, m_{u,S})$ matches the conditional distribution induced by $z_u$. Reducing uncertainty about $z_u$ therefore strictly improves expected utility.

Consequently, any subset $S^*$ that approximately minimizes $H(y \mid x, m_{u,S})$ also approximately maximizes expected utility.
\end{proof}

\section{Sequence Entropy Equivalence Proof}\label{appx:seq_entropy_proof}
\begin{proof}
Let $p_\theta(y \mid x)$ denote the model distribution over token sequences $y=(y_1, \dots, y_T)$ given input $x$.  
At each step $t$, the conditional next-token distribution is
\begin{align*}
    p_\theta (\cdot \mid y_{<t},x),
\end{align*}
where $y_{<t}$ is the prefix of tokens before step $t$.

By the chain rule of probability,
\begin{align*}
    -\log p_\theta (y \mid x) =   \sum_{t=1}^T - \log p_\theta (y_t \mid y_{<t}, x).
\end{align*}
Taking the expectation with respect to $y\sim p_\theta(\cdot \mid x)$ on both sides gives
\begin{align*}
    \mathbb{E}_{y \sim p_\theta} \left[-\log p_\theta (y \mid x)\right] 
    &=  \mathbb{E}_{y \sim p_\theta} \left[ \sum_{t=1}^T -\log p_\theta (y_t \mid y_{<t}, x) \right ] \\
    &= \sum_{t=1}^T \mathbb{E}_{y \sim p_\theta} \left[ -\log p_\theta (y_t \mid y_{<t}, x) \right ] && \text{linearity of expectation}\\
    &= \sum_{t=1}^T \mathbb{E}_{y_{<t} \sim p_\theta} \left[  \sum_{y_t} p_\theta(y_t \mid y_{<t},x) -\log p_\theta(y_t \mid y_{<t}, x) \right ] && \text{law of iterated expectation}\\
    &= \mathbb{E}_{y_{<t} \sim p_\theta} \left [ \sum_{t=1}^T  H\!\left(p_\theta(\cdot \mid y_{<t}, x) \right)\right].
\end{align*}
\end{proof}

\section{\ourmethod-Utility Algorithm}\label{appx:algorithm}
In \cref{alg:boed} we provide the formalized algorithm for \ourmethod{}-Utility.

\begin{algorithm}
\caption{Memory Selection via \ourmethod}
\label{alg:boed}
\begin{algorithmic}[1]
\Require User input $x$, user memory $u \in \mathbb{R}^d$
\Ensure Selected memory subset $m_{u,S}^*$ that maximizes utility

\State \textbf{Candidate Memory Item Identification:}
\State Use LLM to generate candidate subsets $S$ for $x$

\State \textbf{Subset Evaluation:}
\For{each $s \in S$}
    \State Map candidate attritubes to user values $m_{u,S}$
    \State Prompt Reference LM with:
    \State \hspace{1em} (i) $x$ only
    \State \hspace{1em} (ii) $x$ and $m_{u,S}$
    \State Compute estimated utility $\hat{\text{EIG}}(S)$ using \cref{eqn:method_utility}
    \State Store utility score for subset $s$
\EndFor

\State \textbf{Subset Selection:}
\State $s^* \gets \arg\max_{s \in S} \hat{\text{EIG}}(S)$
\If{all MI scores $< \text{threshold} \tau$}
    \State $s^* \gets \emptyset$ \Comment{Select no features}
\EndIf

\end{algorithmic}
\end{algorithm} 

\section{Experimental Details}\label{appx:exp_details}
This section provides the full details of the experiments described in \cref{sec:training_method}.

\subsection{Experimental Setup}
Across all experiments (H1–H3), we consider a fixed set of $n=50$ memory items (e.g., \textit{location, age, dietary restrictions}). The task is to select a subset of up to $n=5$ memory items that can improve personalization.

\subsection{Datasets}\label{appx:dataset}
To evaluate scenarios where personalization is beneficial as well as those where it is not, we use a diverse mixture of user inputs drawn from prior work, including both synthetic and real-world sources. The three datasets we used are described below:

\paragraph{\pf{} \citep{tao2025personafeedbacklargescalehumanannotatedbenchmark}}
This is a synthetically created dataset consisting of user-specific queries designed to reflect inputs from realistic and diverse user personas. First, they constructed personas by representing them as a structured dictionary of memory items, including demographics, personality traits, and preferences. Then, open-source content from social media, reviews, and forums were used to populate these items with specific values. Lastly, an LLM generator was used to produce queries based on a sampled persona and specific sub-scenarios. Additional steps such as similarity-based filtering, rephrasing to avoid directly revealing persona information, and manual annotations ensured diversity and realism. The final dataset contains more than 4,000 persona-specific questions. The original data is in Chinese and therefore, we translated to English using a machine translation model \citep{TiedemannThottingal:EAMT2020}.  
Example: \textit{``Do you have any music for the trip?''}

\paragraph{\fb{} \citep{jiang-etal-2019-freebaseqa}}
A dataset of trivia-style questions and answers collected from human trivia competitions. The original dataset contains approximately 28,000 question–answer pairs.  
Example: \textit{``What type of snake is a sidewinder?''}

\paragraph{WildChat \citep{zhao2024wildchat}}
A large-scale collection of more than 1 million real-world user–ChatGPT interactions, encompassing over 2.5 million dialogue turns. The dataset was compiled by offering free access to ChatGPT and GPT-4 in exchange for user consent to share chat histories. It covers diverse conversational scenarios such as ambiguous requests, code-switching, topic shifts, and political discussions. We filter the data to retain only English conversations labeled as non-toxic, and we use only the first human-written input per conversation to ensure that queries contain sufficient context.  

To determine whether an input is personalized, two human annotators labeled a random selection of WildChat queries as either beneficial for personalization (\wcp{}) or not (\wcnp{}). Only queries with full annotator agreement were retained.  
Examples: \textit{``I'm hungry''; ``Who is Mohamed Gamal Mohamed Maniieh?''}

\paragraph{Synthetic vs. Real-World Datasets}
By design, \pf{} is treated as a personalized input set, and \fb{} as non-personalized. We collectively refer to these as the ``Synthetic'' dataset, since they were not originally created for LLM input. WildChat is considered a ``Real-World'' dataset, as it was collected directly from real user interactions with an LLM. We use both synthetic and real-world datasets to evaluate the robustness of \ourmethod{} across many diverse inputs.

\paragraph{Sample Size by Hypothesis}
The number of samples varies across hypotheses. For H1, we use $n=500$ samples from \pf{} and \fb{} to analyze \ourmethod{}-Utility for binary personalization decisions. For H1', we use $n=25$ inputs each from \pf{}, \fb{}, \wcp{}, and \wcnp{} to extend the H1 analysis to correlation with human judgment. For H2, we use $n=15$ inputs each from \pf{}, \fb{}, \wcp{}, and \wcnp{} to examine the correlation between \ourmethod{} memory selection and human preference. Finally, for H3, we use $n=100$ samples from \pf{} and \fb{} along with $n=50$ samples each from \wcp{} and \wcnp{} to analyze the downstream impact of subsets chosen by \ourmethod{}-Model.

\subsection{Profile Generation}\label{appx:profile_gen}
User profiles are incorporated during the training data curation phase, where specific user values are used to compute the entropy difference of the model required for the utility function in \cref{eqn:method_utility}. Also, we use profiles in our test sets for evaluation.

\paragraph{Memory Item Creation}\label{appx:memory_items}
To construct user profiles, we prompted GPT-4 \citep{openai2024gpt4technicalreport} to generate 50 commonly occurring user memory items. These items span demographic, behavioral, and preference-based attributes, providing a diverse set of values for personalization. The complete list is as follows:
\begin{itemize}
\begin{multicols}{3}
    \item Name
    \item Gender
    \item Location (City/State/Country)
    \item Ethnicity
    \item Race
    \item Annual Household Income
    \item Language spoken at home
    \item Language preference
    \item Sexual Preference
    \item Disabilities
    \item Occupation
    \item Political affiliation
    \item Education level
    \item Hobbies and interests
    \item Favorite sports
    \item Preferred music genre
    \item Favorite books
    \item Preferred movie genre
    \item Favorite TV shows
    \item Favorite foods
    \item Dietary restrictions (e.g., vegetarian, vegan)
    \item Relationship status
    \item Number of children
    \item Pet ownership (types of pets)
    \item Travel history (countries visited)
    \item Travel preferences (adventurous, relaxing)
    \item Social media platforms used
    \item Level of tech-savviness
    \item Preferred mode of communication (text, voice, video)
    \item Work schedule (full-time, part-time)
    \item Industry of employment
    \item Current projects or goals
    \item Long-term aspirations
    \item Financial situation (broad categories)
    \item Health and fitness goals
    \item Preferred exercise activities
    \item Level of environmental consciousness
    \item Volunteer activities or interests
    \item Current challenges or pain points
    \item Personality traits (e.g., introverted, extroverted)
    \item Favorite pastimes
    \item Personal values or beliefs
    \item Cultural background
    \item Religious beliefs
    \item Recent life events (e.g., change in job, moved)
    \item Technology usage (devices, apps)
    \item Frequency of using the service (daily, weekly)
    \item Reasons for using the service
    \item Feedback preferences (detailed, brief)
    \item Preferred tone of communication (formal, casual)
\end{multicols}
\end{itemize}
\paragraph{Memory Item Values}
Next, we used GPT-4 \citep{openai2024gpt4technicalreport} to generate consistent user profiles without contradictions (e.g., ``vegan'' with ``favorite food = steak''), following the prompt in \cref{prompt:profile_generate}. To further ensure logical consistency, we employed GPT-4 to filter out any profiles that appeared illogical, as described in \cref{prompt:profile_filter}. This process was repeated until we obtained $n=100$ unique, logically consistent profiles for both training and testing.

This method yields a diverse set of user profiles, as illustrated in the bar charts of memory item distributions shown in \cref{fig:profile_distr_start} through \cref{fig:profile_distr_end}.

\paragraph{Prompts used for Profile Generation}
\begin{prompt}\label{prompt:profile_generate}
For these 50 features, I want you to create 10 very diverse persona profiles. So for each of the 50 features, I want you to select a value (e.g. "age" = 50) that would make up the persona profile. Make sure the values make sense for each other, meaning don't have a person who is 12 years old have a job as a accountant. 

Output a json file with a dictionary with 10 keys, "Persona\_1", "Persona\_2",...,"Persona\_10", whose corresponding values is a list of the selected values in the same order as the features (e.g. [John, Male, Seattle/WA/USA,...]). 

Do not include any other markdown or text. Here are the Features:
\end{prompt}
\begin{prompt}\label{prompt:profile_filter}
    You are an AI assistant that needs to decide if a user profile is logically sound. This means that the information in the user profile should be consistent and non-conflicting.
    
    Only output a single word, either 'yes' or 'no', to indicate if the user profile is logically sound.
    
    Example 1: If a user profile states that the user is allergic to peanuts, it should not also state their favorite food is peanuts.
    
    Example 2: If a user profiles states highest education is a GED, it should not also state their profession is a doctor.
    
    Example 3: If a user profile states that the user is 25 years old, it should not also state that they have been working for 30 years.
    
    Example 4: If a user profile states that they live in Columbia, it should not state that their main language is French.
    
    User Profile:
\end{prompt}

\subsection{Baselines}\label{appx:baselines}
In this section we describe each baseline that is used in our experimentation.

\paragraph{GPT-4 Prompting \citep{openai2024gpt4technicalreport}:}  
We use GPT-4o as a baseline across H1, H2, and H3. In H1, we apply \cref{prompt:gpt4baseline_H1_zero} for the zero-shot setting and \cref{prompt:gpt4baseline_H1_few} for the few-shot setting to decide whether personalization is needed for a given query. For the few-shot case, we provide four in-context examples: two requiring personalization and two not requiring personalization.  

In H2 and H3, we adapt the prompts to instead elicit the subset of specific user memory items that would be beneficial for personalization. We use \cref{prompt:gpt4baseline_H2_zero} for zero-shot and \cref{prompt:gpt4baseline_H2_few} for few-shot. In the few-shot setting, we provide two in-context examples: one requiring personalization and one not.  

\begin{prompt}\label{prompt:gpt4baseline_H1_zero}
You are a classifier that determines whether a given user input would benefit from user-specific information for the model to respond.

            Your response must be a JSON object with a two key "Decision" with a value of either "True" or "False" and "Features" with either a list of user-specific features (at most 10 features) if "Decision" is "True" or "None" if "Decision" is "False" :
            
            Output "True" if the input would benefit from user-specific information to better respond to the query.
            
            Output "False" if the input would not benefit from user-specific information or would be the same answer regardless of the user.
            
            Respond with only the JSON object. Do not include any explanation or additional text.
            
            Input: {user\_input}

\end{prompt}

\begin{prompt}\label{prompt:gpt4baseline_H1_few}
    'You are a classifier that determines whether a given user input would benefit from user-specific information for the model to respond.
    
            Your response must be a JSON object with a two key "Decision" with a value of either "True" or "False" and "Features" with either a list of user-specific features (at most 10 features) if "Decision" is "True" or "None" if "Decision" is "False" :
            
            Output "True" if the input would benefit from user-specific information to better respond to the query.
            
            Output "False" if the input would not benefit from user-specific information or would be the same answer regardless of the user.
            
            Respond with only the JSON object. Do not include any explanation or additional text.
            
            Example 1:
            
            Input: "Where should I eat tonight?"
            
            Output: {"Decision": "True", "Features": ["location", "dietary restrictions", "income"]}
            
            Example 2:
            
            Input: "Can you help me write an essay on feminism?"
            
            Output: {"Decision": "True", "Features": ["education", "political ideology", "gender"]}
            
            Example 3:
            
            Input: "What is the capital of Austin, Texas?"
            
            Output: {"Decision": "False", "Features": "None"}
            
            Example 4:
            
            Input: "Can you explain what AGI is?"
            
            Output: {"Decision": "False", "Features": "None"}
            
            Now classify the following input: 
            
            Input:  {user\_input}'

\end{prompt}
\begin{prompt}\label{prompt:gpt4baseline_H2_zero}
    Given a user-input and list of features, select which features might be important for answering the user-input. 

    Choose up to 5 features.
    Respond in a JSON with one key "Chosen Features", which has a list of chosen features. 
    Do not include any other text or formatting.

    User Input: {user\_input}
    
    Possible Features: {candidate\_features\_ls}
\end{prompt}

\begin{prompt}\label{prompt:gpt4baseline_H2_few}
Given a user-input and list of features, select which features might be important for answering the user-input. 
        Choose up to 5 features.
        
        Respond in a JSON with one key "Chosen Features", which has a list of chosen features. 
        Do not include any other text or formatting.
        
        Example Input 1: "whats the weather today" 
        
        Example Output 1: {{"Chosen Features": ["Location (City/State/Country)"]}} 
        
        Example Input 2: "What is the capital of Texas?" 
        
        Example Output 2: {{"Chosen Features": []}} 
        
        User Input: {user\_input}
        
        Possible Features: {candidate\_features\_ls}'

\end{prompt} 
\paragraph{Semantic Similarity:}  
In H1–H3, we also include a simple baseline based on semantic similarity. Here, we compute the cosine similarity between user query embeddings and each memory item embeddings using a lightweight sentence-transformer model \citep{all-MiniLM-L6-v2}.  

In H1, we use the validation set, defined as a 20\% hold-out of the data, to tune a threshold for deciding whether personalization is required. Thresholds are determined by grid search over $[-1, 1]$ in increments of 0.5, selecting the value that maximizes agreement with human annotations for each dataset. A query is labeled as requiring personalization if at least one memory item has a similarity score above the selected threshold. The exact thresholds used for each dataset are reported in \cref{tab:semanticsimilarity_H1_thresholds}.

\begin{table}[h!]
    \centering
    \begin{tabular}{c|c}
        Dataset & Threshold \\
        \hline
        All & -1.0 \\
        Synthetic & -1.0 \\
        Real World & 0.5 \\
    \end{tabular}
    \caption{Threshold for H1 analysis of semantic similarity baseline. If at least one memory item had a semantic similarity higher than this threshold, then the query was labeled as needing personalization.}
    \label{tab:semanticsimilarity_H1_thresholds}
\end{table}

Since H2 and H3 require selecting specific memory items for personalization rather than making a simple binary decision, we set an intuitive threshold of $t=0$, corresponding to no similarity between the user query and a memory item. If more than five memory items exceed this threshold, we select the top five with the strongest correlations.  

\paragraph{Simple Baselines:}  
We also compare against several naive baselines. The ``None'' baseline provides no features ($n=0$), serving as a lower bound for personalized inputs and an upper bound for non-personalized ones. The ``All'' baseline includes all available features ($n=50$), testing the effect of providing maximal but noisy context. Finally, the ``Random'' baseline samples $n=5$ items at random, establishing a chance-level reference.  

\paragraph{SOTA Retrieval Methods:}  
We additionally evaluate two state-of-the-art retrieval methods.  

First, BM25 \citep{bm25s}, a widely used ranking function in information retrieval, can be adapted for memory selection. It scores memory items (treated as terms) based on their frequency in each instance relative to the entire dataset, balancing how common a memory item is with how discriminative it is. Items with higher BM25 scores are considered more informative and are selected as the most relevant, making BM25 well-suited for sparse or high-dimensional text and categorical data.  

Second, ReContriever \citep{lei-etal-2023-unsupervised} is a retrieval-based method that scores memory items by their ability to help a model retrieve relevant information. It leverages dense embeddings to measure semantic similarity between queries and candidate items, effectively identifying the most informative items for downstream tasks. Memory items that yield higher retrieval relevance are prioritized, making ReContriever particularly effective for selecting contextually important or predictive inputs in text or structured datasets. Similarly to semantic similarity a threshold of $t=0$ is used as threshold for selecting memory items, with a maximum of up to $n=5$ selected.

\subsection{\ourmethod{}-Utility and \ourmethod{}-Model Details}\label{appx:ourmethod_details}
In this paper, we describe two main methods: \ourmethod{}-Utility, which selects a subset of user memory items with the highest utility for personalization directly, and \ourmethod{}-Model, a lightweight trained model built on data generated by \ourmethod{}-Utility.  

\paragraph{\ourmethod{}-Utility:}  
We use \ourmethod{}-Utility not only for the analyses in H1 and H2 but also to generate training data for fine-tuning \ourmethod{}-Models. In general, we compute entropy in \cref{eqn:method_obj_fn} using Llama-3.1-8B \citep{meta_llama_3_1_8b_instruct} with a temperature of $0.7$ and sampling-based decoding to obtain the response distribution. We restrict the maximum generation length to $T=20$ tokens, chosen to mitigate bias from longer generations having naturally lower entropy.  

Following the algorithm in \cref{sec:method:model}, we first prompt GPT-4 to reduce the candidate memory items from 50 to at least 1 and at most 10 before computing \ourmethod{}-Utility (see \cref{prompt:gpt4_reducenum}). Utility scores are then estimated using $n=5$ Monte Carlo samples. For each input, we generate $n=5$ distinct user profiles as separate training points, which can yield entirely unique best memory subsets of data depending on the values in the user profile. 

The chosen hyperparameters (e.g., model size, number of Monte Carlo samples, etc.) reflect a tradeoff: obtaining reliable signals of memory item importance while keeping computations lightweight and feasible for large-scale inference. Further discussion of these choices is provided in \cref{appx:robust_utility_scores}.  
\begin{prompt}\label{prompt:gpt4_reducenum}
Given a user-input and list of features, select which features might be important for answering the user-input.

Choose at least 1 and up to 10 features.

Respond in a JSON with one key "Chosen Features", which has a list of chosen features. '
Do not include any other text or formatting.
\end{prompt}

\paragraph{Choosing the Personalization Threshold:}  
To determine the threshold for deciding whether to apply personalization or select no personalization, we used a held-out validation set of size $n=100$ from each of \pf{} and \fb{}. For each query, we computed the average maximum utility score over the $n=5$ generated user profiles. We labeled inputs from \pf{} as personalized (1) and inputs from \fb{} as non-personalized (0). These scores were combined and labeled by dataset to construct an ROC curve.  

We then selected the threshold that maximized Youden’s J statistic ($J = \text{TPR} - \text{FPR}$), which identifies the point that best balances sensitivity and specificity. Using this procedure, we determined an optimal threshold of $t=0.29$, which was applied in all downstream analyses.  
 
\paragraph{\ourmethod{}-Models}\label{appx:model_training}  
For training the downstream models, we sampled approximately 2.5K unique \pf{} queries and 2K unique \fb{} queries. Each query was paired with five user profiles, and \ourmethod{}-Utility was applied to select the best memory subsets, yielding 13K and 9K training examples, respectively. We excluded all \fb{} queries that exceeded the personalization threshold, thereby removing any noisy non-personalized samples. This design choice emphasized that non-personalized inputs should not include any memory items.  

We then used these (query, memory subset) pairs to train two models. The first model, used in H1–H3, was a binary logistic regression implemented with DeBERTa-v3-large \citep{he2021deberta} and a final linear classification layer. To train this model, we converted each (query, memory subset) pair into multiple (query, memory item, label) triplets. Training was performed with a learning rate of $5 \times 10^{-6}$ for 12 epochs, and the best model was selected based on the highest F1 score on a held-out validation set ($n=1K$).  

The second model, used for H3, was a multiclass, multilabel logistic regression implemented again with DeBERTa-v2-large \citep{he2021deberta} and a final multiclass classification layer. Here, labels were constructed by merging the selected subsets from the $n=5$ diverse user profiles per query. This produced a single label vector in which all memory items chosen in any of the five samples were marked as significant. This resulted in more dense labels, training the model to select more memory items. Training was performed with a learning rate of $2 \times 10^{-5}$ for 12 epochs, and the best model was again chosen based on the highest F1 score on a held-out validation set ($n=1K$).  

\paragraph{Selecting Regression Threshold} Since \ourmethodmodel{} is trained as a logistic regression model (either binary or multiclass), its output represents the probability of selecting a particular memory item. A common approach is to use a cutoff of $0.5$ to make a binary decision, however, since our task is more specialized, instead we tuned this cutoff using a held-out evaluation set ($n=1K$). Specifically, we perform a grid search over boundaries in $[0,1]$ and select the cutoff that achieved the highest F1 score on the evaluation set. This yielded $b=0.81$ for \ourmethodconserv{} and $b=0.93$ for \ourmethodlessconserv{}.

\subsection{Comparing Size of Methods}
In \cref{tab:model_sizes}, we summarize the size of each model used in both our method and the baselines. Our approach uses a model comparable in scale to SOTA baselines, yet it is over $400\times$ smaller than GPT-4.
\begin{table}[t]
\centering
\caption{Baseline methods and their corresponding model sizes.}
\label{tab:model_sizes}
\resizebox{\linewidth}{!}{
\begin{tabular}{l l l}
\toprule
\textbf{Method} & \textbf{Model} & \textbf{\# Parameters} \\
\midrule
\rowcolor{gray!20} \ourmethod{} & microsoft/deberta-v3-large & 418M \\
Semantic Similarity & all-MiniLM-L6-v2 & $\sim$22.7M \\
GPT4o & GPT4o & Not disclosed (est. $>$200B) \\
ReContriever & Yibin-Lei/ReContriever & $\sim$110M \\
\bottomrule
\end{tabular}
}
\end{table}
 
\subsection{Evaluation Metric and Analysis}\label{appx:eval_prompts}

\paragraph{H1 Evaluation:}  
In H1, we compare the distribution of maximum utility scores between assumed personalized and non-personalized inputs. Specifically, we evaluate utility scores on a non-personalized dataset (\fb{}) and a personalized dataset (\pf{}). We sampled $n=500$ inputs from each dataset and computed the maximum utility score using \cref{eqn:method_utility} for every input. The resulting distributions were then analyzed using divergence metrics and hypothesis testing.  

For divergence metrics, we use Jensen–Shannon (JS) divergence, a symmetric measure of similarity between two probability distributions \(P\) and \(Q\). It is defined as:  

\[
\text{JS}(P \| Q) = \frac{1}{2} \, \text{KL}(P \| M) + \frac{1}{2} \, \text{KL}(Q \| M), \quad \text{where } M = \frac{P+Q}{2}
\]

Here, \(\text{KL}(P \| M)\) is the Kullback–Leibler divergence from \(P\) to the average distribution \(M\). JS divergence is always finite, symmetric, and bounded between 0 and \(\log 2\).  

For hypothesis testing, we use the two-sample Kolmogorov–Smirnov test to assess whether the two samples come from the same distribution. The null hypothesis states that they do, and we test significance at the $\alpha = 0.05$ level.

\paragraph{H1' Evaluation:}  
In H1', performance was evaluated using human labels as the gold standard. For each sample, four human annotators rated whether $n=100$ samples would benefit from personalization using a 0–2 scale. If the annotators’ confidence in personalization was 1 or higher, we considered the input to be personalized. Instructions provided to annotators are given in \cref{prompt:humanannotation_binary}. We only included inputs with unanimous agreement ($n=64$) in the evaluation. The average interrater agreement between pairs of annotators was $\kappa = 0.59$ using the Cohen's kappa. The primary metrics used were recall, measuring accuracy in identifying personalized inputs, and specificity, measuring accuracy in identifying non-personalized inputs.

\begin{prompt}\label{prompt:humanannotation_binary}
Instructions: Read the user input and then decide if the input would benefit from specific user information.

 0 = No Personalization Needed, 1 = Personalize is might be helpful, 2= Personalization is definitely helpful
\end{prompt}

\paragraph{H2 Evaluation:}  
For H2, we evaluated how closely each method's selected subset of memory items matched the subsets chosen by humans. To create the human-selected subsets, four annotators were asked to select all memory items (from the full set of 50) that they felt would be helpful in generating a response. Instructions provided to annotators are given in \cref{prompt:humanannotation_memory items}. For each input, we included only the memory items selected by at least three annotators to form the majority subset. This process resulted in an average of $4.19$ memory items per query for inputs where at least one item was chosen. This average guided our experimental design for H2 and H3, as we restricted each method to select no more than $n=5$ items.  

We evaluated performance using precision, recall, and F1 score. Here, precision measures the fraction of items selected by a method that were also in the human majority subset, recall measures the fraction of items in the human majority subset that were correctly selected by the method, and F1 is the harmonic mean of precision and recall, providing a single metric that balances both.

\begin{table}[h]
\centering
\caption{Average number of features selected for personalized vs.\ not-personalized inputs. Methods closes to human performance are \textbf{bolded}.}
\label{tab:H2_avg_feat}
\begin{tabular}{lccc}
\toprule
 & Personalized & Not-Personalized & Difference \\
\midrule \rowcolor{gray!15}
Human Annotation           & 6.45 & 1.43 & 5.02 \\
\ourmethodutility{}              & 1.87 & 0.80 & 1.07 \\
\ourmethodconserv{}       & 4.87 & \textbf{1.56 }& \textbf{3.31} \\
\ourmethodlessconserv{}      & 3.27 & 1.67& 1.60 \\
Semantic Similarity & \textbf{5.00} & 5.00 & 0.00 \\
GPT4-ZeroShot                      & 3.30 & 2.61 & 0.69 \\
GPT4-FewShot                      & 1.80 & 1.57 & 0.23 \\
\bottomrule
\end{tabular}
\end{table}
 \paragraph{Number of Memory Items Selected in H2} Lastly, we include table \ref{tab:H2_avg_feat} which compares the average number of memory items selected by each method in H2. Of note, we see that human annotations show a clear separation: personalized inputs have on average $6.5$ features selected, compared to only $1.4$ for non-personalized inputs, a difference of $5$. Among automatic methods, \ourmethod{}-Model and \ourmethod{}-Utility comes closest to this behavior, achieving a range of $3.3$ and $1.08$, whereas the next closest method is GPT4-ZeroShot which achieves only a $0.67$ difference. This further highlights how \ourmethod{} excels at both distinguishing when memory items are beneficial as well as which features are beneficial. 

\begin{prompt}\label{prompt:humanannotation_memory items}
For each input, decide which (if any) user-features would be helpful to better respond to the input. 
\end{prompt}

\paragraph{H3 Evaluation:}  
In H3, we evaluate the downstream impact of the memory subsets chosen by each method on generated responses. For each user input, responses are generated using the query combined with the subset of memory items selected by each method. This evaluation is performed with two models: GPT-4 \citep{openai2024gpt4technicalreport}, a high-performing model with robust capabilities, and LLaMA 70B Instruct \citep{meta_llama_3_1_70b_instruct}, a smaller model more susceptible to noise. This setup demonstrates \ourmethod{}-Models’s flexibility in adapting to models with varying capacities and usage scenarios.  

Performance is measured using \textit{generation win rate}, defined as the proportion of pairwise comparisons in which \ourmethod{}-Models’s output is judged preferable to that of a baseline. GPT-4 is used as the judge, comparing responses generated with \ourmethod{}-Models’s selected subset against those generated with each baseline’s subset. LLM-as-Judge is validated to align closely with third-party human annotations for personalization when combined with in-context examples, model confidence scores \citep{dong-etal-2024-llm}.  

To ensure robust evaluation, our methodology includes:  
\begin{itemize}
    \item \textit{In-context examples}: Three examples are provided in the prompt for each evaluation to demonstrate the expected style of reasoning and scoring.  
    \item \textit{Confidence thresholds}: Only responses for which the model reports a self-generated confidence score above the 70\% percentile are included in the final analysis.  
    \item \textit{Chain-of-thought prompting}: Evaluation prompts encourage the model to reason step-by-step and justify its preference for one response over another, improving consistency and interpretability of the win-rate comparisons.  
\end{itemize}

For personalized datasets (\pf{} and \wcp{}), the evaluation metrics in the chain-of-thought include:  
\begin{itemize}
    \item \textit{Completeness}: Does the response fulfill the user’s request?  
    \item \textit{Personalization}: Does the response appropriately incorporate relevant details from the user profile?  
    \item \textit{Overpersonalization}: Does the response include irrelevant, distracting, or unnecessary user details?  
\end{itemize}

For non-personalized datasets (\fb{} and \wcnp{}), the metrics are:  
\begin{itemize}
    \item \textit{Completeness/Accuracy}: Does the response correctly and fully answer the user’s request?  
    \item \textit{Neutrality}: Is the response unbiased and general, rather than tailored to the user?  
\end{itemize}

Full evaluation prompts can be found in \cref{prompt:H3eval_personalized} for personalized datasets, \cref{prompt:H3eval_nonpersonalized} for non-personalized datasets, and \cref{prompt:H3eval_incontext} for the three in-context examples used in all evaluation.

\begin{prompt}\label{prompt:H3eval_incontext}
Example 1:

{{"user profile": ['Preferred exercise activities: Weightlifting', "Hobbies and interests: ['Surfing', 'Photography']", 'Location (City/State/Country): Melbourne/Victoria/Australia', "Social media platforms used: ['Facebook', 'LinkedIn']", "Technology usage (devices, apps): ['Smartphone', 'Laptop']],

    "user input": " Why are there five rings in Olympic logo? And what does they symbolize?”,
    
    "response\_1": " The five rings in the Olympic logo represent the five continents: Africa, the Americas, Asia, Europe, and Oceania. They are interlinked to symbolize the unity and friendship among nations, highlighting the idea of bringing together athletes from around the world to compete in the spirit of sportsmanship. The colors of the rings (blue, yellow, black, green, and red) were chosen because at least one of these colors appears in the national flags of all participating countries.”,
    
    "response\_2": “The five rings in the Olympic logo represent the five continents of the world‚ Africa, the Americas, Asia, Europe, and Oceania. This design, created by Pierre de Coubertin in 1913, symbolizes the unity and connection among the nations and cultures participating in the Olympic Games. The rings are interlocked, indicating that these continents are united through the spirit of sport, competition, and the Olympic ideals of excellence, friendship, and respect. The colors of the rings (blue, yellow, black, green, and red) along with the white background were chosen because at least one of these colors appears on the flag of every nation competing in the Olympics.”,
    
“overall”: “0”,

“overall\_explanation”: “Although they are worded a bit differently, they generally have the same information and do not include unhelpful user-information in the answer.”}}

Example 2:

{{"user profile": ["Social media platforms used: ['Instagram', 'Facebook']", 'Dietary restrictions (e.g., vegetarian, vegan): None', 'Gender: Female', "Pet ownership (types of pets): ['Cat']", 'Current projects or goals: Continue education', 'Religious beliefs: Christian', 'Favorite TV shows: ["Grey\'s Anatomy"]', 'Reasons for using the service: Medical updates', 'Current challenges or pain points: Shift work', 'Travel preferences (adventurous, relaxing): Relaxing'],

    "user input": "The 2011 album 'The King of Limbs' was released by which English rock band?”,
    
    "response\_1": "The 2011 album 'The King of Limbs' was released by the English rock band Radiohead.”,
    
    "response\_2": “The album 'The King of Limbs' was released by the English rock band Radiohead”,
    
“overall”: “0”,

“overall\_explanation”: “These answer are the same, minus only minor differences”}}

Example 3:

{{"user profile": ['Current challenges or pain points: Commuting stress', 'Location (City/State/Country): Bangalore/Karnataka/India', 'Favorite sports: Cricket', "Personal values or beliefs: ['Kindness', 'Resilience']", 'Political affiliation: Centrist', 'Name: Arjun Mehta', 'Reasons for using the service: Access tools', 'Financial situation (broad categories): Financially modest', 'Level of tech-savviness: Moderate', 'Cultural background: South Asian', 'Frequency of using the service (daily, weekly): Weekly', "Favorite TV shows: ['Sacred Games']", 'Travel preferences (adventurous, relaxing): Relaxing', 'Preferred mode of communication (text, voice, video): Voice', "Favorite foods: ['Dosa', 'Chaat']", 'Health and fitness goals: Improve health'],

    "user input": " Are there any recommended effective ways to enhance the fun of English grammar learning?”,
    
    "response\_1": "Hi Arjun! Here are some recommended ways to enhance the fun of learning English grammar:
1. **Games and Quizzes**: Engage in grammar games or online quizzes that challenge your understanding in a fun way. Websites like Kahoot! or Quizlet offer interactive tools.
2. **Group Learning**: Form a study group with friends or peers. This can make the learning process more social and less intimidating. You can discuss grammar topics over chai or dosa!
3. **Music and Lyrics**: Listen to English songs and try to analyze the lyrics. Identifying verb tenses or other grammatical structures within your favorite songs can be both educational and entertaining.
4. **Humor in Learning**: Explore humorous content, such as grammar memes or funny videos about grammar mistakes, to lighten the mood.
5. **Role Play**: Practice dialogues using correct grammar in fun role-playing scenarios that interest you‚ perhaps even related to cricket!
By combining these methods with your existing interests and values, like resilience and kindness, you can make grammar learning a much more enjoyable experience. Happy learning!”,

    "response\_2": “Improving the enjoyment of English grammar learning can be approached in various fun and engaging ways. Here are some recommendations that might resonate with you:
1. **Gamify Learning**: Use educational games and apps that focus on grammar. Many platforms have quizzes, challenges, and rewards that make learning feel like a fun competition.
2. **Interactive Activities**: Engage in activities like grammar charades or sentence-building races with friends or family. This adds a social element to learning.
3. **Threshold Learning**: Set small, achievable goals and reward yourself when you reach them. This creates a sense of accomplishment and makes the process more enjoyable.
4. **Join a Language Group**: Participate in local or online language exchange groups where you can practice speaking with others. This can be a relaxing way to enhance your skills without the pressure of formal study.
Would you like any specific resources or activities tailored to your preferences?”,

“overall”: “2”,

“overall\_explanation”: “Response 1 is over personalized to the profile and inserts user-information in an unnatural and unhelpful way”}}
\end{prompt}

\begin{prompt}\label{prompt:H3eval_personalized}
You are an impartial and knowledgeable evaluator. You will be shown a user profile, a user input, and two model-generated responses. Your task is to decide if **only one of the responses** meets the evaluation criteria. If both responses do (or neither does), mark it as a tie (0).

    Evaluation Criteria:

    - **Completeness**: Does the response fulfill the user’s request? 
    
    - **Personalization**: Does the response appropriately use relevant details from the user profile?  
    
    - **Overpersonalization**: Does the response include irrelevant, distracting, or unnecessary details about the user?  
    
    - **Overall**: Taking into account completeness, and being appropriately personalized without being overly personal, is one response clearly better than the other?

     Important Notes:
    
    - Assign a score of 1 if **only response 1** satisfies the criterion.  
    
    - Assign a score of 2 if **only response 2** satisfies the criterion.  
    
    - Assign a score of 0 if **both responses** satisfy it or **neither does**.  
    
    - Do not favor length or style unless it clearly affects the criterion.  
    
    - Be strict: if the difference is marginal, mark it as a tie (0).

     Output Instructions:

    Each of the 4 score fields must be:  
    
    - `"1"` = only response 1 satisfies the criterion  
    
    - `"2"` = only response 2 satisfies the criterion  
    
    - `"0"` = both or neither satisfy the criterion  
    
    Additionally, assess your confidence in this decision by assigning a certainty level from 1 to 100. Use the following guidelines to assign the certainty level:
    
    1–20 (Uncertain): There is insufficient or minimal evidence. The decision is largely based on weak or indirect hints.
    
    21–40 (Moderately Confident): There is noticeable evidence supporting a preference, though it is not comprehensive, and other interpretations are possible.
    
    41–60 (Quite Confident): There is clear and convincing evidence that supports this decision, though it is not entirely decisive.
    
    61–80 (Confident): There is strong evidence that clearly supports this decision, with very little ambiguity.
    
    81–100 (Highly Confident): There is direct and explicit evidence that decisively supports this prediction.

    Lastly, include a **brief (1–2 sentences)** and justify the score using the criteria above.  
    
    You must return a JSON object with these 12 keys:
    
    - "completeness", "completeness\_explanation", "completeness\_confidence"
    
    - "personalization", "personalization\_explanation", "personalization\_confidence"
    
    - "overpersonalization", "overpersonalization\_explanation", "overpersonalization\_confidence"
    
    - "overall", "overall\_explanation", "overall\_confidence"
    
    **Do not include any text outside the JSON block.**
    Here is the evaluation task:
    
    Given: {{
    "user input": "{user\_input}",
    "response\_1": "{response\_a}",
    "response\_2": "{response\_b}"
    }} 
    Generation: 
\end{prompt}

\begin{prompt}\label{prompt:H3eval_nonpersonalized}
     f'''You are an impartial and knowledgeable evaluator. You will be shown a user input and two model-generated responses. Your task is to decide if **only one of the responses** meets the evaluation criteria. If both responses do (or neither does), mark it as a tie (0).

    Evaluation Criteria:

    - **Completeness \& Accuracy**: Does the response correctly and fully answer the user’s request?  
    
    - **Neutrality**: Is the response unbiased and general, rather than tailored to the user?  
    
    - **Overall**: Taking into account completeness, accuracy, and neutrality is one response clearly better than the other?

     Important Notes:
    
    - Assign a score of 1 if **only response 1** satisfies the criterion.  
    
    - Assign a score of 2 if **only response 2** satisfies the criterion. 
    
    - Assign a score of 0 if **both responses** satisfy it or **neither does**.  
    
    - Do not favor longer responses unless they add essential information.  
    
    - Be strict: if the difference is marginal, mark it as a tie (0).

     Output Instructions:
    Each of the 4 score fields must be:  
    
    - `"1"` = only response 1 satisfies the criterion  
    
    - `"2"` = only response 2 satisfies the criterion  
    
    - `"0"` = both or neither satisfy the criterion  
    
    Additionally, assess your confidence in this decision by assigning a certainty level from 1 to 100. Use the following guidelines to assign the certainty level:
    
    1–20 (Uncertain): There is insufficient or minimal evidence. The decision is largely based on weak or indirect hints.
    
    21–40 (Moderately Confident): There is noticeable evidence supporting a preference, though it is not comprehensive, and other interpretations are possible.
    
    41–60 (Quite Confident): There is clear and convincing evidence that supports this decision, though it is not entirely decisive.
    
    61–80 (Confident): There is strong evidence that clearly supports this decision, with very little ambiguity.
    
    81–100 (Highly Confident): There is direct and explicit evidence that decisively supports this prediction.

    Lastly, include a **brief (1–2 sentences)** and justify the score using the criteria above.  
    
    You must return a JSON object with these 12 keys:
    
    - "completeness\_accuracy", "completeness\_accuracy"\_explanation", "completeness\_accuracy"\_confidence"
    
    - "neutrality", "neutrality\_explanation", "neutrality\_confidence"

    - "overall", "overall\_explanation", "overall\_confidence"
    
    **Do not include any text outside the JSON block.**

    Here is the evaluation task:
    
    Given: {{
    "user input": "{user\_input}",
    "response\_1": "{response\_a}",
    "response\_2": "{response\_b}"
    }} 
    Generation: 
\end{prompt} 
\section{Other Details}\label{appx:other}
\subsection{Software} We used Python 3.13.2, Pytorch 2.7.0, and HuggingFace Transformers 4.52.4.All code is licensed under the
Apache License 2.0.

\subsection{Hardware} All experiments were run on a cluster with 4 NIVIDIA H100 GPUs and 1 NIVIDIA A100 GPU with 80B memory. For most inference jobs we used one GPU. For supervised fine-tuning, we used one GPU. The training for twelve epochs took about 2 hours. 

\subsection{Artifact Terms of Use} 
\pf \citep{tao2025personafeedbacklargescalehumanannotatedbenchmark}: Apache-2.0
\fb{} \citep{jiang-etal-2019-freebaseqa}:  Creative Commons Attribution-NonCommercial-ShareAlike 3.0 International License
WildChat \citep{zhao2024wildchat}: AI2 ImpACT License

\section{Information About Use Of AI Assistants}\label{appx:ai_use}
In this project, AI assistants were used for results visualization code (e.g., figures and tables) and for minor writing edits.

\section*{Distribution of Processed Profile Features}

\begin{figure}[htbp]
    \centering
    \includegraphics[width=\textwidth]{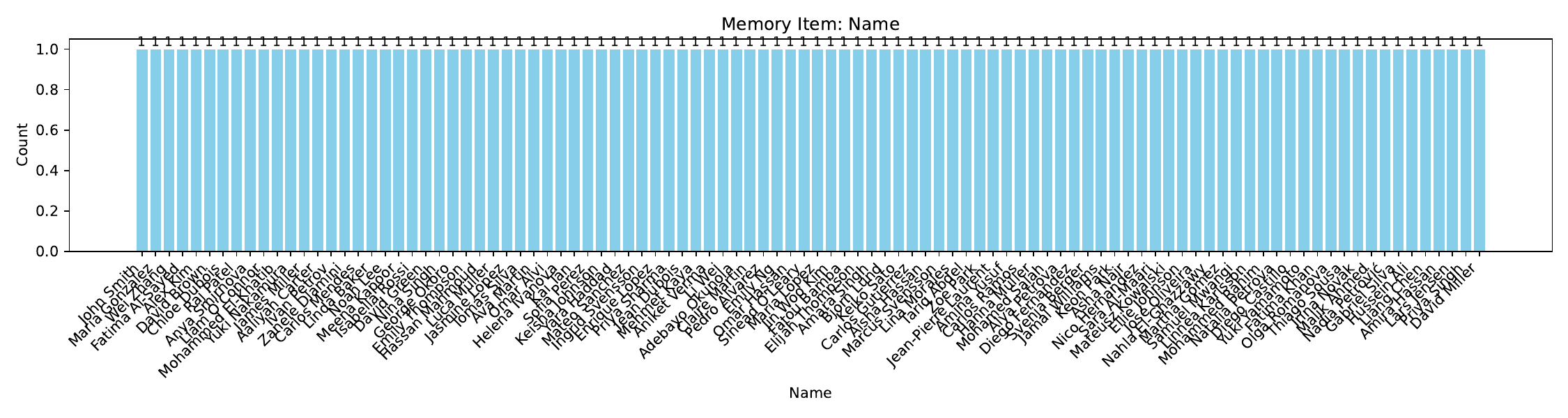}
    \caption{Distribution of Name}
    \label{fig:profile_distr_start}
\end{figure}

\begin{figure}[htbp]
    \centering
    \includegraphics[width=\textwidth]{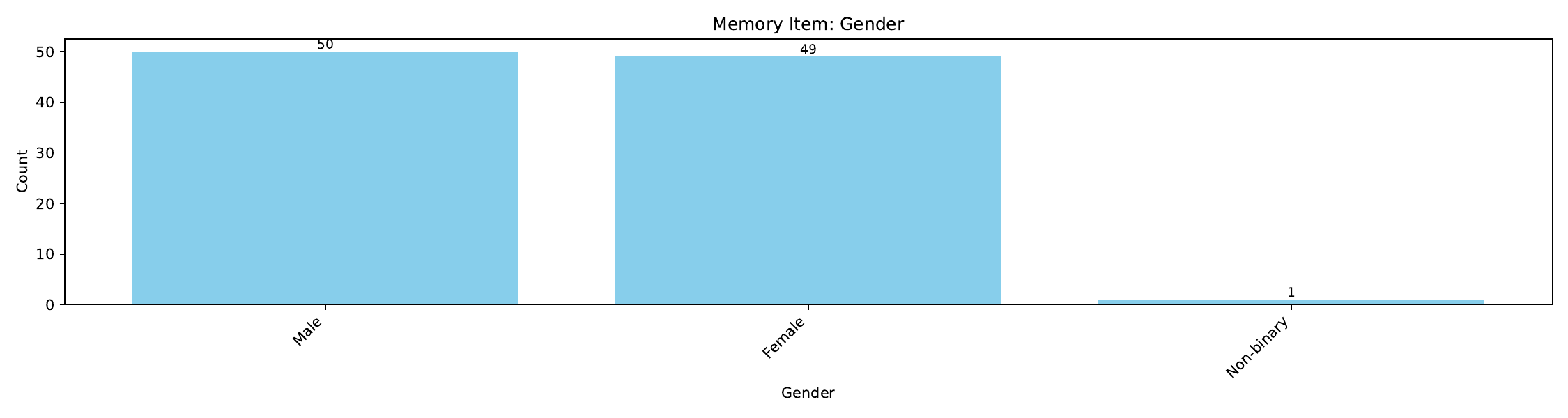}
    \caption{Distribution of Gender}
\end{figure}

\begin{figure}[htbp]
    \centering
    \includegraphics[width=\textwidth]{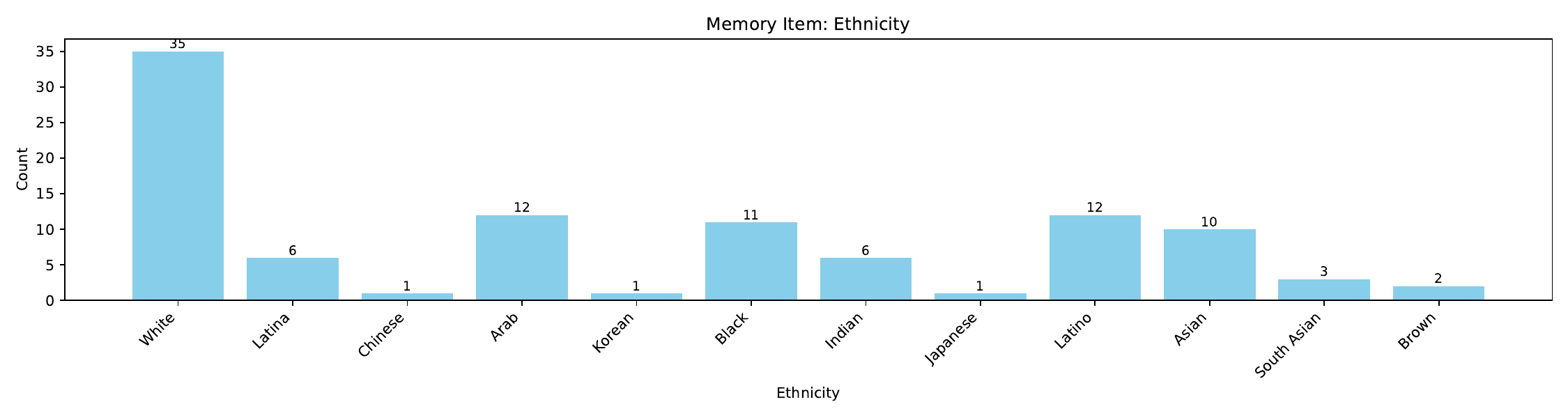}
    \caption{Distribution of Ethnicity}
\end{figure}

\begin{figure}[htbp]
    \centering
    \includegraphics[width=\textwidth]{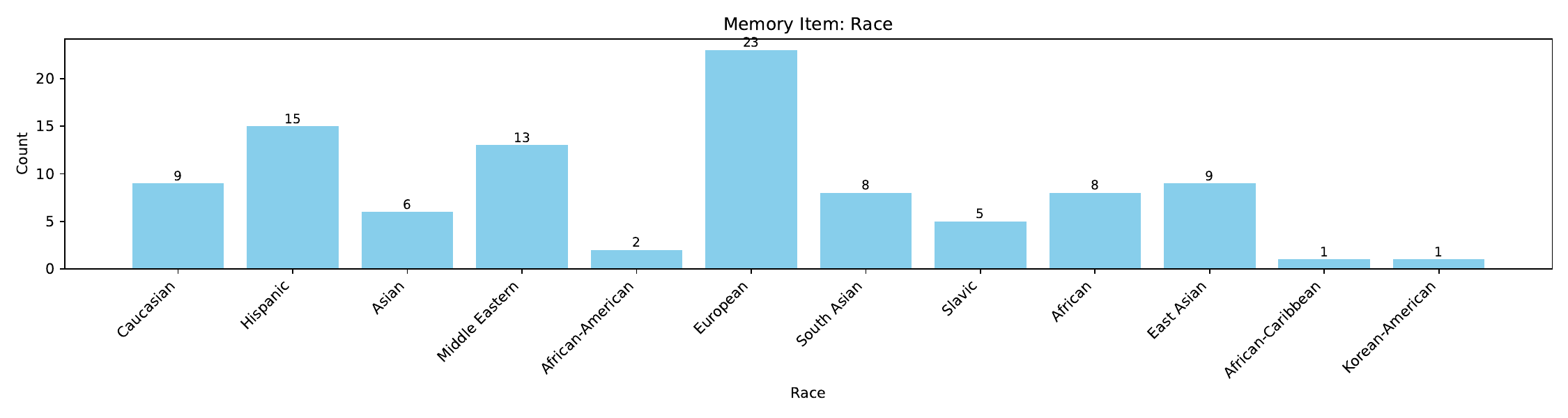}
    \caption{Distribution of Race}
\end{figure}

\newpage
\begin{figure}[htbp]
    \centering
    \includegraphics[width=\textwidth]{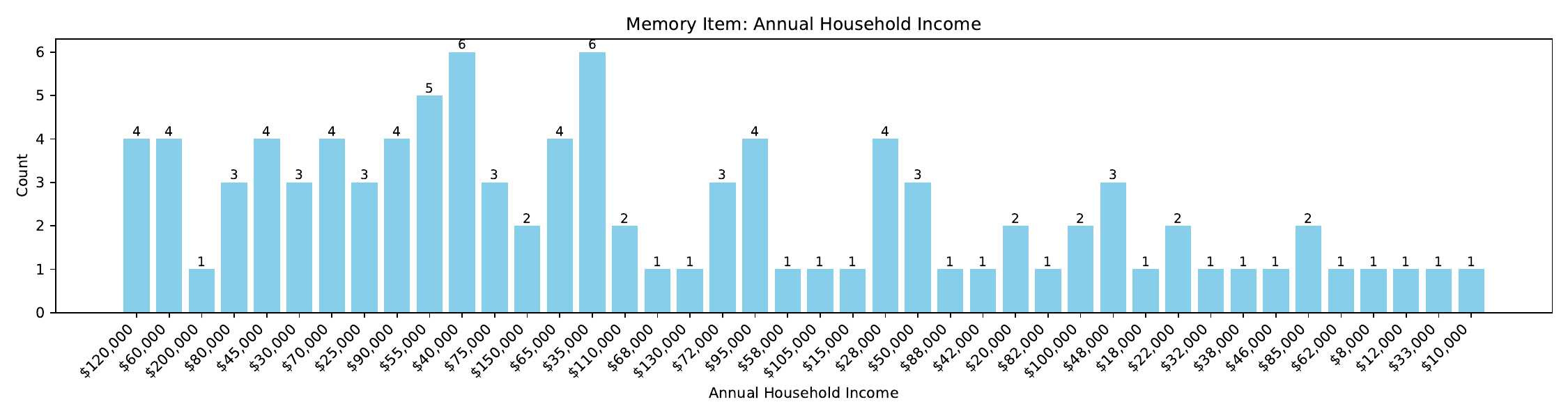}
    \caption{Distribution of Annual Household Income}
\end{figure}

\begin{figure}[htbp]
    \centering
    \includegraphics[width=\textwidth]{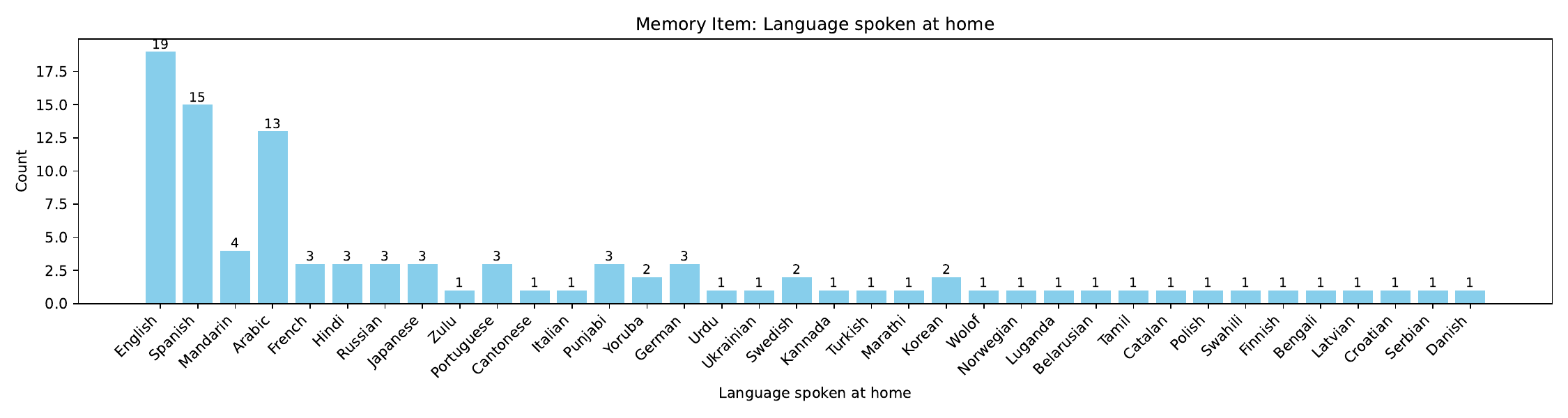}
    \caption{Distribution of Language spoken at home}
\end{figure}

\begin{figure}[htbp]
    \centering
    \includegraphics[width=\textwidth]{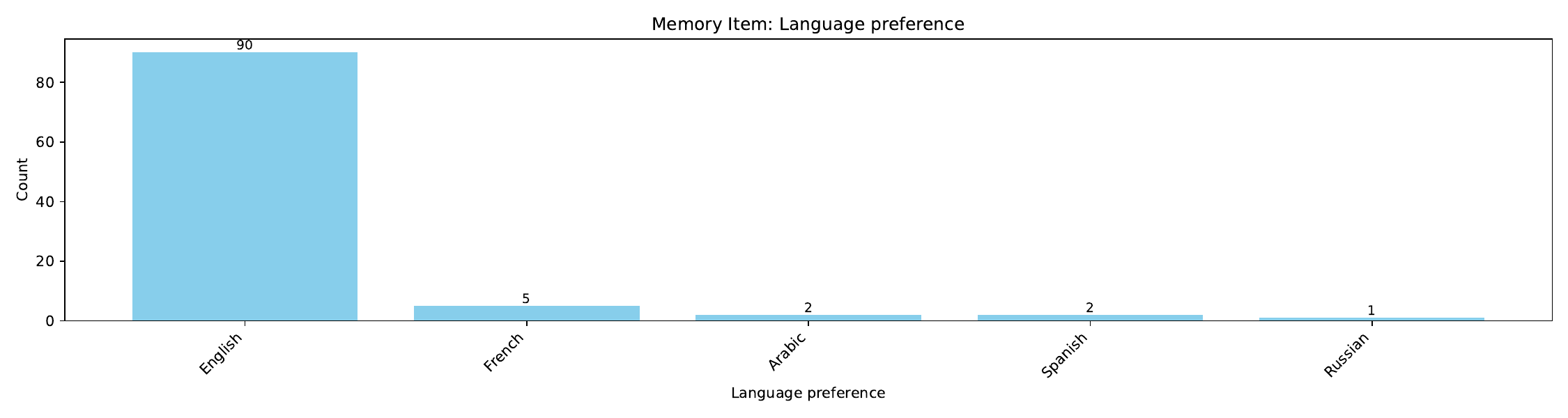}
    \caption{Distribution of Language preference}
\end{figure}

\begin{figure}[htbp]
    \centering
    \includegraphics[width=\textwidth]{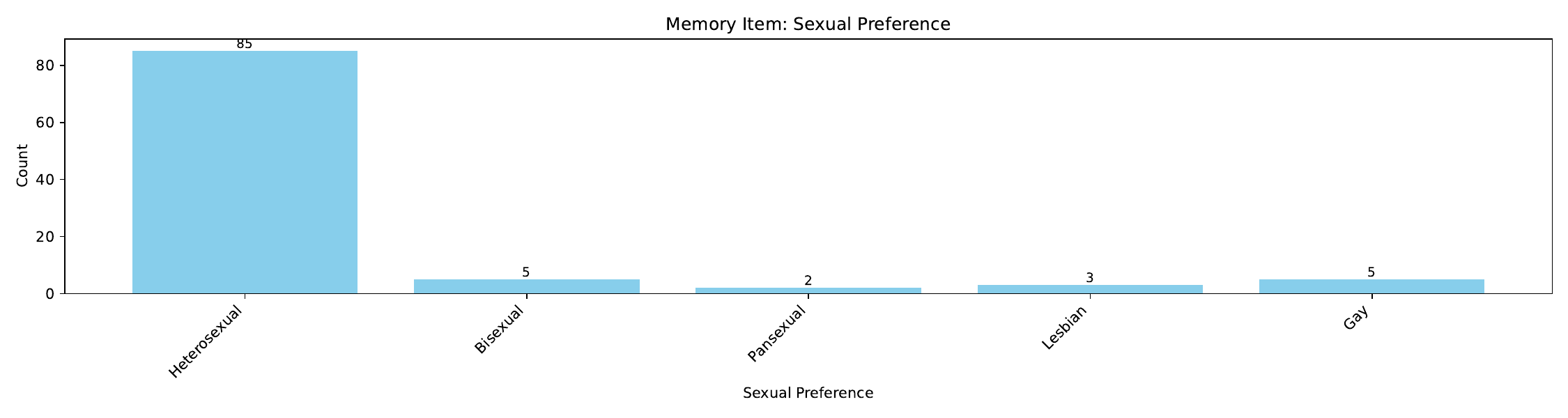}
    \caption{Distribution of Sexual Preference}
\end{figure}

\begin{figure}[htbp]
    \centering
    \includegraphics[width=\textwidth]{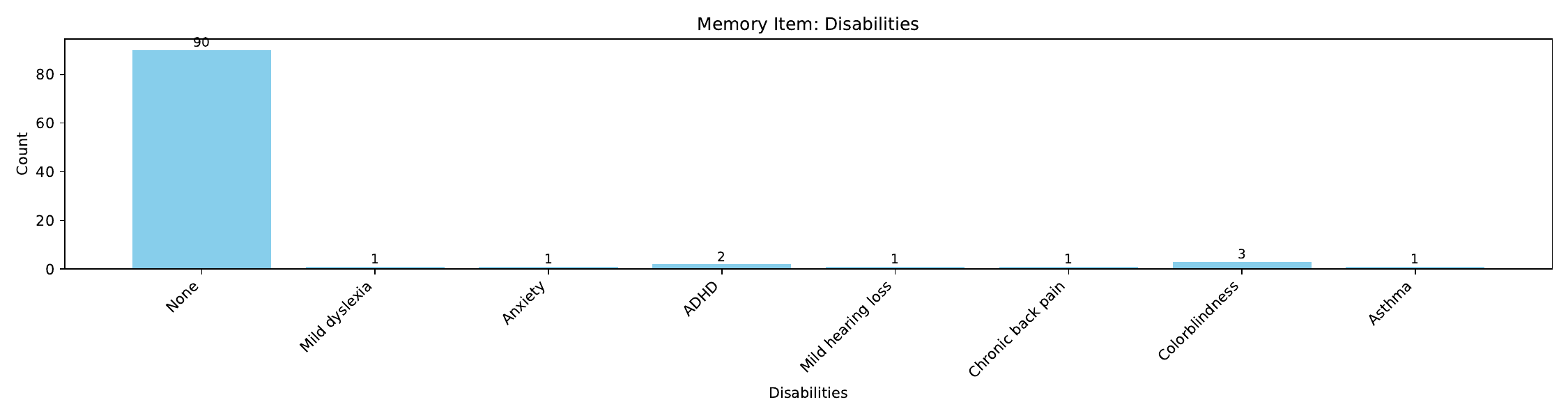}
    \caption{Distribution of Disabilities}
\end{figure}

\begin{figure}[htbp]
    \centering
    \includegraphics[width=\textwidth]{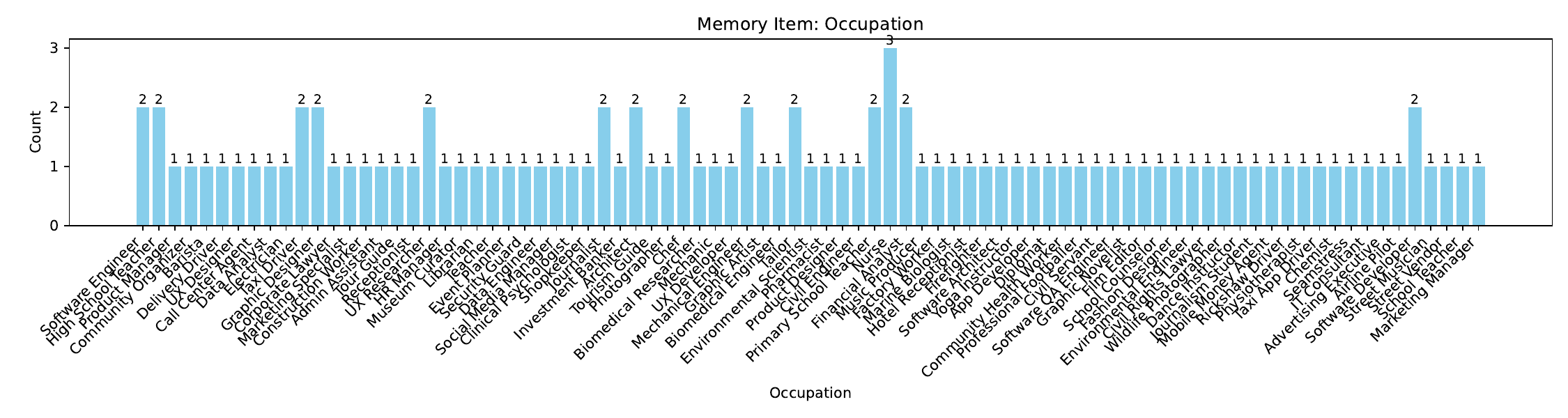}
    \caption{Distribution of Occupation}
\end{figure}

\begin{figure}[htbp]
    \centering
    \includegraphics[width=\textwidth]{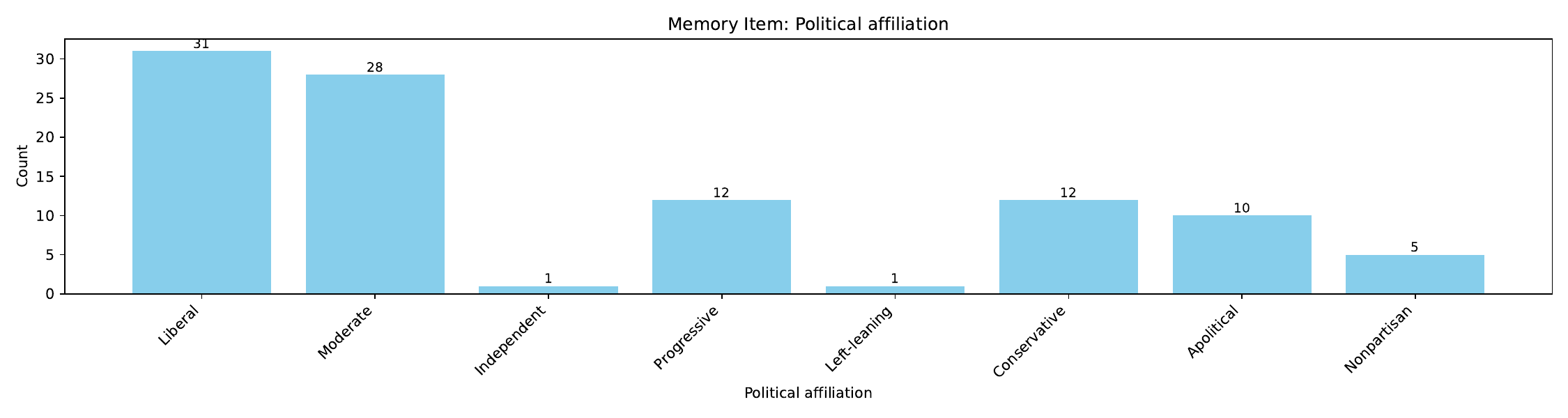}
    \caption{Distribution of Political affiliation}
\end{figure}

\begin{figure}[htbp]
    \centering
    \includegraphics[width=\textwidth]{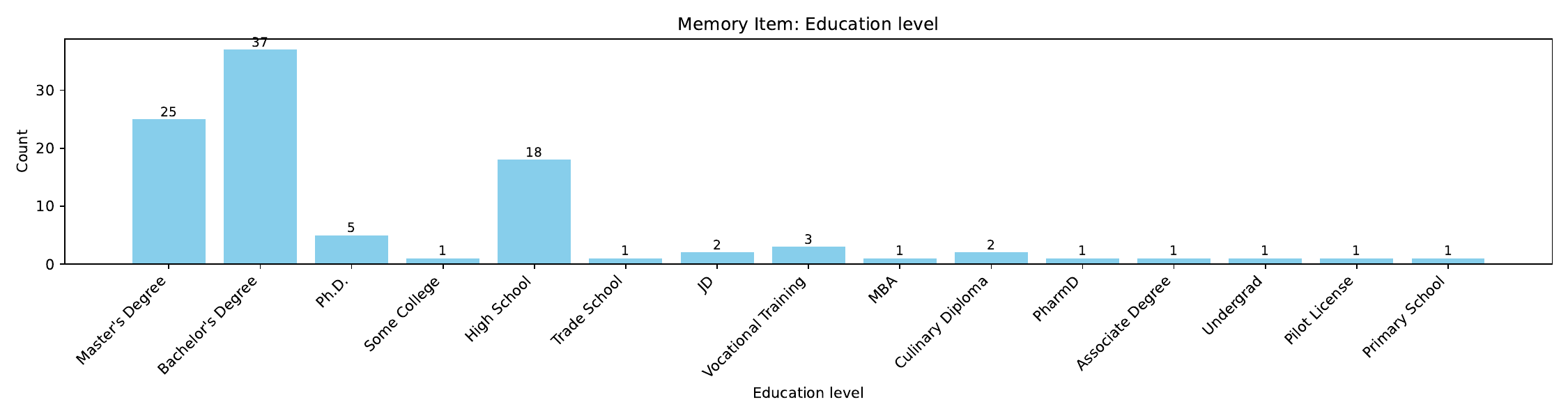}
    \caption{Distribution of Education level}
\end{figure}

\begin{figure}[htbp]
    \centering
    \includegraphics[width=\textwidth]{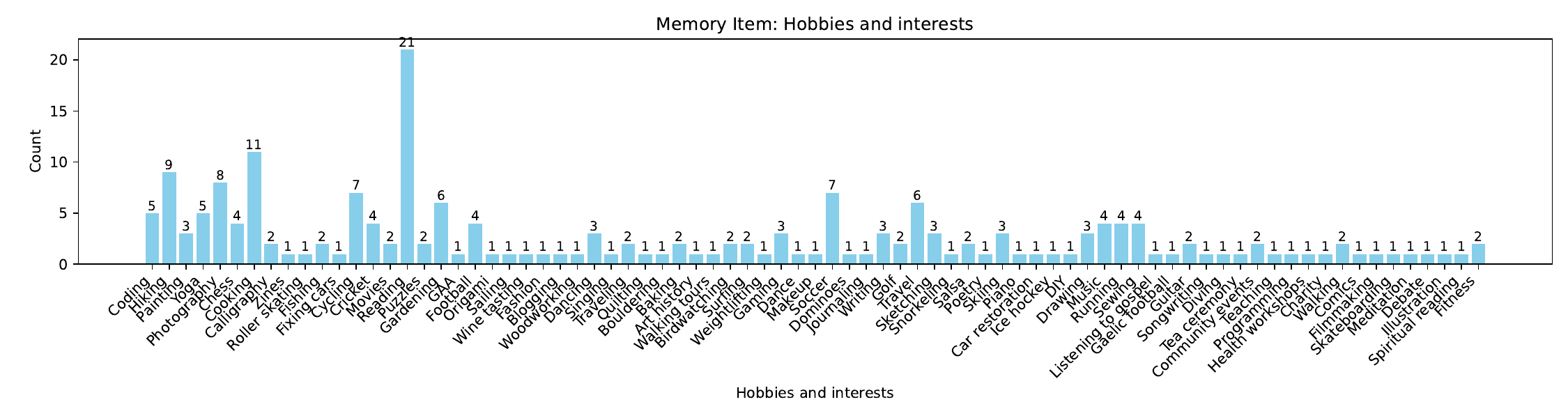}
    \caption{Distribution of Hobbies and interests}
\end{figure}

\begin{figure}[htbp]
    \centering
    \includegraphics[width=\textwidth]{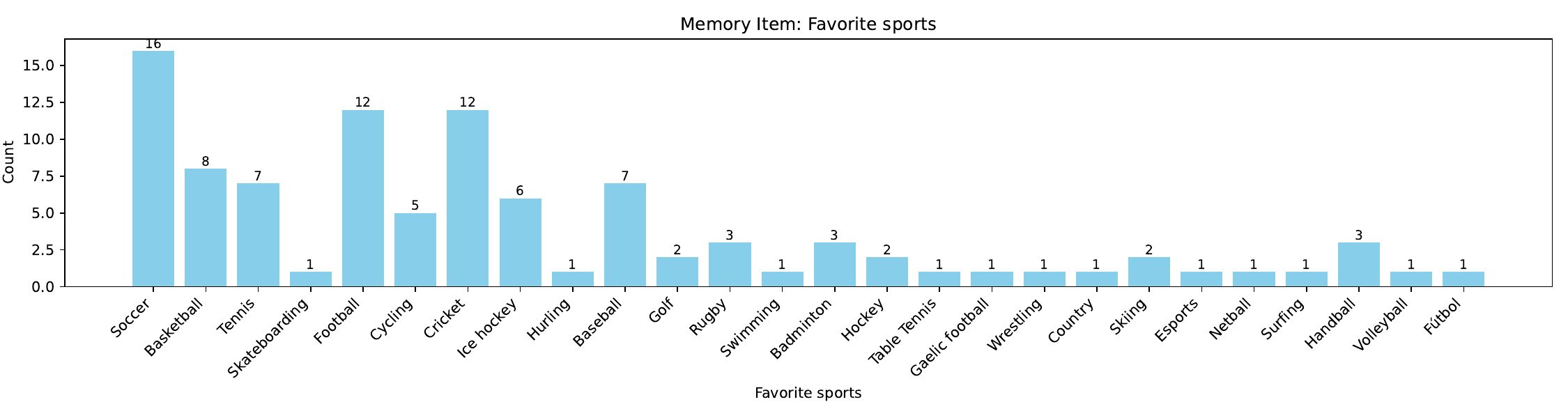}
    \caption{Distribution of Favorite sports}
\end{figure}

\begin{figure}[htbp]
    \centering
    \includegraphics[width=\textwidth]{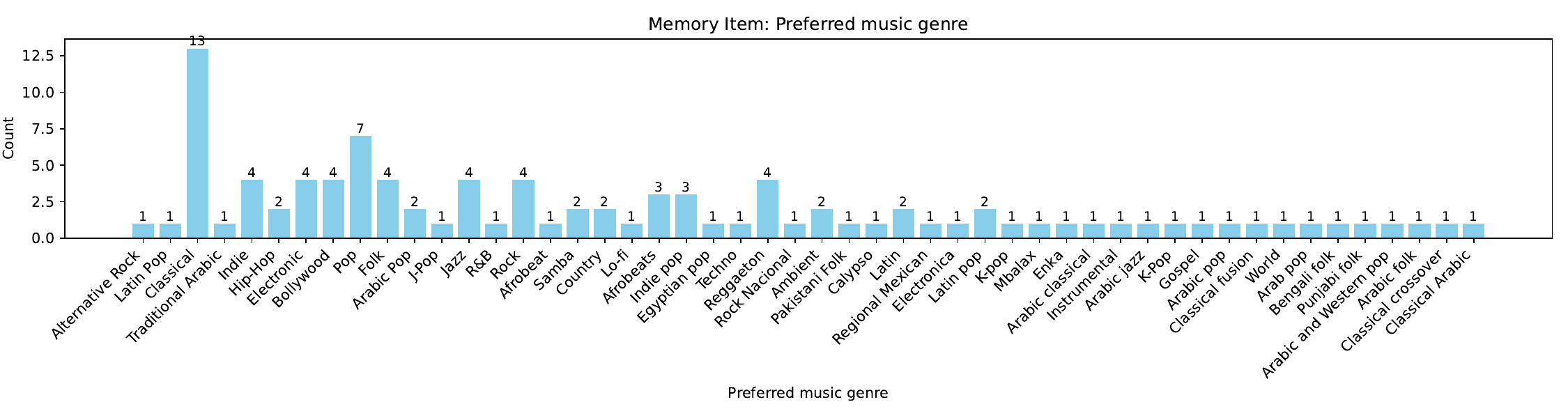}
    \caption{Distribution of Preferred music genre}
\end{figure}

\begin{figure}[htbp]
    \centering
    \includegraphics[width=\textwidth]{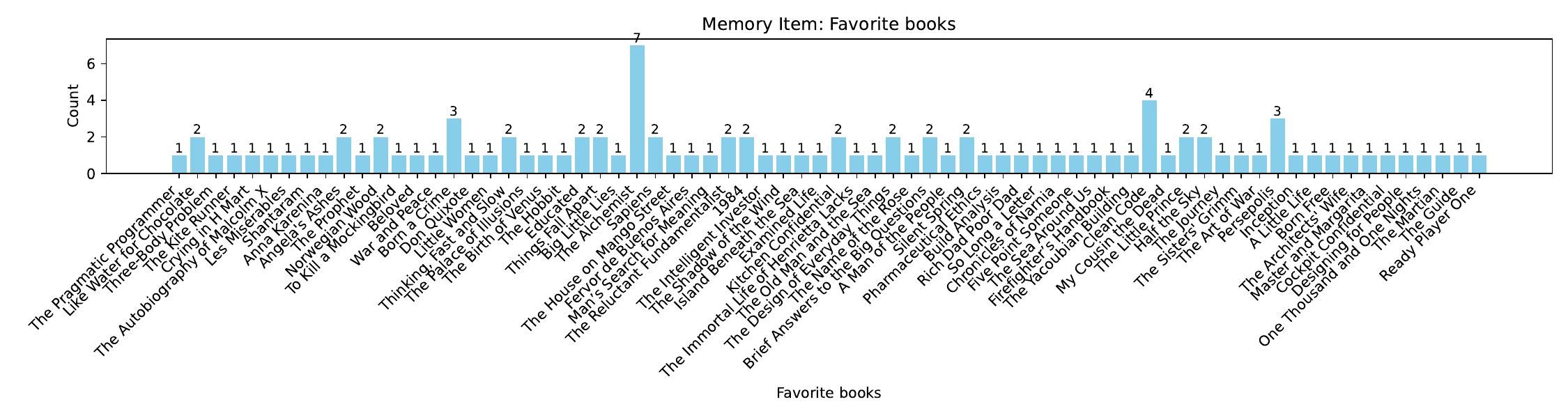}
    \caption{Distribution of Favorite books}
\end{figure}

\begin{figure}[htbp]
    \centering
    \includegraphics[width=\textwidth]{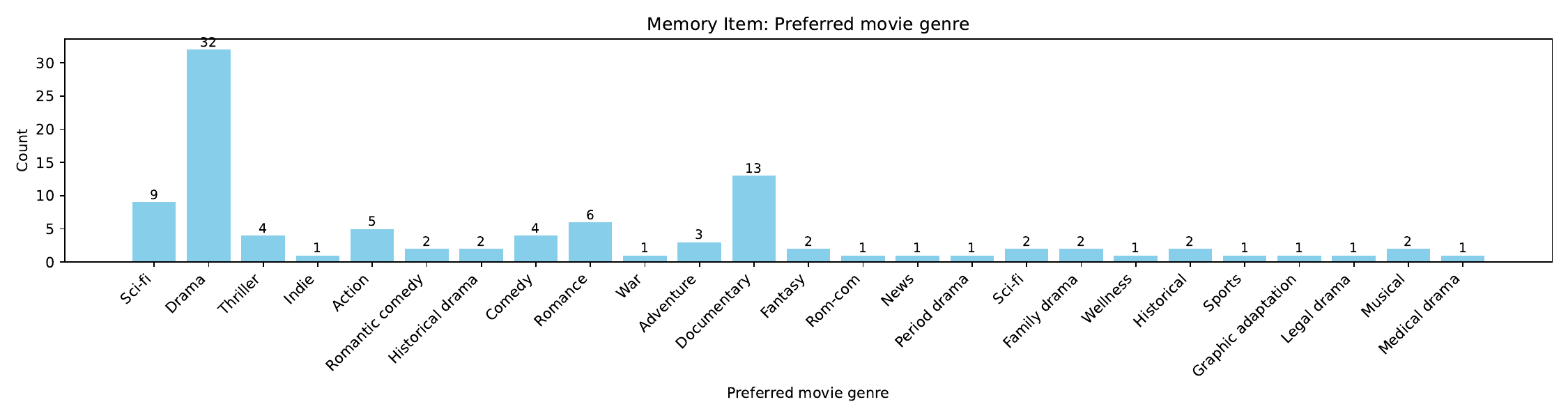}
    \caption{Distribution of Preferred movie genre}
\end{figure}

\begin{figure}[htbp]
    \centering
    \includegraphics[width=\textwidth]{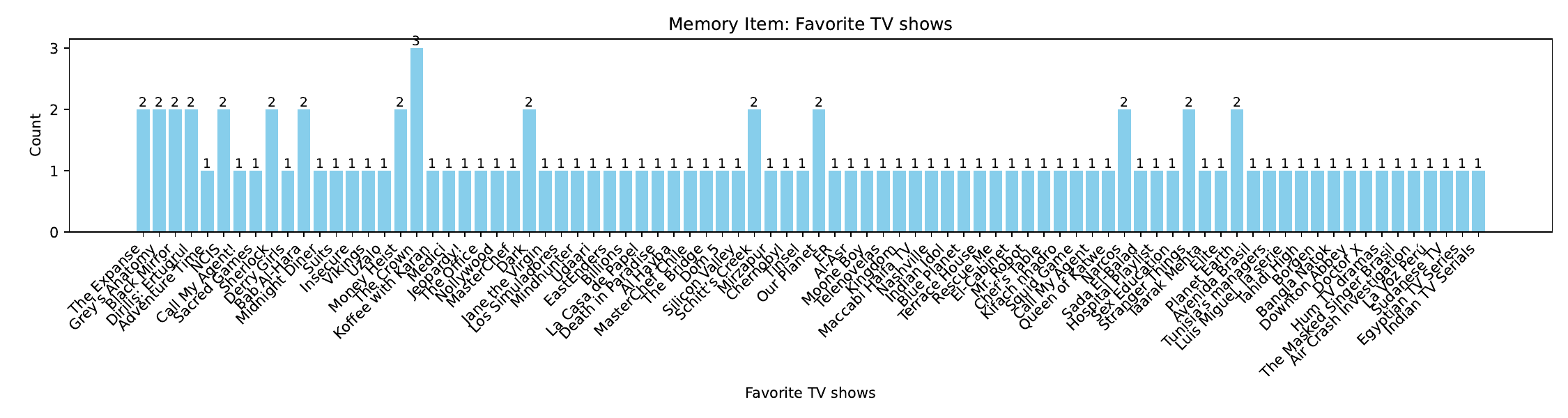}
    \caption{Distribution of Favorite TV shows}
\end{figure}

\begin{figure}[htbp]
    \centering
    \includegraphics[width=\textwidth]{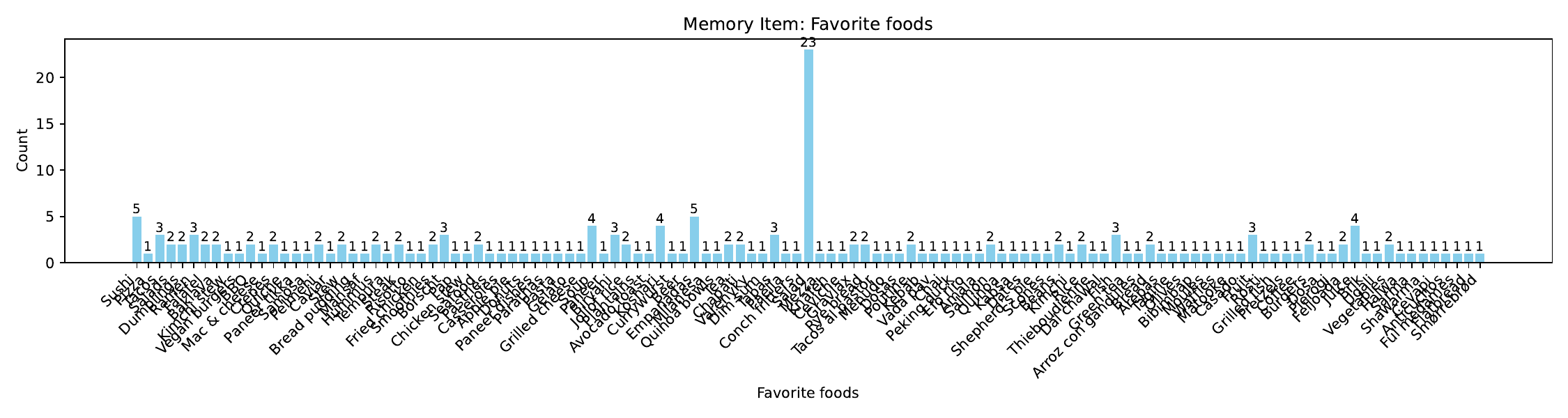}
    \caption{Distribution of Favorite foods}
\end{figure}

\begin{figure}[htbp]
    \centering
    \includegraphics[width=\textwidth]{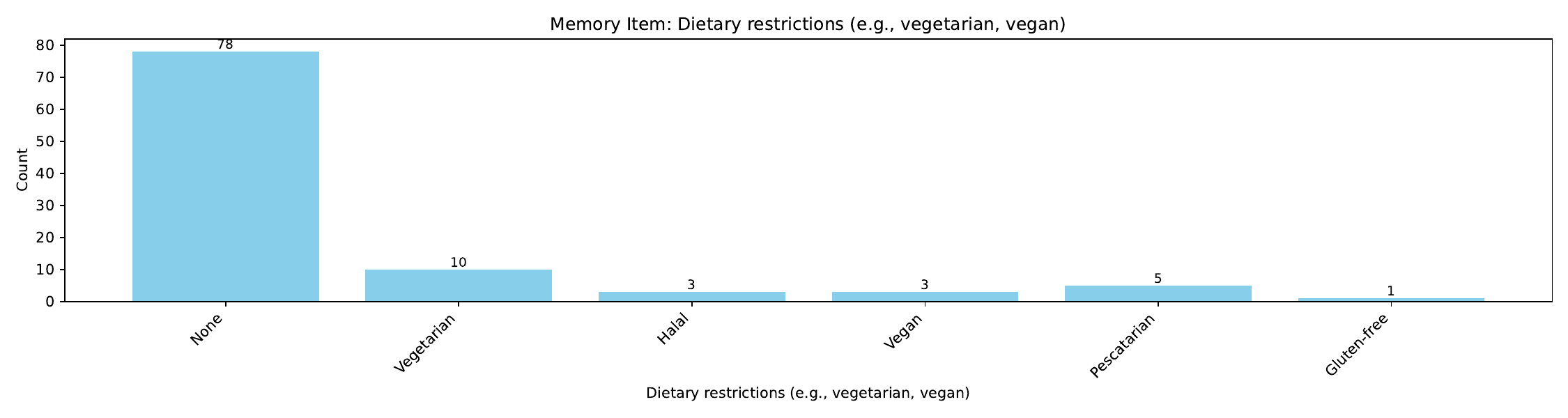}
    \caption{Distribution of Dietary restrictions (e.g., vegetarian, vegan)}
\end{figure}

\begin{figure}[htbp]
    \centering
    \includegraphics[width=\textwidth]{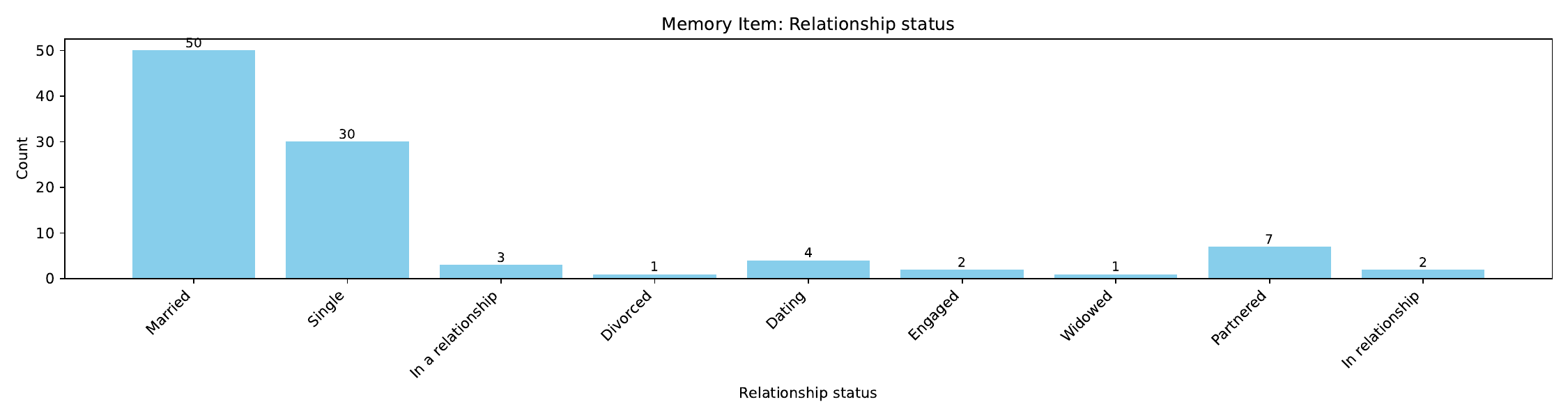}
    \caption{Distribution of Relationship status}
\end{figure}

\begin{figure}[htbp]
    \centering
    \includegraphics[width=\textwidth]{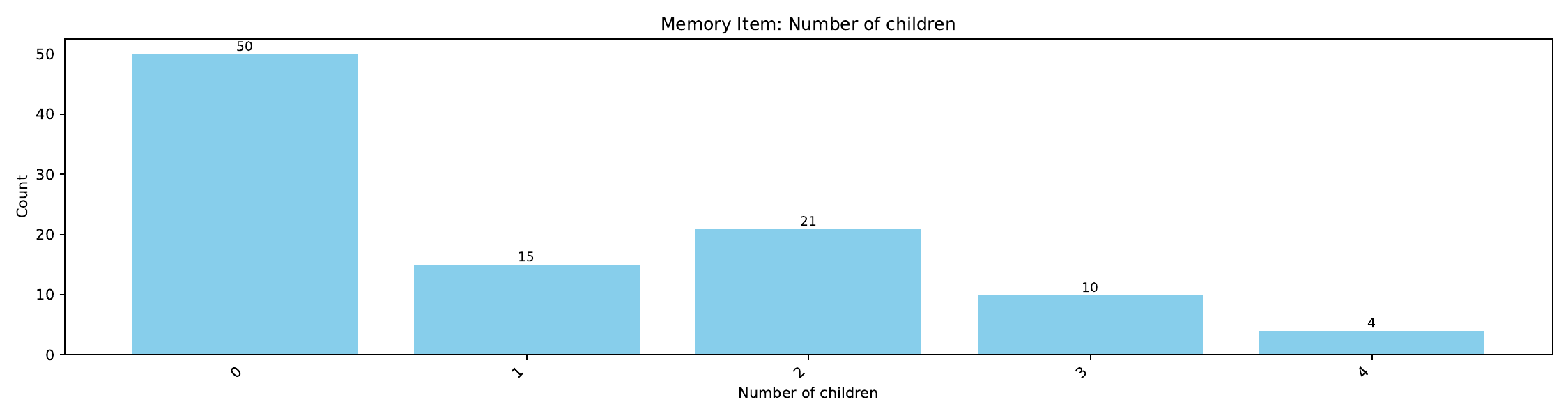}
    \caption{Distribution of Number of children}
\end{figure}

\begin{figure}[htbp]
    \centering
    \includegraphics[width=\textwidth]{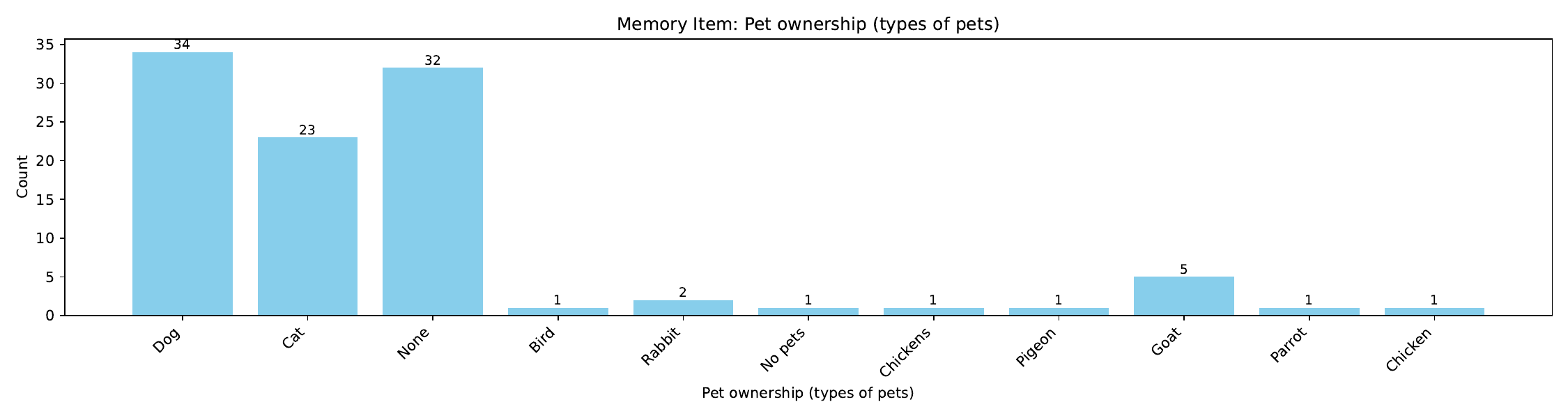}
    \caption{Distribution of Pet ownership (types of pets)}
\end{figure}

\begin{figure}[htbp]
    \centering
    \includegraphics[width=\textwidth]{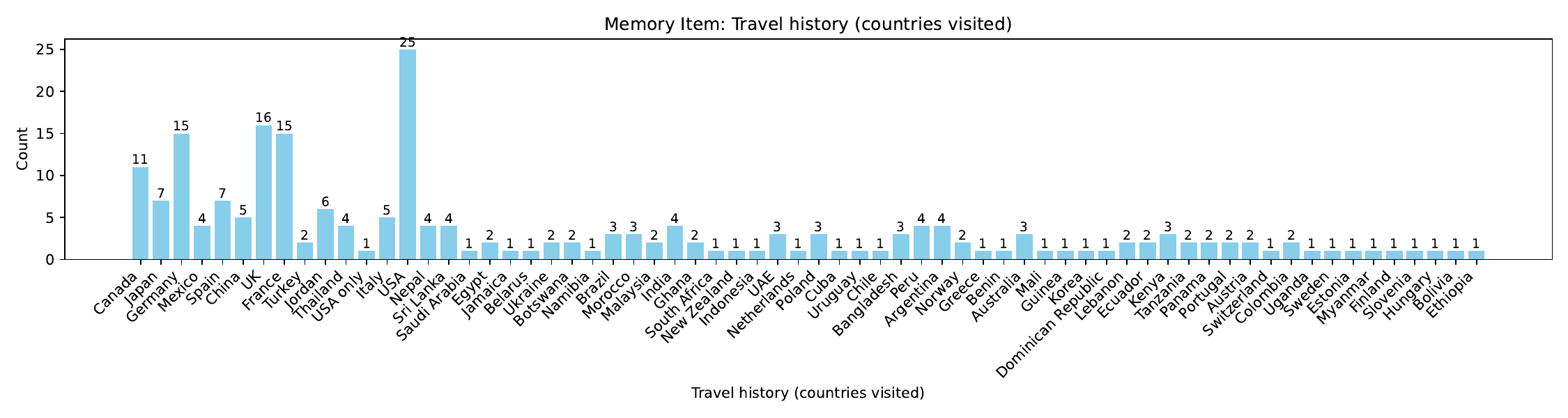}
    \caption{Distribution of Travel history (countries visited)}
\end{figure}

\begin{figure}[htbp]
    \centering
    \includegraphics[width=\textwidth]{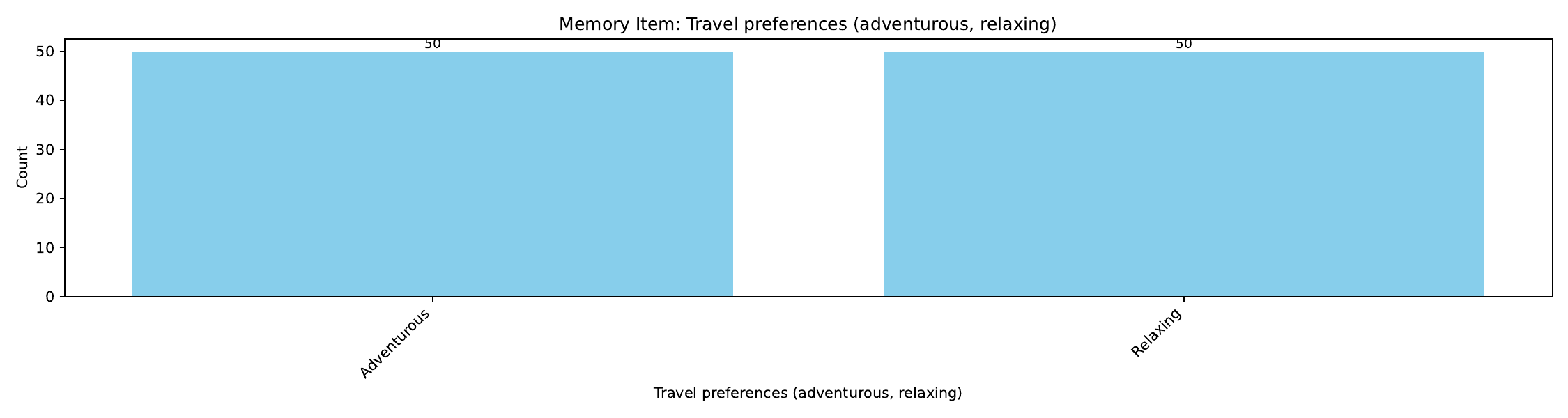}
    \caption{Distribution of Travel preferences (adventurous, relaxing)}
\end{figure}

\begin{figure}[htbp]
    \centering
    \includegraphics[width=\textwidth]{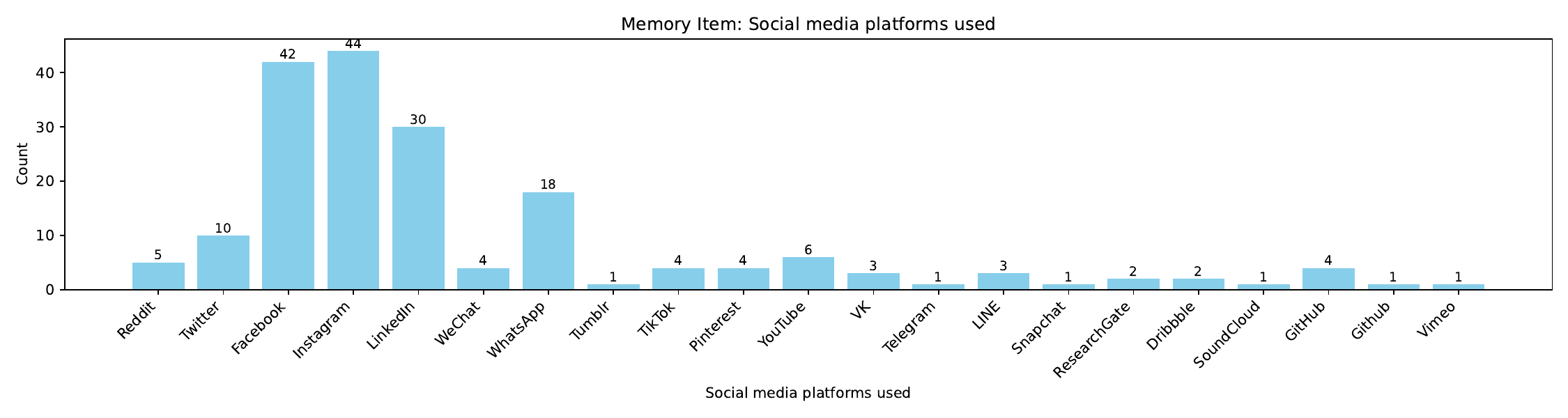}
    \caption{Distribution of Social media platforms used}
\end{figure}

\begin{figure}[htbp]
    \centering
    \includegraphics[width=\textwidth]{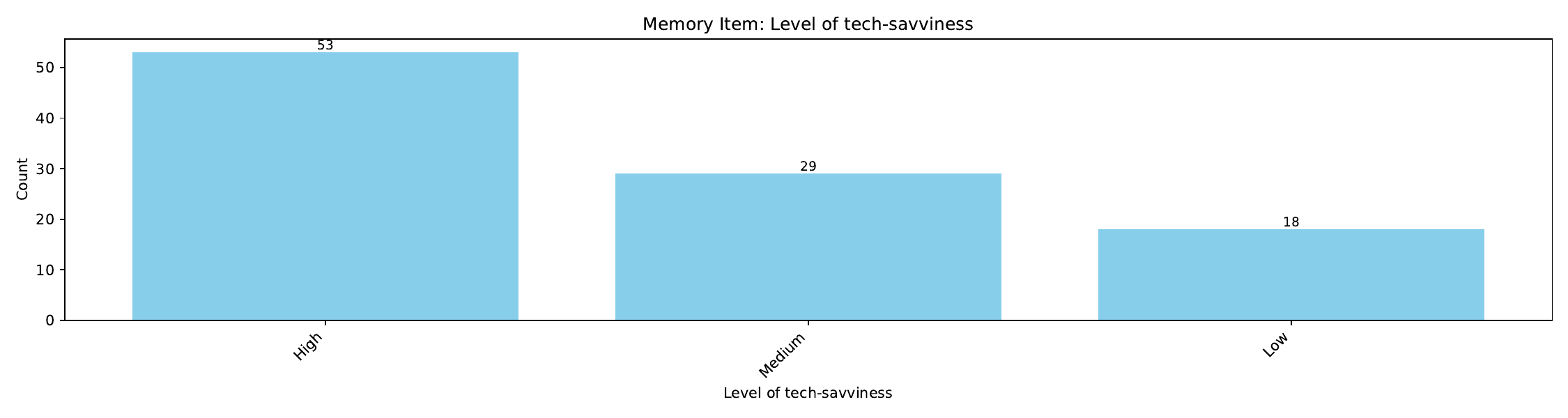}
    \caption{Distribution of Level of tech-savviness}
\end{figure}

\begin{figure}[htbp]
    \centering
    \includegraphics[width=\textwidth]{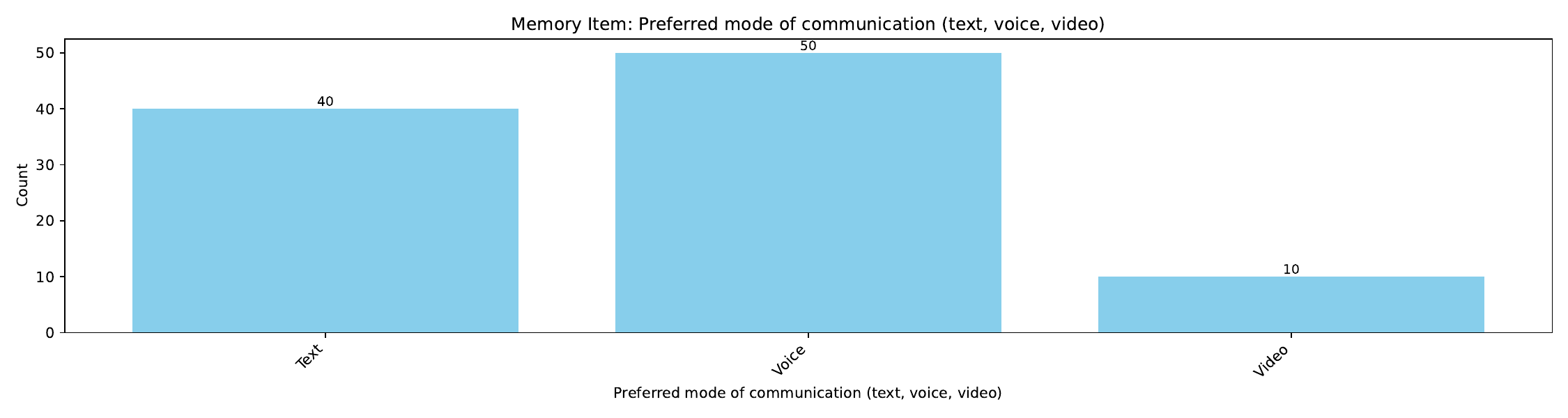}
    \caption{Distribution of Preferred mode of communication (text, voice, video)}
\end{figure}

\begin{figure}[htbp]
    \centering
    \includegraphics[width=\textwidth]{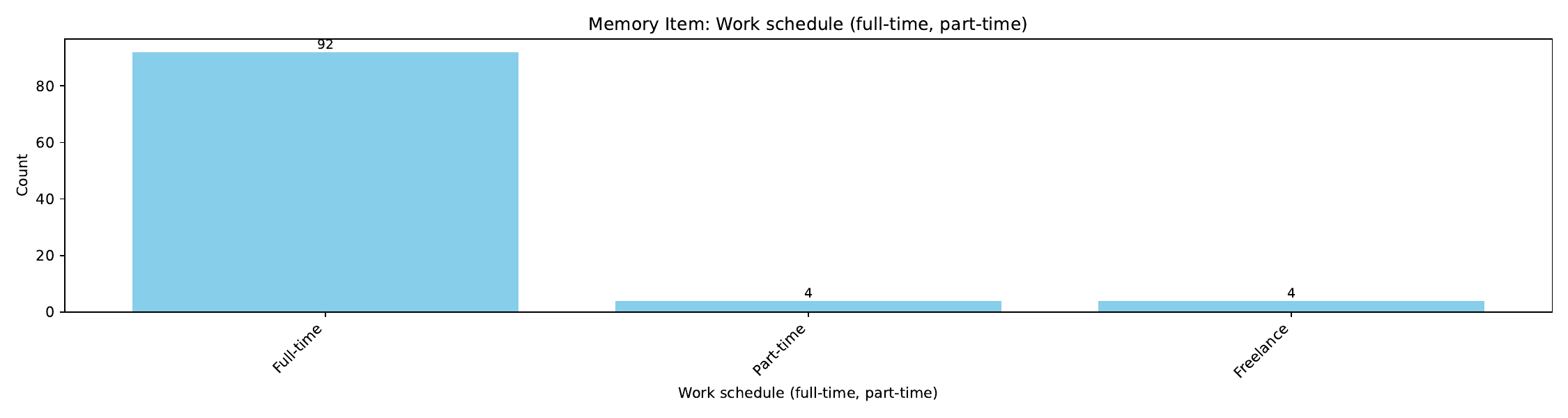}
    \caption{Distribution of Work schedule (full-time, part-time)}
\end{figure}

\begin{figure}[htbp]
    \centering
    \includegraphics[width=\textwidth]{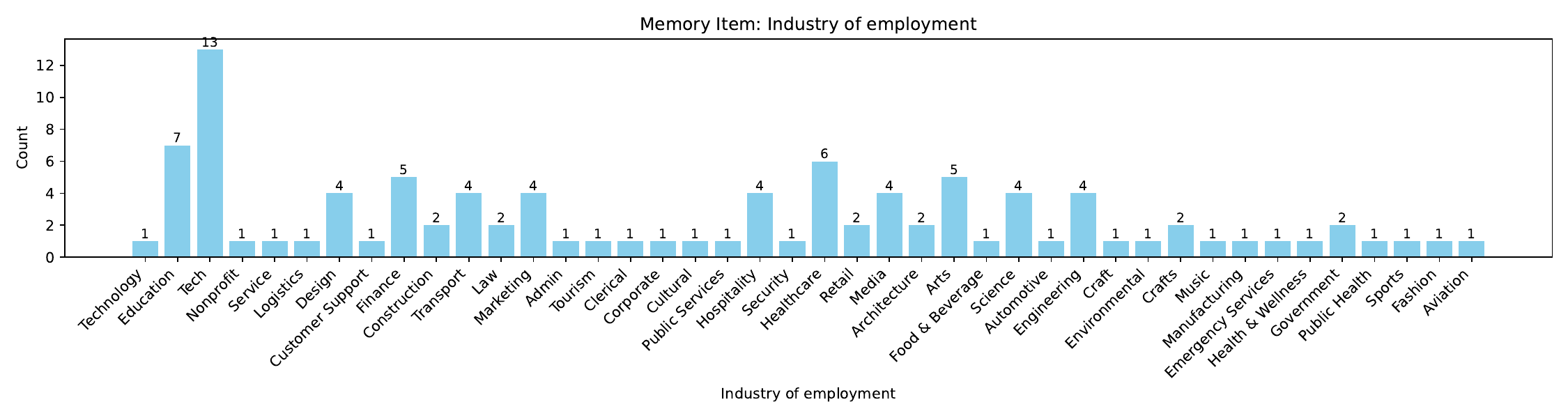}
    \caption{Distribution of Industry of employment}
\end{figure}

\begin{figure}[htbp]
    \centering
    \includegraphics[width=\textwidth]{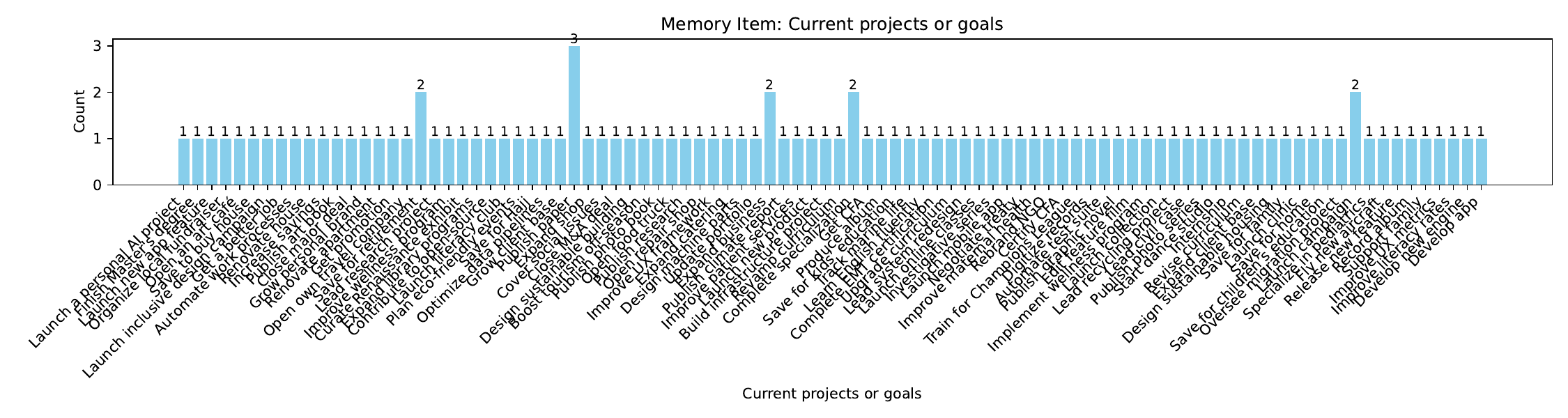}
    \caption{Distribution of Current projects or goals}
\end{figure}

\begin{figure}[htbp]
    \centering
    \includegraphics[width=\textwidth]{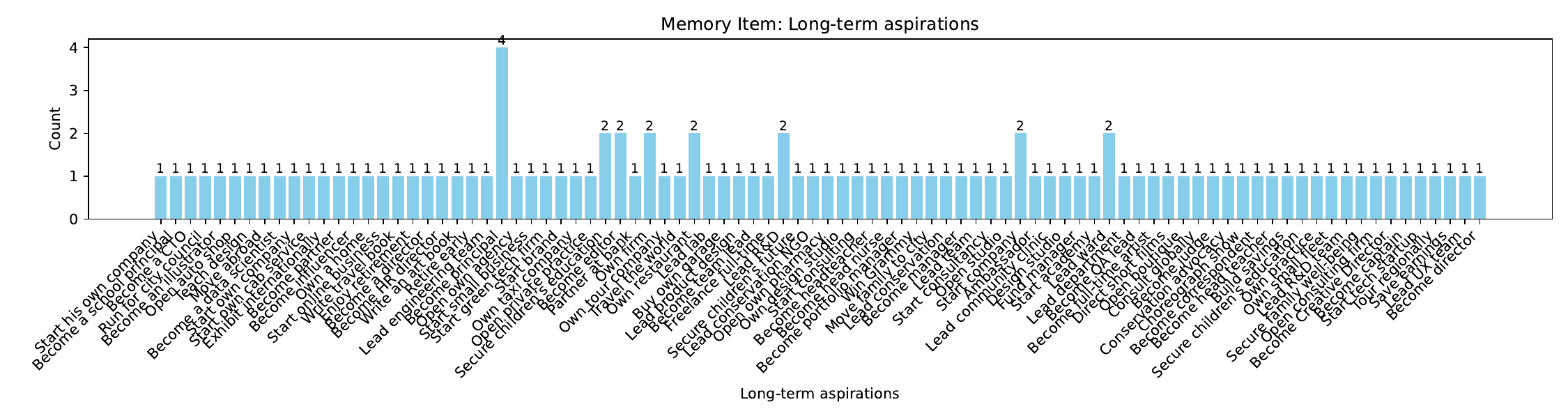}
    \caption{Distribution of Long-term aspirations}
\end{figure}

\begin{figure}[htbp]
    \centering
    \includegraphics[width=\textwidth]{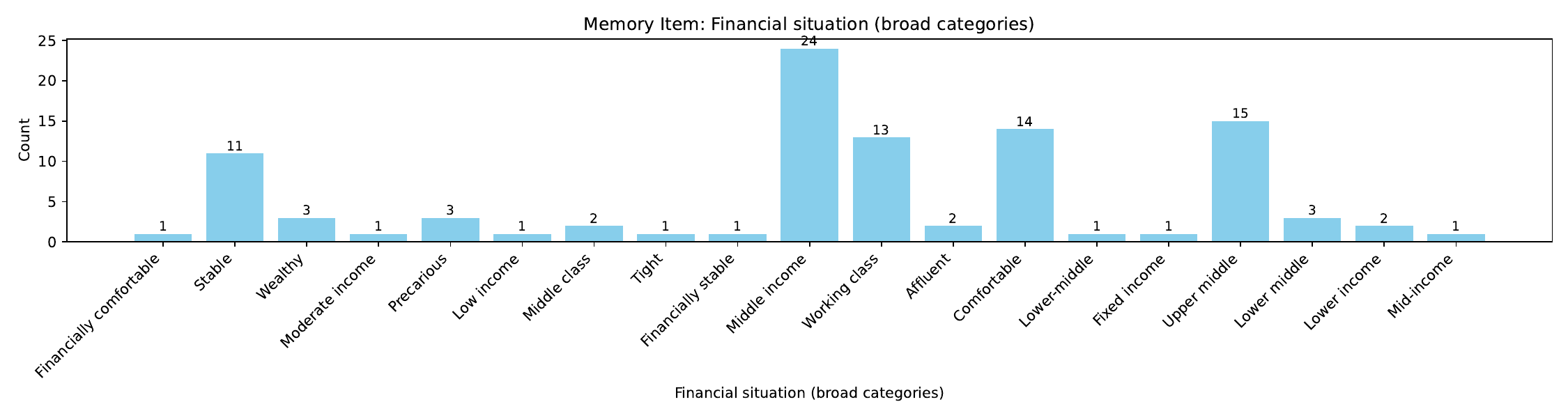}
    \caption{Distribution of Financial situation (broad categories)}
\end{figure}

\begin{figure}[htbp]
    \centering
    \includegraphics[width=\textwidth]{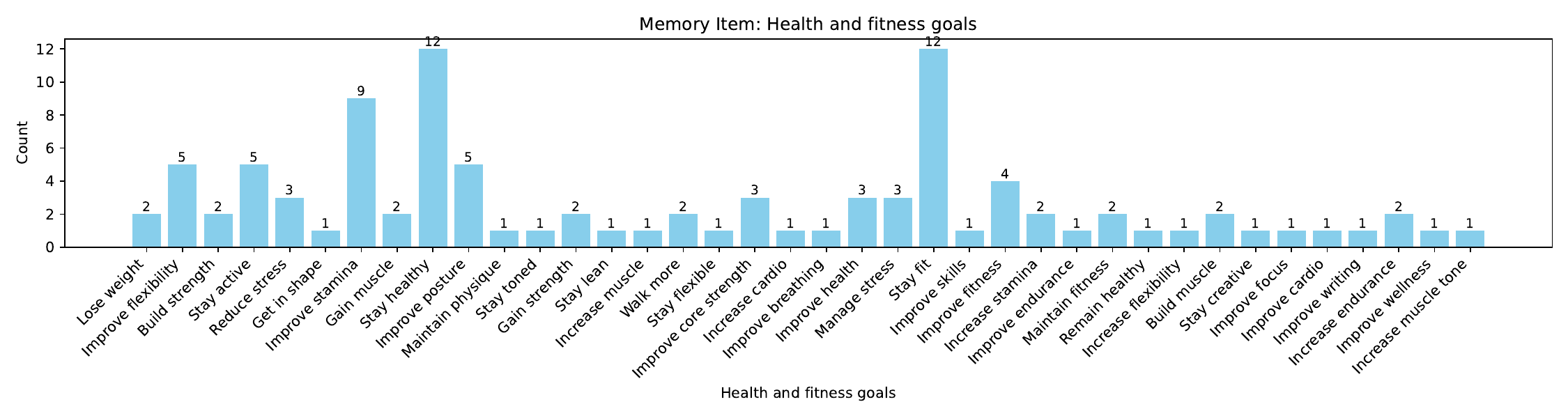}
    \caption{Distribution of Health and fitness goals}
\end{figure}

\begin{figure}[htbp]
    \centering
    \includegraphics[width=\textwidth]{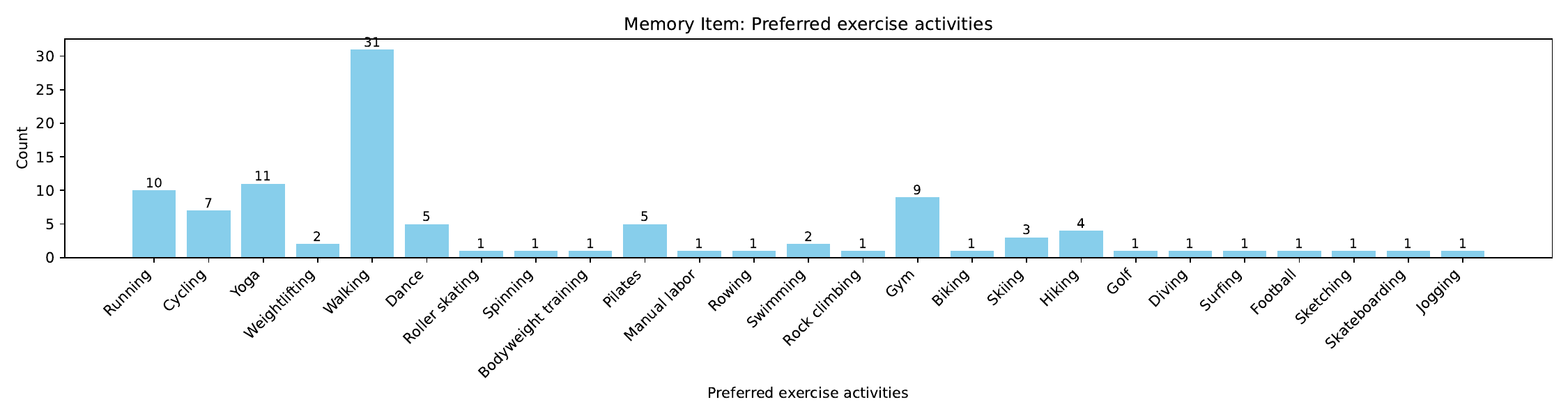}
    \caption{Distribution of Preferred exercise activities}
\end{figure}

\begin{figure}[htbp]
    \centering
    \includegraphics[width=\textwidth]{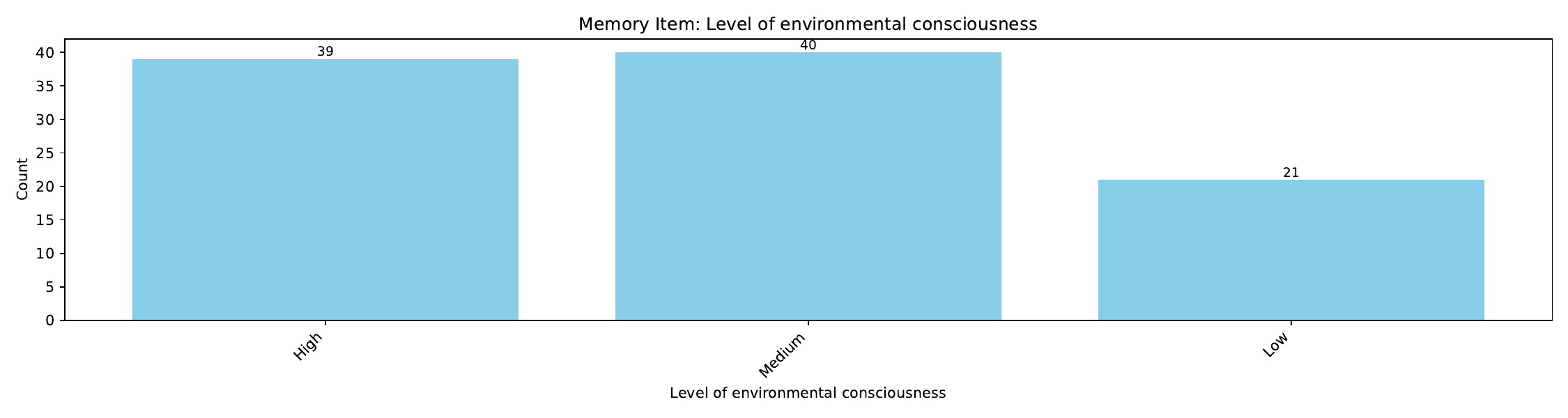}
    \caption{Distribution of Level of environmental consciousness}
\end{figure}

\begin{figure}[htbp]
    \centering
    \includegraphics[width=\textwidth]{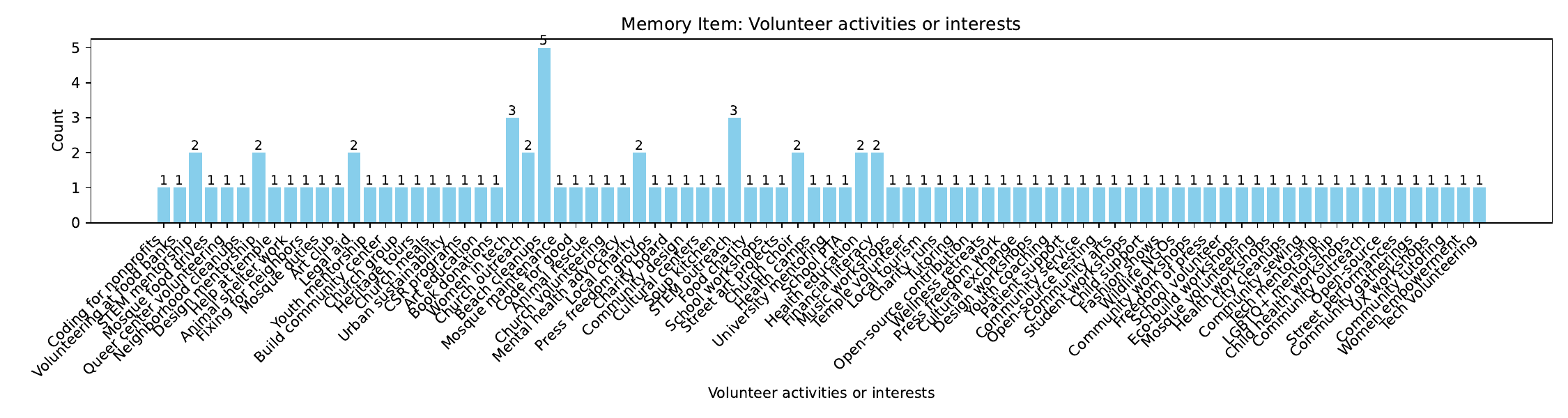}
    \caption{Distribution of Volunteer activities or interests}
\end{figure}

\begin{figure}[htbp]
    \centering
    \includegraphics[width=\textwidth]{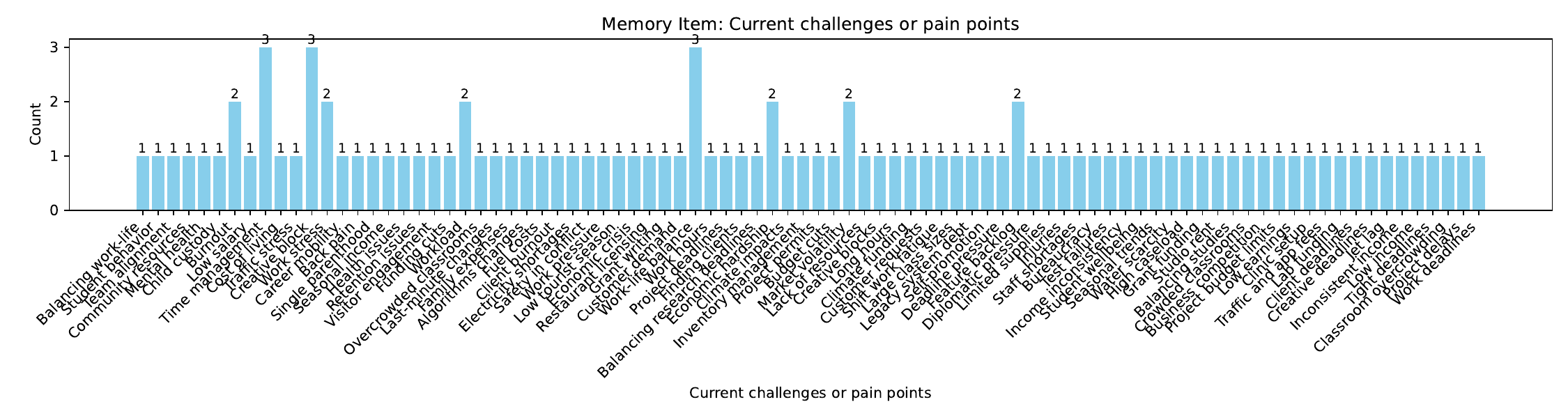}
    \caption{Distribution of Current challenges or pain points}
\end{figure}

\begin{figure}[htbp]
    \centering
    \includegraphics[width=\textwidth]{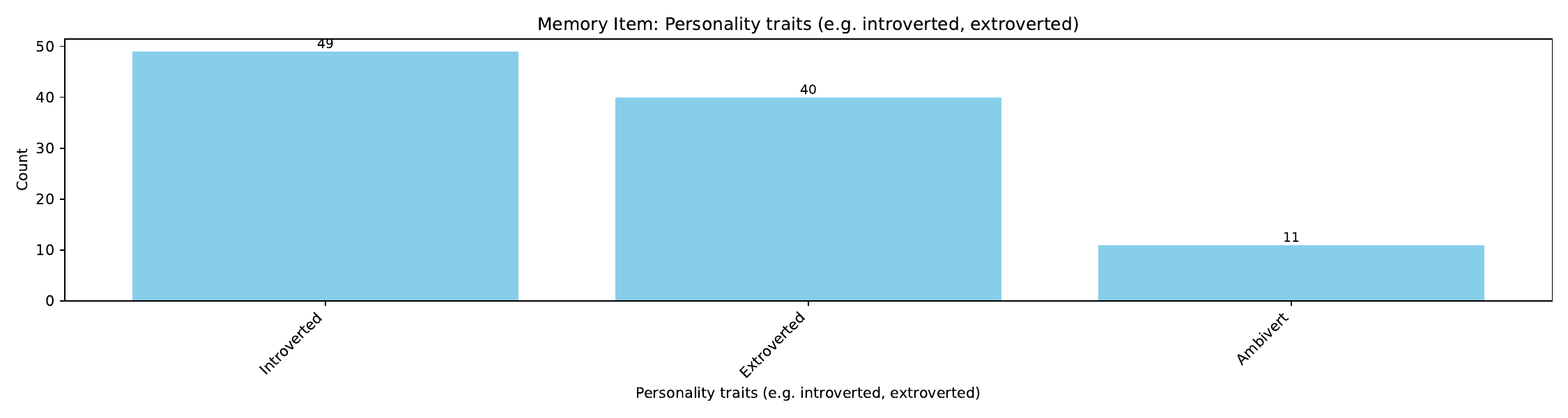}
    \caption{Distribution of Personality traits (e.g. introverted, extroverted)}
\end{figure}

\begin{figure}[htbp]
    \centering
    \includegraphics[width=\textwidth]{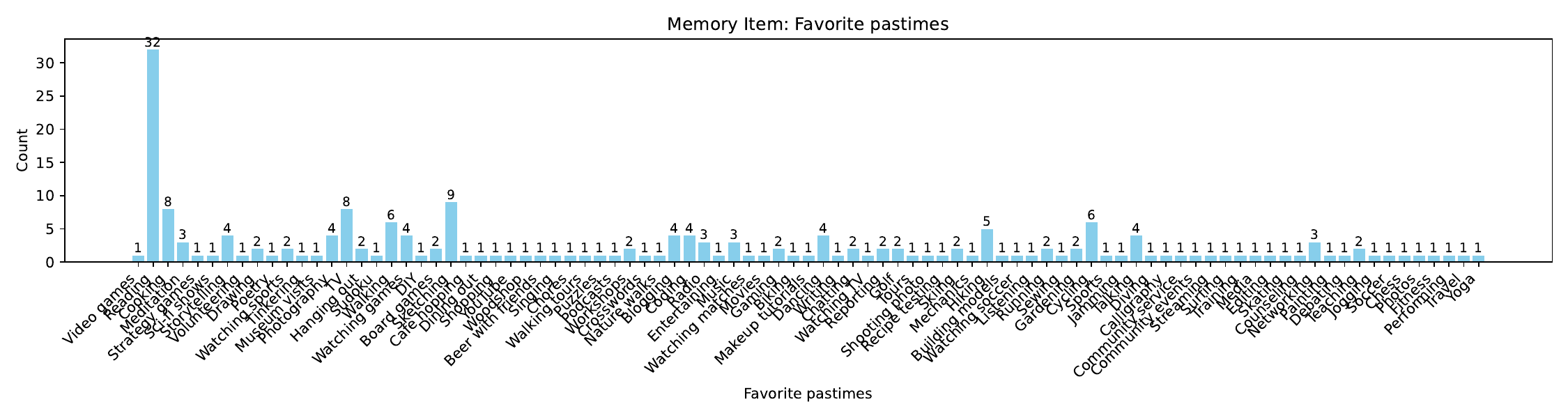}
    \caption{Distribution of Favorite pastimes}
\end{figure}

\begin{figure}[htbp]
    \centering
    \includegraphics[width=\textwidth]{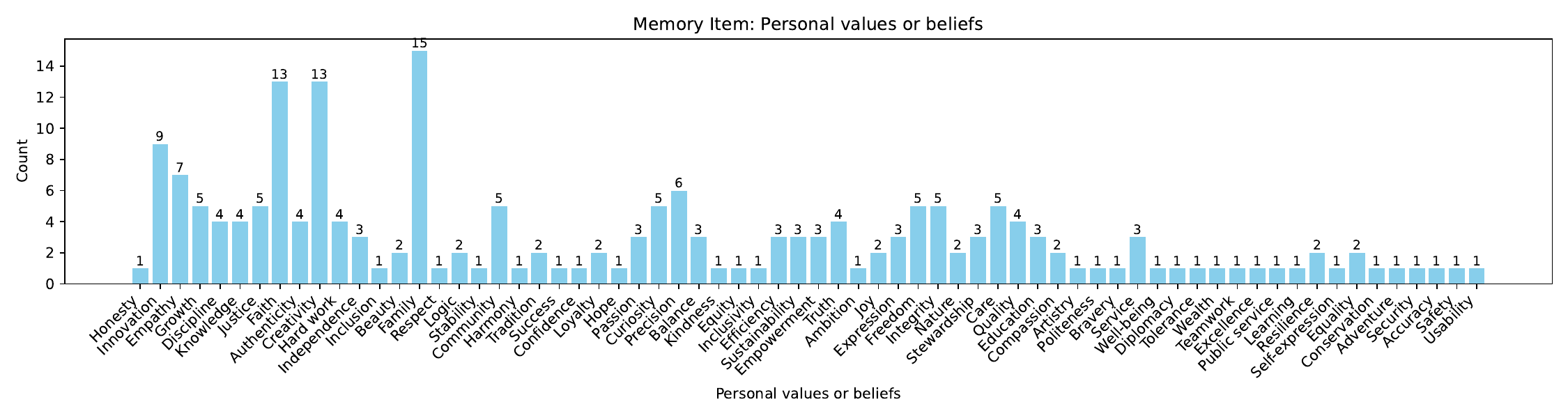}
    \caption{Distribution of Personal values or beliefs}
\end{figure}

\begin{figure}[htbp]
    \centering
    \includegraphics[width=\textwidth]{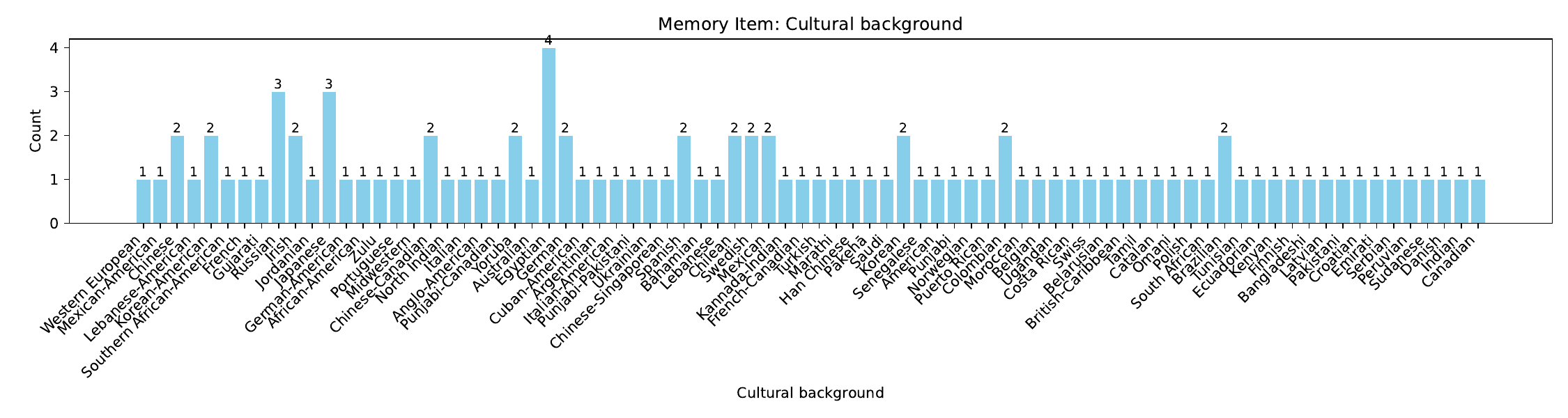}
    \caption{Distribution of Cultural background}
\end{figure}

\begin{figure}[htbp]
    \centering
    \includegraphics[width=\textwidth]{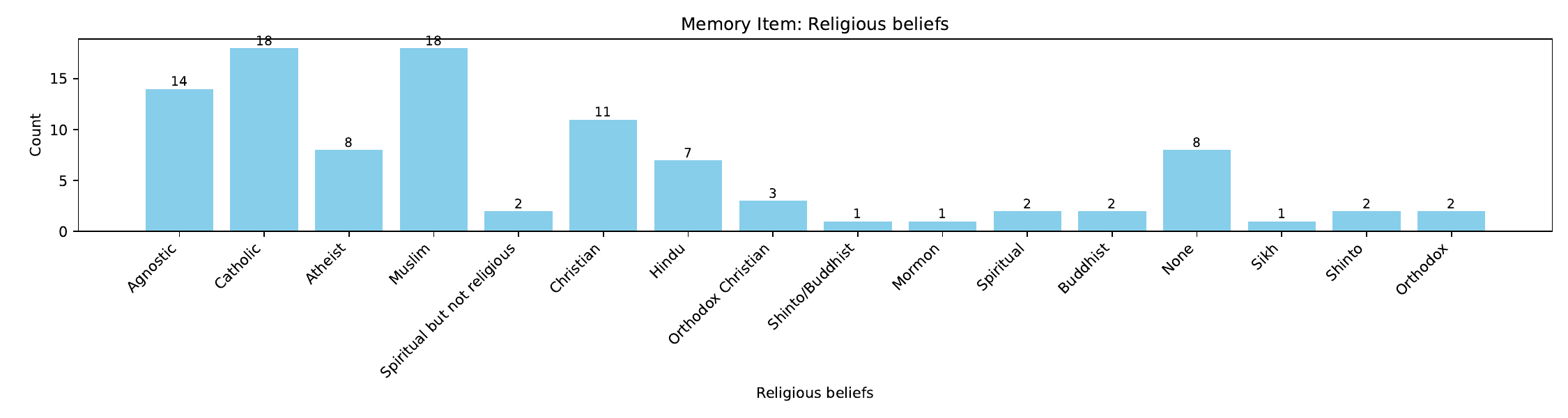}
    \caption{Distribution of Religious beliefs}
\end{figure}

\begin{figure}[htbp]
    \centering
    \includegraphics[width=\textwidth]{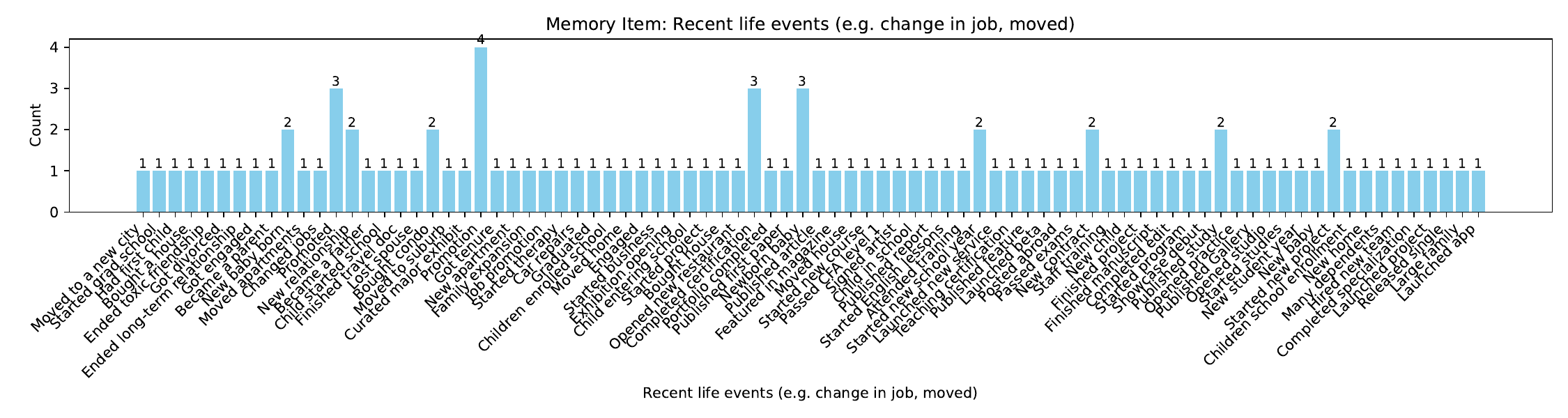}
    \caption{Distribution of Recent life events (e.g. change in job, moved)}
\end{figure}

\begin{figure}[htbp]
    \centering
    \includegraphics[width=\textwidth]{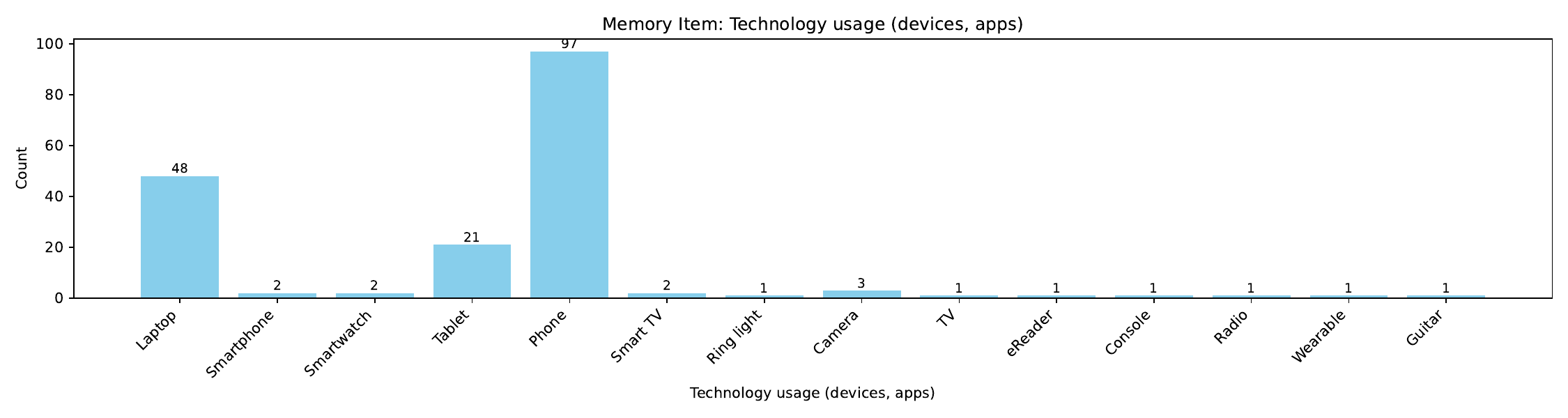}
    \caption{Distribution of Technology usage (devices, apps)}
\end{figure}

\begin{figure}[htbp]
    \centering
    \includegraphics[width=\textwidth]{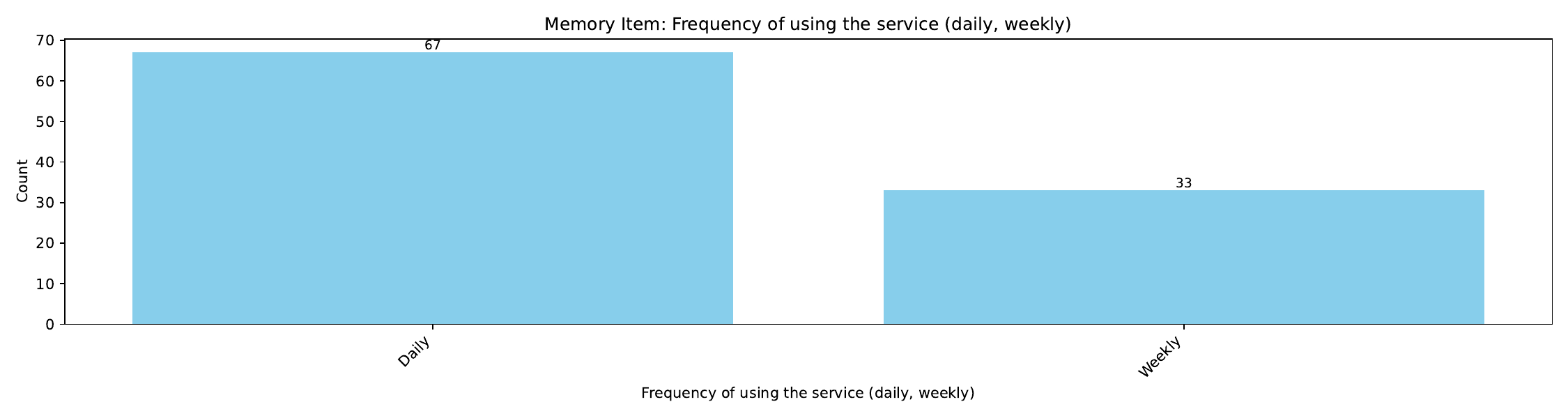}
    \caption{Distribution of Frequency of using the service (daily, weekly)}
\end{figure}

\begin{figure}[htbp]
    \centering
    \includegraphics[width=\textwidth]{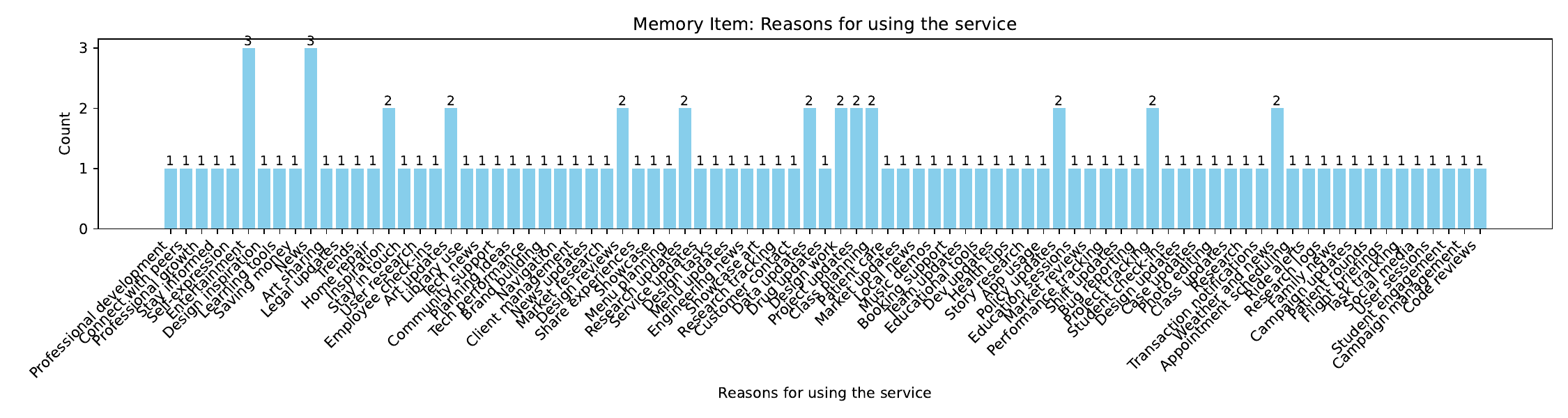}
    \caption{Distribution of Reasons for using the service}
\end{figure}

\begin{figure}[htbp]
    \centering
    \includegraphics[width=\textwidth]{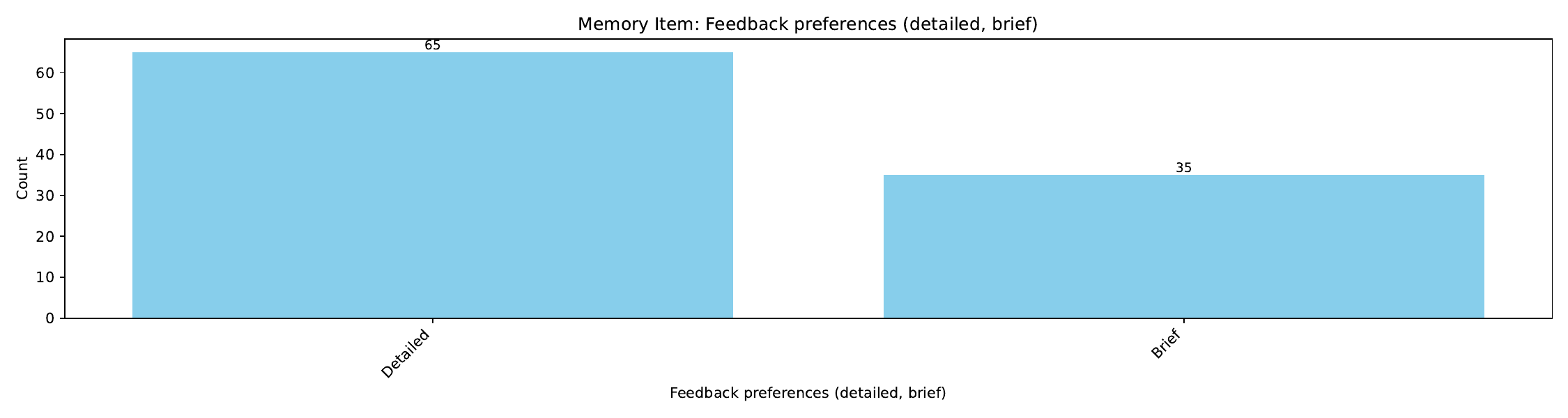}
    \caption{Distribution of Feedback preferences (detailed, brief)}
\end{figure}

\begin{figure}[htbp]
    \centering
    \includegraphics[width=\textwidth]{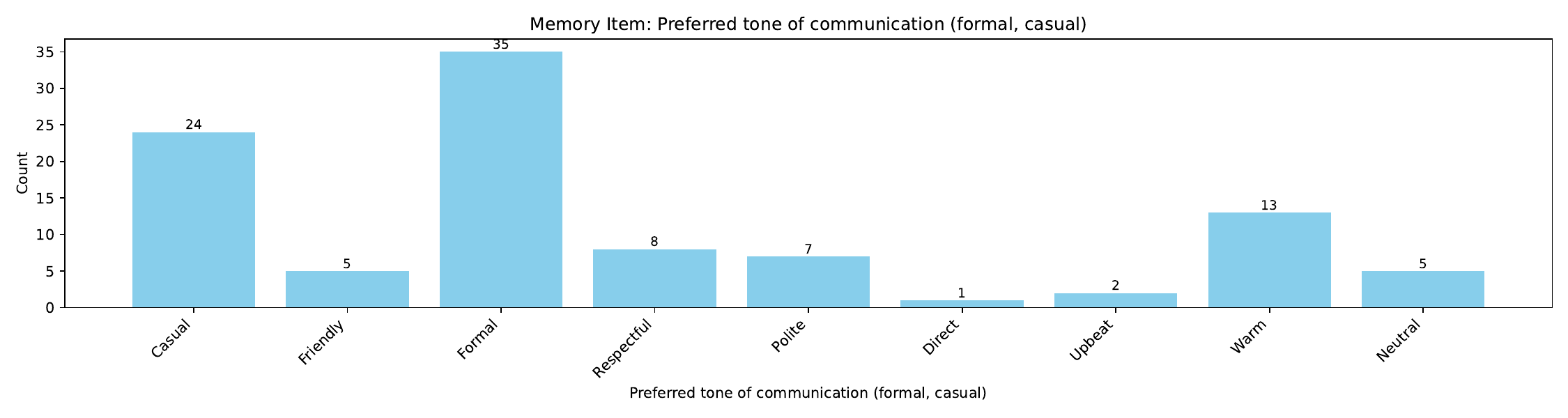}
    \caption{Distribution of Preferred tone of communication (formal, casual)}
\end{figure}

\begin{figure}[htbp]
    \centering
    \includegraphics[width=\textwidth]{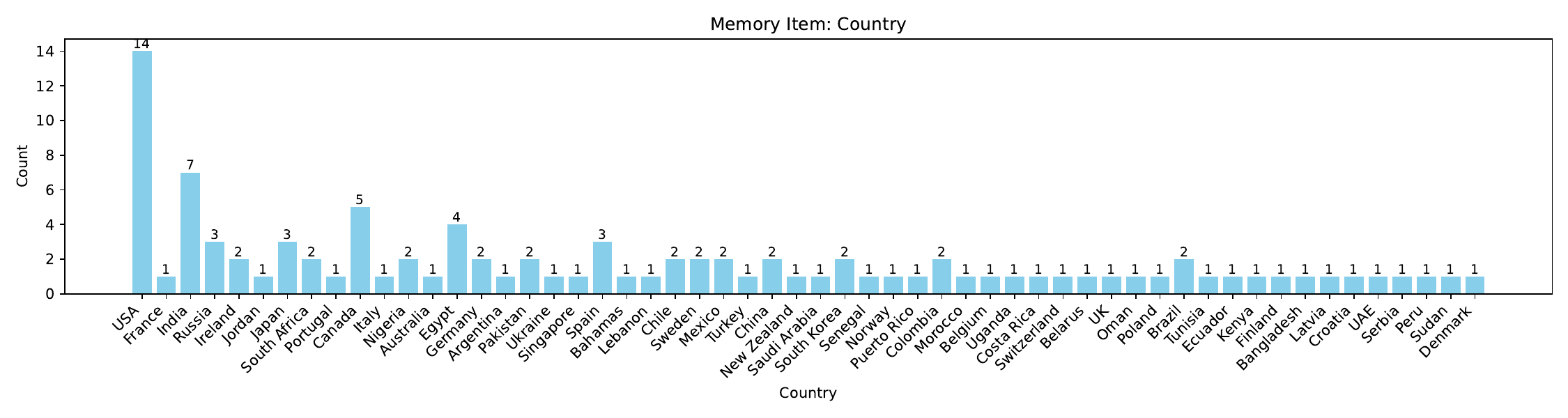}
    \caption{Distribution of Country}
    \label{fig:profile_distr_end}
\end{figure}

\end{document}